\pdfoutput=1
\documentclass[11pt]{article} 
\def\arXiv{1}

\author{
  Daniel Levy\thanks{Equal contribution. Work was done during an 
  internship at Google Research.}$^{~\, ,1}$ ~~~~~Ziteng Sun$^{*, 2}$ ~~~~~ 
  Kareem Amin$^{3}$ ~~~~~ Satyen Kale$^3$\\
  \arxiv{Alex Kulesza$^3$ ~~~~~ Mehryar Mohri$^{3, 4}$ ~~~~~ Ananda 
  	Theertha Suresh$^3$}
  \notarxiv{\textbf{Alex Kulesza$^3$ ~~~~~ Mehryar Mohri$^{3, 4}$ ~~~~~ 
  Ananda 
  Theertha Suresh$^3$} }\\ \\
  $^1$Stanford University ~~~ $^2$Cornell University ~~~ $^3$Google Research ~~~ $^4$Courant Institute \\
  \texttt{\href{danilevy@stanford.edu}{danilevy@stanford.edu}, ~~ \href{zs335@cornell.edu}{zs335@cornell.edu},}\\
    \texttt{\{\href{kamin@google.com}{kamin}, \href{satyenkale@google.com}{satyenkale}, \href{kulesza@google.com}{kulesza}, \href{mohri@google.com}{mohri}, \href{theertha@google.com}{theertha}\}@google.com}
  \date{}
}
\newcommand{\notarxiv}[1]{foo}
\newcommand{\arxiv}[1]{ba}

\ifdefined \arXiv
	\renewcommand{\arxiv}[1]{#1}%
	\renewcommand{\notarxiv}[1]{\ignorespaces}%
\else%
	\renewcommand{\arxiv}[1]{\ignorespaces}%
	\renewcommand{\notarxiv}[1]{#1}%
        \fi%
        
\usepackage{comment}
\usepackage{xspace}
\usepackage[protrusion=true,expansion=true]{microtype}
\usepackage{fullpage}

\usepackage{mathtools}

\usepackage{amsfonts,amsmath,amssymb,pbox}
  \usepackage[colorlinks=true,allcolors=blue]{hyperref}
  \usepackage{amsthm}
  \usepackage{thmtools}
\usepackage{booktabs}
\usepackage{multirow}
\usepackage{dsfont} %

\usepackage{algorithmic,algorithm}

\usepackage[shortlabels]{enumitem}
\setitemize{noitemsep,topsep=0pt,parsep=0pt,partopsep=0pt}
\setenumerate{noitemsep,topsep=0pt,parsep=0pt,partopsep=0pt}

\usepackage{thm-restate}

\newtheorem{assumption}{Assumption}

\theoremstyle{plain}

\newtheorem{theorem}{Theorem}
\newtheorem{lemma}{Lemma}
\newtheorem{claim}{Claim}
\newtheorem{proposition}{Proposition}
\newtheorem{corollary}{Corollary}
\newtheorem{definition}{Definition}

\theoremstyle{definition}
\newtheorem{remark}{Remark}

\newtheorem*{example*}{Example}

\newcommand{\ignore}[1]{}

\newcommand{\EE}{\mathbb{E}}

\newcommand{\RR}{\mathbb{R}}

\newcommand{\expectation}[1]{\EE\left[#1\right]}

\newcommand{\expectsub}[2]{\EE_{#1}\left[#2\right]}

\def \cA     {{\cal A}}

\def \cD     {{\cal D}}

\def \cL     {{\cal L}}

\def \cP     {{\cal P}}
\def \cQ     {{\cal Q}}

\def \cS     {{\cal S}}

\def \cZ     {{\cal Z}}

\newcommand{\Var}{{\rm Var}}

\newcommand{\ie}{\textit{i.e.,}\xspace}  %
\newcommand{\newzs}[1]{{\color{red}#1}}

\def \ceil#1{{\lceil{#1}\rceil}}

\def \floor#1{{\lfloor{#1}\rfloor}}

\def \Paren#1{{\left({#1}\right)}}

\newcommand{\probof}[1]{\P\Paren{#1}}

\def\ignore#1{}

\newcommand{\bi}{\begin{itemize}}
\newcommand{\ei}{\end{itemize}}

\def\orpro{\mathop{\mathchoice
   {\vee\kern-.49em\raise.7ex\hbox{$\cdot$}\kern.4em}
   {\vee\kern-.45em\raise.63ex\hbox{$\cdot$}\kern.2em}
   {\vee\kern-.4em\raise.3ex\hbox{$\cdot$}\kern.1em}
   {\vee\kern-.35em\raise2.2ex\hbox{$\cdot$}\kern.1em}}\limits}

\def\andpro{\mathop{\mathchoice
 {\wedge\kern-.46em\lower.69ex\hbox{$\cdot$}\kern.3em}
 {\wedge\kern-.46em\lower.58ex\hbox{$\cdot$}\kern.25em}
 {\wedge\kern-.38em\lower.5ex\hbox{$\cdot$}\kern.1em}
 {\wedge\kern-.3em\lower.5ex\hbox{$\cdot$}\kern.1em}}\limits}

\def\simge{\mathrel{%
   \rlap{\raise 0.511ex \hbox{$>$}}{\lower 0.511ex \hbox{$\sim$}}}}

\def\simle{\mathrel{
   \rlap{\raise 0.511ex \hbox{$<$}}{\lower 0.511ex \hbox{$\sim$}}}}

\newcommand{\mc}[1]{\mathcal{#1}}

\newcommand{\mrm}[1]{\mathrm{#1}}
\newcommand{\msf}[1]{\mathsf{#1}}

\newcommand{\norms}[1]{\|{#1}\|} %

\DeclarePairedDelimiter{\abs}{\lvert}{\rvert} 
\DeclarePairedDelimiter{\brk}{[}{]}
\DeclarePairedDelimiter{\crl}{\{}{\}}
\DeclarePairedDelimiter{\prn}{(}{)}

\DeclarePairedDelimiter{\norm}{\|}{\|}
\DeclarePairedDelimiter{\tri}{\langle}{\rangle}

\newcommand{\defeq}{\coloneqq}
\newcommand{\eqdef}{\eqqcolon}
\newcommand{\what}[1]{\widehat{#1}} %
\newcommand{\indic}[1]{1\!\left\{#1\right\}} %

\newcommand{\R}{\mathbb{R}}
\newcommand{\N}{\mathbb{N}}

\newcommand{\E}{\mathbb{E}} %
\renewcommand{\P}{\mathbb{P}} %
\newcommand{\var}{{\rm Var}} %
\newcommand{\simiid}{\stackrel{\rm iid}{\sim}}

\newcommand{\dkl}[2]{D_{\rm kl}\left({#1} \mid \mid {#2}\right)}
\providecommand{\argmin}{\mathop{\rm argmin}}

\providecommand{\minimize}{\mathop{\rm minimize}}

\arxiv{\usepackage[numbers, square]{natbib}}

\title{Learning with User-Level Privacy}

\usepackage{stackengine}
\newcommand\tsup[2][2]{%
 \def\useanchorwidth{T}%
  \ifnum#1>1%
    \stackon[-.5pt]{\tsup[\numexpr#1-1\relax]{#2}}{\scriptscriptstyle\sim}%
  \else%
    \stackon[.5pt]{#2}{\scriptscriptstyle\sim}%
  \fi%
}

\newcommand{\tstop}{T_s}

\newcommand{\dtv}[2]{\norm{#1 - #2}_{\msf{TV}}}
\arxiv{
  
}

\newcommand{\lap}[1]{\text{Lap}\Paren{#1}}

\newcommand{\smooth}{H}

\newcommand{\zu}{\z^{(u)}}

\newcommand{\z}{z}
\newcommand{\zz}{Z}
\newcommand{\ZZ}{\mc{Z}}

\newcommand{\eps}{\varepsilon}

\newcommand{\A}{\mathsf{A}}
\newcommand{\dham}{\mathrm{d}_{\mathsf{Ham}}}

\newcommand{\f}{\ell}
\newcommand{\ff}{\mc{L}}

\newcommand{\minimaxpurebdd}{\mathfrak{M}^{\msf{user}}_{m, n}(\Theta, \mc{F}_B, \epsilon)}

\renewcommand{\O}{O}
\newcommand{\Ologlog}{\tilde{\tilde}{O}}
\newcommand{\Olog}{\tilde{O}}

\newcommand{\Omloglog}{\tilde{\Omega}}

\newcommand{\alg}{\msf{A}}

\newcommand{\Auser}{\mc{A}_{\eps, \delta}^{\msf{user}}}

\newcommand{\Auserp}{\mc{A}_{\eps}^{\msf{user}}}

\newcommand{\Gbar}{\underline{G}}
\newcommand{\Gtr}{\msf{N}^\mrm{tr}}
\newcommand{\Gsig}{\widetilde{G}}
\newcommand{\whmat}{\mathbf{H}}
\arxiv{ \usepackage{color} \usepackage[dvipsnames]{xcolor}}

\newcommand{\dotp}[2]{\langle #1, #2\rangle}

  \newcommand{\minimaxitem}{\mathfrak{M}^{\msf{item}}_n(\Theta, \mc{F}, \epsilon)}

\makeatletter
\long\def\@makecaption#1#2{
  \vskip 0.8ex
  \setbox\@tempboxa\hbox{\small {\bf #1:} #2}
  \parindent 1.5em  %
  \dimen0=\hsize
  \advance\dimen0 by -3em
  \ifdim \wd\@tempboxa >\dimen0
  \hbox to \hsize{
    \parindent 0em
    \hfil 
    \parbox{\dimen0}{\def\baselinestretch{0.96}\small
      {\bf #1.} #2
    } 
    \hfil}
  \else \hbox to \hsize{\hfil \box\@tempboxa \hfil}
  \fi
}
\makeatother

\newcommand{\hetero}{\Delta}%
\newcommand{\poly}{\mathsf{poly}}

 \usepackage{graphicx}
\notarxiv{
  \usepackage{titlesec}
 \titlespacing*{\section}
 {0pt}{6pt plus 2pt minus 1pt}{3pt plus 0pt minus 3pt}
 \titlespacing*{\subsection}
 {0pt}{3pt plus 0pt minus 3pt}{0pt plus 0pt minus 3pt}
 \titlespacing*{\paragraph}
 {0pt}{0pt plus 0pt minus 0pt}{6pt}
}
\newcommand{\normgrad}{G}
\newcommand{\normpara}{R}
\newcommand{\ploss}{\cL}
\newcommand{\vecz}{\mathbf{z}}

\newcommand{\theerthaedit}[2]{{\color{red}{{#2}}}}
\newcommand{\dl}[1]{{\bf\color{blue} [DL]: #1}}

\renewcommand{\theerthaedit}[1]{}

\begin{document}

\maketitle

\begin{abstract}
We propose and analyze algorithms to solve a range of learning tasks under
user-level differential privacy constraints. Rather than guaranteeing only the
privacy of individual samples, user-level DP protects a user's entire
contribution ($m \ge 1$ samples), providing more stringent but more realistic
protection against information leaks.  We show that for high-dimensional mean
estimation, empirical risk minimization with smooth losses, stochastic convex
optimization, and learning hypothesis classes with finite metric entropy, the
privacy cost decreases as $O(1/\sqrt{m})$ as users provide more samples. In
contrast, when increasing the number of users $n$, the privacy cost decreases at
a faster $O(1/n)$ rate.  We complement these results with lower bounds showing
the minimax optimality of our algorithms for mean estimation and stochastic
convex optimization. Our algorithms rely on novel techniques for private mean
estimation in arbitrary dimension with error scaling as the concentration radius
$\tau$ of the distribution rather than the entire range.
 \end{abstract}

\section{Introduction}

Releasing seemingly innocuous functions of a data set can easily compromise the
privacy of individuals, whether the functions are simple counts
\citep{HomerSzReDuTeMuPeStNeCr08} or complex machine learning models like deep
neural networks \citep{ShokriSh15, FredriksonJhRi15}. To protect against such
leaks, \citeauthor{DworkMcNiSm06} proposed the notion of \emph{differential
  privacy} (DP). Given some data from $n$ participants in a study, we say that a
statistic of the data is differentially private if an attacker who already knows
the data of $n - 1$ participants cannot reliably determine from the statistic
whether the $n$-th remaining participant is Alice or Bob. With the recent
explosion of publicly available data, progress in machine learning, and
widespread public release of machine learning models and other statistical
inferences, differential privacy has become an important standard and is widely
adopted by both industry and government \citep{GooglePrivacy19, ApplePrivacy17,
  ding2017collecting, USCensus18}.

The standard setting of DP described in~\cite{DworkMcNiSm06} assumes that
each participant contributes a \emph{single} data point to the
dataset, and preserves privacy by ``noising'' the output in a way that
is commensurate with the maximum contribution of a single
example. This is not the situation faced in many applications of
machine learning models, where users often contribute \emph{multiple}
samples to the model---for example, when language and image
recognition models are trained on the users' own data, or in federated
learning settings \citep{kairouz2019advances}. As a result, current
techniques either provide privacy guarantees that degrade with a
user’s increased participation or naively add a substantial amount of
noise, relying on the group property of differential privacy, which
significantly harms the performance of the deployed model.

To remedy this issue, we consider \emph{user-level} DP, which instead of
guaranteeing privacy for individual samples, protects a user's \emph{entire
  contribution} ($m\ge 1$ samples). This is a more stringent but more realistic
privacy desideratum. To hold, it requires that the output of our algorithm does
not significantly change when changing user's entire
contribution---i.e. possibly swapping up to $m$ samples in total. We make this
formal in Definition~\ref{def:user_dp}.
 Very recently, for the reasons outlined
above, there has been increasing interest in user-level DP for applications such
as estimating discrete distributions under user-level privacy constraints
\citep{LiuSuYuKuRi20}, PAC learning with user-level privacy \citep{ghazi2021user},  and bounding user contributions in ML models
\citep{amin2019bounding, epasto2020smoothly}. Differentially private SQL with
bounded user contributions was proposed in~\cite{wilson2019differentially}. 
User-level privacy has
been also studied in the context of learning models via federated learning
\citep{mcmahan2017learning, mcmahan2018general,wang2019beyond,
  augenstein2019generative}.

In this paper, we tackle the problem of \emph{learning} with user-level privacy
in the central model of DP. In particular, we provide algorithms and analyses
for the tasks of mean estimation, empirical risk minimization (ERM), stochastic
convex optimization (SCO), and learning hypothesis classes with finite metric
entropy.  Our utility analyses assume that all users draw their samples
i.i.d. from related distributions, a setting we refer to as \emph{limited
  heterogeneity}. On these tasks, naively applying standard mechanisms, such as
Laplace or Gaussian, or using the group property with item-level DP estimators,
both yield a privacy error independent of $m$.
We first develop novel private
mean estimators in high dimension with statistical and privacy error scaling
with the (arbitrary) concentration radius rather than the range.
Our algorithms rely on (privately) answering a sequence of adaptively chosen
queries using users' samples, e.g., gradient queries in stochastic gradient
descent algorithms. We show that for these tasks, the additional error due to
privacy constraints decreases as $O(1/\sqrt{m})$, contrasting with the naive
rate---independent of $m$. %
Interestingly, increasing $n$, the number of users, decreases the privacy
cost at a faster $O(1/n)$ rate.

Importantly, our results imply concrete practical recommendations on sample
collection, \emph{regardless of the level of heterogeneity}. Indeed, increasing
$m$ will yield the most value in the i.i.d.\ setting and will yield no
improvement when the users’ distributions are arbitrary. As the real-world will
lie somewhere in between, our results exhibit a regime where, for any
heterogeneity, it is strictly better to collect more users (increasing $n$) than
more samples per user (increasing $m$).

\subsection{Our Contributions and Related Work} \label{sec:results}

We provide a theoretical tool to construct estimators for tasks with
user-level privacy constraints and apply it to a range of learning
problems.
\paragraph{Optimal private mean estimation and uniformly concentrated queries
  (Section~\ref{sec:sq})} \arxiv{In this section, we present our main technical
tool. }We show that for a random variable in $[-B, B]$ concentrated in an unknown
interval of radius $\tau$ (made precise in Definition~\ref{def:concentrated}),
we can privately estimate its mean with \arxiv{statistical and private }error
proportional to $\tau$ rather than \arxiv{the whole range }$B$, as we would 
obtain using
standard private mean estimation techniques such as Laplace mechanism
~\citep{dwork2014algorithmic}. 
When data is concentrated in
$\ell_\infty$-norm, several papers show that one can achieve an error scaling
with $\tau$ rather than $B$, either asymptotically~\citep{smith2011privacy}, for
Gaussian mean-estimation~\citep{karwa2017finite, kamath2019privately}, for
sub-Gaussian symmetric distributions~\citep{cai2019cost, bun2019average} or for
distributions with bounded $p$-th moment~\citep{kamath2020private}. We propose a
private mean estimator (Algorithm~\ref{alg:winsorized_highd}) with error scaling
with $\tau$ that works in arbitrary dimension when data is concentrated in
$\ell_2$-norm (Theorem~\ref{thm:winsorized_highd}). Our algorithm 
avoids a superfluous $\sqrt{d}$ factor compared to naively applying 
previous 
approaches coordinate-wise. In
Corollary~\ref{coro:bounded}, %
we show it (optimally) solves mean estimation under user-level privacy
constraints for random vectors bounded in $\ell_2$-norm. In
\notarxiv{Appendix~\ref{sec:uc}}
\arxiv{Theorem~\ref{thm:sq_highd}}, 
we show that for uniformly concentrated queries
(see Definition~\ref{def:uniform_concentration}), sequentially applying
Algorithm~\ref{alg:winsorized_highd} privately answers $K$ adaptively chosen
queries with privacy cost $\tilde{O}(\tau\sqrt{K}/n\eps)$.

Our conclusions relate to the growing literature in adaptive data 
analysis.
While a sequence of work~\citep{dwork2015preserving, 
bassily2016algorithmic,
	feldman2017median, feldman2018calibrating} 
use techniques from differential privacy and their answers are $(\eps,
\delta)$-DP with $\eps =
\Theta(1)$, our work guarantees privacy for arbitrary
$\eps$ with the additional assumption of uniform concentration.

\paragraph{Empirical risk minimization (Section~\ref{sec:erm})}
An influential line of papers studies ERM under item-level privacy
constraints~\citep{chaudhuri2011differentially, kifer2012private,
  bassily2014private}. Importantly, these papers assume \emph{arbitrary} data,
i.e., not necessarily samples from users' distributions. The exact analog of ERM
in the user-level setting is consequently less interesting as, for $n$ data
points $\crl{z_1, \ldots, z_n}$, in the worst case, each user $u\in\brk{n}$
contributes $m$ copies of $z_u$ and the problem reduces to the item-level
setting. Instead, we consider the (related) problem of ERM when users 
contribute
points sampled i.i.d.
Assuming some regularity (\ref{ass:smoothness} and~\ref{ass:subG}), 
we develop and analyze algorithms for ERM under user-level
DP constraints for convex, strongly-convex, and non-convex losses
(Theorem~\ref{thm:erm}). \arxiv{We show in Theorem~\ref{thm:erm} that the cost
  of privacy decreases as $O(1/\sqrt{m})$ for the convex and non-convex cases
  and $O(1/m)$ for the strongly convex case as user contribution increases.
}%

\paragraph{Optimal stochastic convex optimization
  (Section~\ref{sec:sco})}
Under item-level DP (or equivalently, user-level DP with $m = 1$), a sequence of
work~\citep{chaudhuri2011differentially, bassily2014private, bassily2019private,
  bassily2020stability, feldman2020private} establishes the constrained minimax
risk as
$\tilde{\Theta} (\arxiv{\normgrad \normpara(}1/\sqrt{n} + \sqrt{d}
/(n\eps)\arxiv{)})$\arxiv{ when the loss is $G$-Lipschitz and the parameter
  space has diameter less than $R$}. In this paper, with the additional
assumptions that the losses are individually smooth\footnote{We note that the
  results only require $\tilde{O}(n^{3/2})$-smooth losses. For large
  $n$---keeping all other problem parameters fixed---this is a very weak
  assumption. More precisely, when $n > \text{poly}\prn{d, m, 1/\eps}$, our
  algorithm on a smoothed version $\tilde{\f}$ of $\f$ (e.g., using the Moreau
  envelope~\citep{GuzmanNe15}) yields optimal rates for non-smooth
  losses. Whether the smoothness assumption can be removed altogether is an open
  question.}  \arxiv{and the stochastic gradients are $\sigma^2$-sub-Gaussian,
  we prove an upper bound %
of
$\tilde{O} (\normpara\sqrt{\normgrad \Gbar }/\sqrt{mn} + \normpara\Gsig
\sqrt{d}/n\sqrt{m} \eps)$ and a lower bound of
$\Omega (\normpara \Gbar /\sqrt{mn} + R \Gbar \sqrt{d}/n\sqrt{m} \eps)$ on the
population risk, where $\Gsig = \sigma\sqrt{d}$ and
$\Gbar = \min \{ \normgrad, \Gsig\}$. We present precise statements in
Theorems~\ref{thm:dp_sco_upper} and~\ref{thm:dp_sco_lower}. When
$G = \Omega(\sigma\sqrt{d})$, the privacy rates match and when
$G = O(\sigma\sqrt{d})$, the statistical rates match (in both cases up to
logarithmic factors). We leave closing the gap outside of this regime to future
work.}  \notarxiv{and the gradients are sub-Gaussian random vectors, we prove
matching upper (Theorem~\ref{thm:dp_sco_upper}) and lower bounds
(Theorem~\ref{thm:dp_sco_lower}) of order
$\tilde{\Theta}(1/\sqrt{nm} + \sqrt{d} / (n\sqrt{m}\eps))$ in a regime we make
precise. We leave closing the gap outside of this regime to future work.}

\arxiv{
\paragraph{Function classes with finite metric entropy under pure DP
  (\arxiv{Section}\notarxiv{Appendix}~\ref{sec:warm-up})}
Our previous results only hold for approximate user-level DP. Turning to pure
DP, we consider function classes with bounded range and finite metric
entropy. We provide an estimator combining our mean estimation techniques with
the private selection mechanism of \cite{liu2019private}. For a finite class of
size $K$, we achieve an excess risk of
$\tilde{O}(\sqrt{(\log K)/mn} + (\log K)^{3/2}/n\sqrt{m}\eps)$. We further prove
a lower bound of $\Omega(\sqrt{(\log K)/mn} + \log K/n\sqrt{m}\eps)$. For SCO
with $G$-Lipschitz gradients and $\ell_\infty$-bounded domain, a covering number
argument implies an excess risk of
$\tilde{O}(\normgrad \normpara \sqrt{d/mn} + d^{3/2}/n\sqrt{m}\eps)$ under pure
user-level DP constraints. While the statistical rate is optimal, we observe a
gap of order $\sqrt{d}$ for the privacy cost compared to the lower bound, which
we leave as future work.}

\paragraph{Limit of learning with a fixed number of users
  (Appendix~\ref{sec:limit})} Finally, we resolve a conjecture
of~\cite{amin2019bounding} and prove that with a fixed number of users, even in
the limit $m\to\infty$ (i.e., each user has an infinite number of samples), we cannot
reach zero error. In particular, we prove that for all the learning tasks we
consider, the risk under user-level privacy constraints is at least
$\Omega(e^{- \eps n})$ regardless of $m$.  Note that this does not contradict
the results above since they require $n = \Omega((\log m)/\eps)$.

\notarxiv{Finally, we provide results in Appendix~\ref{sec:warm-up}
  for learning under \emph{pure} user-level DP for function classes with finite
  metric entropy. We apply these to SCO with $\ell_\infty$ constraints
  (Remark~\ref{rm:l_infity}) and achieve (near)-optimal rates.}
\section{Preliminaries}

\paragraph{Notation.} Throughout this work, $d$ denotes the dimension, $n$ the
number of users%
, and $m$ the number of samples per user. Generically, $\sigma$ will denote the
sub-Gaussian parameter, $\tau$ the concentration radius, $\nu$ the variance of a
random vector and $P$ a data distribution. We denote the optimization variable
with $\theta\in\Theta\subset\R^d$, use $\z$ (or $\zz$ when random) to denote the
data sample supported on a space $\ZZ$, and $\f\colon \Theta\times\ZZ\to\R$ for
the loss function. Gradients (denoted $\nabla$) are always taken with respect to
the optimization variable $\theta$. For a convex set $\mc{C}$, $\Pi_\mc{C}$
denotes the euclidean projection on $\mc{C}$, i.e.\
$\Pi_\mc{C}(y) \defeq \argmin_{z\in\mc{C}}\norms{y-z}_2$. We use $\msf{A}$ to
refer to (possibly random) private mechanisms and $X^n$ as a shorthand for the
dataset $(X_1, \ldots, X_n)$. For two distributions $P$ and $Q$, we denote by
$\dtv{P}{Q}$ their total variation distance and $\dkl{P}{Q}$ their
Kullback-Leibler divergence. For a random vector $X \sim P$ supported on 
$\R^d$, we use $\var(P)$ or $\var(X)$ to denote $\EE\left[\|X - 
\EE[X]\|_2^2\right]$, which is equal to the trace of the covariance matrix of 
$X$.

Next, we consider differential privacy in the most general way, which only
requires specifying a dataset space $\mathbb{S}$ and a distance $\mathrm{d}$ on
$\mathbb{S}$.
\begin{definition}[Differential Privacy] \label{def:user_dp} Let $\eps, \delta 
\ge 0$. Let
  $\A\colon \mathbb{S}\to\Theta$ be a (potentially randomized) 
mechanism. We say
  that $\A$ is $(\eps, \delta)$-DP with respect to $\mrm{d}$ if for any
  measurable subset $O \subset \Theta$ and all $S, S'\in\mathbb{S}$ 
satisfying
  $\mrm{d}(S, S') \le 1$,
\begin{equation}
\probof{\A(S) \in O} \le e^\eps\probof{\A(S') \in O} + \delta.
\end{equation}
If $\delta=0$, we refer to this guarantee as pure differential privacy.
\end{definition}
For a data space $\ZZ$, choosing $\mathbb{S} = \ZZ^n$ and
$\mrm{d}(S, S') = \dham(S, S') = \sum_{i = 1}^n \indic{z_i \neq z'_i}$ 
recovers
the canonical setting considered in most of the literature---we refer to 
this as
\emph{item-level} differential privacy. When we wish to guarantee privacy 
for \emph{users} rather than
individual samples, we instead assume a \emph{structured} dataset
into which each of $n$ users contributes $m > 1$ samples. This 
corresponds to
$\mathbb{S} = \prn{\ZZ^m}^n$ such that for $\cS \in \mathbb{S}$, we 
have
\begin{equation*}
\cS = \prn{S_1, \ldots, S_n}, \mbox{~where~} S_u = \crl*{\zu_1, \ldots, 
\zu_m}%
\mbox{~and~}
\mrm{d}_{\msf{user}}(\cS, \cS') \defeq \sum_{u = 1}^n
\indic{S_u \neq S'_u},%
\end{equation*}
which means that, in this setting, two datasets are neighboring if at most 
one
of the user's contributions differ. We henceforth refer to this setting as
\emph{user-level} differential privacy.

\paragraph{Distributional assumptions.} In the case of user-level privacy with
$n$ users each providing $m$ samples, we assume existence of a collection of
distributions $\crl{P_u}_{u\in\brk{n}}$ over $\mc{Z}$. One then observes the
following user-level dataset\footnote{For simplicity, we assume that
  $\abs{S_u} = m$ but our guarantees directly extend to the setting where users
  have different number of samples with $m$ replaced by
  $\mathsf{median}(m_1, \ldots, m_n)$ using techniques
  from~\cite{LiuSuYuKuRi20}. We leave eliciting the optimal rates in settings
  when $m_u$ is an arbitrary random variable to future work.}
\begin{equation}\label{eq:sampling}
  \mc{S} = \prn*{S_1, \ldots, S_n} \mbox{~~where~~} S_u \simiid P_u.
\end{equation}
In this paper, we consider the \emph{limited heterogeneity} setting, i.e.\ when
the users have related distributions. This setting is more reflective of
practice, especially in light of growing interest towards federated learning
applications~\cite{kairouz2019advances, WoodworthPaSr20}. 
\begin{assumption}[Limited heterogeneity setting]\label{ass:heterogeneous} There
  exists a distribution $P_0$  over $\mc{Z}$ such that all the user 
  distributions
  are close to $P_0$ in total variation distance, i.e.\
  \begin{equation*}
    \max_{u\in\brk{n}} \dtv{P_u}{P_0} \le \hetero,
  \end{equation*}
  where $\hetero \ge 0$ quantifies the level of heterogeneity. Note that
  $\hetero = 0$ corresponds to assumption~\ref{ass:homogeneous}.
\end{assumption}

Note that our TV-based definition is natural in this setting as it is closely
related to the notion of \emph{discrepancy} (or \emph{$d_A$ distance}) which
plays a key role in domain adaption scenarios
\citep{MansourMohriRostamizadeh2009, BenDavidBlitzerCrammerPereira2007}. Lower
bound results have been given in terms of the discrepancy measure (see
\citep{Ben-DavidLuLuuPal2010}), which further justify the adoption of this
definition in the presence of multiple distributions.

In the case that $\hetero=0$, \ref{ass:heterogeneous}~reduces to the standard
\emph{homogeneous setting}. Many fundamental papers choose this setting when
explicating minimax rates under constraints (e.g.\ in distributed optimization
and federated learning~\cite{WoodworthWaSmMcSr18} or under communication
constraints~\cite{DuchiJoWaZh14, BravermanGaMaNgWo16}).
\begin{assumption}[Homogeneous setting]\label{ass:homogeneous} The distributions
  of individual users are equal, meaning there exists $P_0$ such that for all
  $u\in\brk{n}$, $P_u = P_0$.
\end{assumption}

In this paper, we develop techniques and provide matching upper and lower bounds
for solving learning tasks in the homogeneous setting. In
Appendix~\ref{sec:extension}, we prove that our techniques naturally apply to
the heterogeneous setting in a black-box fashion,
and for all considered problems provide meaningful 
guarantees under Assumption~\ref{ass:heterogeneous}. Moreover, the 
algorithm achieves almost optimal 
rate whenever $\hetero$
is (polynomially) small. See the detailed statement in 
Theorem~\ref{thm:reduction}.%

\subsection{ERM and stochastic convex optimization}

\paragraph{Assumptions on the loss.} Throughout this work, we assume that the
parameter space $\Theta$ is closed, convex, and satisfies
$\norm{\theta - \vartheta}_2 \le R$ for all $\theta, \vartheta \in\Theta$. We
also assume that the loss $\f\colon \Theta \times \ZZ \to \R$ is $G$-Lipschitz
w.r.t. the $\ell_2$-norm\footnote{It is straightforward to develop analogs of
  the results of Sections~\ref{sec:sq} and \ref{sec:erm} for arbitrary norms,
  but we restrict our attention to the $\ell_2$ norm in this work for
  clarity.}, meaning that for all $\z \in \ZZ$, for all $\theta\in\Theta$,
$\norm{\nabla \f(\theta;\z)}_2 \le G$. We further consider the following
assumptions.

\begin{assumption}\label{ass:smoothness} The function $\f(\cdot;\z)$ is
  $\smooth$-smooth.
  In other words, the gradient $\nabla \f(\theta;\z)$
  is $\smooth$-Lipschitz in the variable $\theta$ for all $\z\in\ZZ$.
\end{assumption}
\begin{assumption}\label{ass:subG} The random vector $\nabla \f(\theta;Z)$ is
  $\sigma^2$-sub-Gaussian for all $\theta \in \Theta$ and $Z \sim
  P_0$. Equivalently, for all $v \in \R^d$,
  $\tri{v, \nabla \f(\theta;\zz)}$ is a $\sigma^2$-sub-Gaussian random variable, i.e.,
  \[
    \E\left[ \exp\prn*{\tri*{v, \nabla \f(\theta;\zz) -
          \E[\nabla \f(\theta;\zz)]}}\right] \le \exp\prn*{\norms{v}_2^2
        \sigma^2/2}.
  \]
\end{assumption}
In this work, our rates often depend on the sub-Gaussianity and Lipschitz
parameters $\sigma$ and $G$, and thus
we define the shorthands $\Gsig \defeq \sigma\sqrt{d}$ and
$\Gbar \defeq \min\crl{G, \Gsig}$. %
Intuitively, the $G$-Lipschitzness assumption bounds the gradient in a ball
around $0$ (independently of $\theta$), while sub-Gaussianity implies that, for
each $\theta$, $\nabla \f(\theta;Z)$ likely lies in
$\mathbb{B}_2^d(\nabla \ff(\theta;P_0), \Gsig)$. Generically, there is no ordering
between $G$ and $\Gsig$: for linear loss
$\f(\theta;z) = \tri{\theta, z}$, depending on $P_0$, it can hold that
$G \ll \Gsig$ (e.g., $P_0 = \msf{Unif}\crl{-v, v}$ for $v\in\R^d$), $\Gsig \ll G$
(e.g., $P_0$ is $\msf{N}(\mu, \sigma^2I_d)$ truncated in a ball around $\mu$, with
$\norm{\mu}_2 \gg \sigma\sqrt{d}$) or $G \approx \Gsig$ (e.g.,
$P_0 = \msf{Unif}\crl{-1, +1}^d$).

We introduce the tasks we consider in this work, namely empirical risk
minimization (ERM) and stochastic convex optimization (SCO). For a collection of
samples from $n$ users $\cS = \prn{S_1, \ldots, S_n}$, where each
$S_u = \crl{\zu_1, \ldots, \zu_m} \in \ZZ^m$, we define the empirical 
risk objectives
\begin{equation}\label{eq:erm-user}
  \ff(\theta;S_u) \defeq \frac{1}{m} \sum_{i =1}^m \f(\theta;\zu_i) \mbox{~~and~~}
  \ff(\theta;\cS) \defeq \frac{1}{n} \sum_{u = 1}^n
  \ff(\theta;S_u) = \frac{1}{mn}\sum_{u = 1}^n \sum_{i = 1}^m 
  \f\prn{\theta;\zu_i}.
\end{equation}
In the user-level setting we wish to minimize $\ff(\theta;\cS)$ under user-level
privacy constraints. Going beyond the empirical risk, we also solve
SCO~\citep{ShalevShwartz2009StochasticCO}, i.e.\ minimizing a convex population
objective when provided with samples from each users' distributions. In
the user-level setting, for a convex loss $\f$ and a convex constraint set
$\Theta$, we observe
$\cS = (S_1, \ldots, S_n) \sim \otimes_{u \in \brk{n}}(P_u)^m$ and wish to
\begin{equation}\label{eq:pop_hetero}
  \minimize_{\theta\in\Theta} \, \frac{1}{n}\sum_{u\in\brk{n}} \ff(\theta;P_u) \defeq
  \frac{1}{n} \sum_{u\in\brk{n}} \E_{P_u}\brk{\f(\theta;\zz)}.%
\end{equation}

In the homogeneous case (Assumption~\ref{ass:homogeneous}), this reduces to the
classic SCO setting:
\begin{equation}\label{eq:pop}
\minimize_{\theta\in\Theta} \, \ff(\theta;P_0) \defeq
 \E_{P_0}\brk{\f(\theta;\zz)}.%
\end{equation}
%
%
%
%
%
%
%
%
%
%
%
%
%
%
%
%
%
%
%
%
%
%

\subsection{Uniform concentration of queries}
Let $\phi:  \cZ \rightarrow \RR^d$  be a $d$-dimensional query function. 
We define 
concentration of 
random 
variables and
uniform concentration of multiple queries as follows.

\begin{definition}\label{def:concentrated} A (random) sample $X^n$ supported on
	$\brk{-B, B}^d$
	is $(\tau, \gamma)$-concentrated (and we call $\tau$ the ``concentration radius'') if there exists $x_0 \in [-B, B]^d$ such that with
        probability at least $1 -
        \gamma$, %
	\[
	\max_{i \in [n]} \norm{X_i - x_0}_2 \le \tau.
	\]
\end{definition}

\begin{definition}[Uniform concentration of vector
  queries] \label{def:uniform_concentration} Let
  $\cQ_B^d = \{ \phi\colon \cZ \rightarrow [-B, B]^d\}$ be a family of queries
  with bounded range. For $Z^n = (Z_1, \ldots, Z_n) \simiid P$, we say that
  $(Z^n, \cQ_B^d)$ is $(\tau, \gamma)$-uniformly-concentrated if with probability at least $1-\gamma$, we have
	\[
	\max_{i \in [n]} \sup_{\phi \in \cQ_B^d} \, \norm[\Big]{\phi(Z_i) - \EE_{Z
				\sim P}[\phi(Z)]}_2 \leq \tau.
	\]
\end{definition}
In this work, we will often consider $\sigma^2$-sub-Gaussian random variables
(or vectors), which are concentrated according to
Definition~\ref{def:concentrated}. For example, if $X^n$ is drawn i.i.d. from a
$\sigma^2$-sub-Gaussian random vector supported on $\brk{-B, B}^d$, then it is
$(\sigma \sqrt{d \log(2 n/\gamma)}, \gamma)$-concentrated around its mean (see,
e.g.,~\cite{Vershynin19}).
Finally, we define a distance between random variables (and estimators).
\begin{definition}[$\beta$-close Random Variables]
  For any two random variables $X_1\sim P_1$ and $X_2\sim P_2$, we say $X_1$ and
  $X_2$ are $\beta$-close, if $\dtv{P_1}{P_2} \le \beta$. %
  We use the notation
  $X_1 \sim_{\beta} X_2$ if $X_1$ and $X_2$ are $\beta$-close.
\end{definition}
$\beta$-closeness is useful as, in many of our results, the private estimator
we propose returns a simple unbiased estimate with high
probability %
and is bounded otherwise. Thus, it suffices to do the analysis in the ``nice''
case and crudely bound the error otherwise.
\section{High Dimensional Mean Estimation and Uniformly Concentrated 
Queries} 
\label{sec:sq}
In this section, we present a private mean estimator with
privacy cost proportional to the concentration radius. Using these techniques,
we show 
that, under uniform concentration, we answer adaptively-chosen queries with
privacy cost proportional to the concentration radius instead of the whole
range. Our theorems guarantee that the estimator is $\beta$-close (with $\beta$
exponentially small in $n$) to a simple unbiased estimator with small
noise. \arxiv{This formulation makes the analysis much simpler when applying the
technique to gradient methods. }We further show how to directly translate these
results into bounds on the estimator error, which we demonstrate by
providing tight bounds on estimating the mean of $\ell_2$-bounded
random vectors under user-level DP constraints 
(Corollary~\ref{coro:bounded}).

Given i.i.d samples $X^n$ from a distribution $P$ supported on $\RR^d$ with mean
$\mu$, the goal of mean estimation is to design a private estimator that
minimizes the $\expectation{\norm{\alg(X^n) - \mu}_2^2}$. We focus on distributions 
with bounded supprot $[-B, B]^d$. However, our algorithm also generalize to the case 
when the mean is guaranteed to be in $[-B, B]^d$. In
the user-level setting (in the homogeneous case), one observes a dataset $\cS$
sampled as in~\eqref{eq:sampling} and wishes to minimize
$\E\brk{\norm{\alg(\cS)- \E P_0}_2^2}$ under user-level privacy constraints.
       %
%
%
%
%
%
%
%
%
%
%
%
%
%
%
%
%
%
%
%
%
%
%
%
%
%
%
%
%
%
%
%
%
%
%
%
%
%
%
%
%
%
%
%
%
%
%
%
%
%
%
%
%
%
%
%
%
%
%
We first focus on the scalar case. %

\begin{algorithm}[h]
	\caption{\textbf{WinsorizedMean1D($X^n, \eps, \tau, B$)}: Winsorized 
		Mean Estimator (WME)}
	\begin{algorithmic}[1] \label{alg:winsorized} \REQUIRE
		$X^n := (X_1, X_2, ..., X_n) \in [-B, B]^n$, $\tau:$ concentration
		radius, privacy parameter $\eps > 0$.
		\STATE $[a, b] = \textbf{PrivateRange}(X^n, \eps/2, \tau, B)$ with
		$|b - a| = 4\tau$. \hfill
		\COMMENT{Algorithm~\ref{alg:priv_range}
			in
				Appendix~\ref{sec:private-range}.
			}
		\STATE Sample
		$\xi \sim \lap{0, \frac{8\tau}{\eps n}}$ and return
		\[
		\bar{\mu} = \frac{1}{n} \sum_{i = 1}^n \Pi_{[a, b]}(X_i) + \xi,
		\]
		where $\Pi_{[a, b]}(x) = \max \{ a, \min \{x , b\}\}$.
	\end{algorithmic}
\end{algorithm}

\paragraph{Mean estimation in one dimension.} The algorithm uses a two-stage
procedure, similar in spirit to those of \cite{smith2011privacy},
\cite{karwa2017finite}, and \cite{kamath2020private}. In the first stage of this
procedure, we use the approximate median estimation in~\cite{feldman2017median},
detailed in Algorithm~\ref{alg:priv_range}
in Appendix~\ref{sec:private-range}, 
to privately estimate a crude interval in
which the means lie, with accuracy $\Theta(\tau)$. The second stage clips the
mean around this interval, reducing the sensitivity from $O(B)$ to $O(\tau)$,
and adds the appropriate Laplace noise. With high probability, we can recover
the guarantee of the Laplace mechanism with smaller sensitivity since the
samples are concentrated in a radius $\tau$. We present the formal guarantees of
Algorithm~\ref{alg:winsorized} in Theorem~\ref{thm:winsorized} and defer its
proof to Appendix~\ref{proof:thm-winsorized}.

\begin{restatable}{theorem}{restateWinsorized}\label{thm:winsorized}
  Let $X^n$ be a dataset supported on $[-B, B]$. The output of
  Algorithm~\ref{alg:winsorized}, denoted by $\alg(X^n)$, is
  $\eps$-DP. Furthermore, if $X^n$ is $(\tau, \gamma)$-concentrated, it holds
  that
  \[
    \alg(X^n) \sim_{\beta} \frac{1}{n}\sum_{i = 1}^n X_i + \lap{\frac{8\tau}{n
        \eps}},
  \]
  where
  $\beta = \min \left\{1, \gamma + \frac{B}{\tau} \exp \Paren{-\frac{n\eps}{8}} \right\}$.
  Moreover, Algorithm~\ref{alg:winsorized} runs in time
  $\tilde{O}(n + \log(B/\tau))$.
\end{restatable}
Compared to~\citep{karwa2017finite, kamath2019privately, 
kamath2020private}, our algorithm runs in time $\tilde{O}(n + \log(B/\tau))$
instead of $\tilde{O}(n + B/\tau)$ owing to the approximate median estimation
algorithm in~\cite{feldman2017median}, which is faster when $\tau \ll B$. 
\paragraph{Mean estimation in arbitrary dimension.} In the general
$d$-dimensional case, if $X^n$ is concentrated in $\ell_\infty$-norm, one
simply applies Algorithm~\ref{alg:winsorized} to each dimension. However, when
$X^n$ is concentrated in $\ell_2$-norm, naively upper bounding
$\ell_\infty$-norm by the $\ell_2$-norm will incur a superfluous $\sqrt{d}$
factor: if $\norm{v}_2 \le \rho$, each $\abs{v_j}$ is possibly as large as
$\rho$. To remedy this issue, we use the random rotation trick
in~\cite{ailon2006approximate, pmlr-v70-suresh17a}.  This guarantees that all
coordinates have roughly the same range: for $v \in \R^d$, with high
probability, $\norm{Rv}_\infty \leq \Olog(\norm{v}_2 / \sqrt{d})$, where $R$ is
the random rotation. We present this procedure in
Algorithm~\ref{alg:winsorized_highd} and its performance in
Theorem~\ref{thm:winsorized_highd}.

\begin{algorithm}[h]
	\caption{\textbf{WinsorizedMeanHighD($X^n, \eps, \delta, \tau, B, 
	\gamma$)}:
		WME - High Dimension}
	\begin{algorithmic}[1]
          \REQUIRE $X^n := (X_1, X_2, ..., X_n), X_i \in [-B, B]^d$,
          $\tau, \gamma$: concentration radius and probability, privacy
          parameter $\eps, \delta> 0$.
		\STATE Let $D = \mathsf{Diag}(\omega)$ where $\omega$ is sampled 
		uniformly
		from $\crl{\pm 1}^{d}$.  \STATE\label{line:random-rot} Set $U = 
		d^{-1/2}\whmat D$,
		where $\whmat$ is a $d$-dimensional Hadamard matrix. For all $i \in [n]$, 
		compute
		\vspace{-8pt}
		\[ Y_i = U X_i.
		\] 
		\vspace{-15pt} \STATE Let
                $\eps' = \frac{\eps}{\sqrt{8d\log(1/\delta)}}, \tau' = 10 \tau
                \sqrt{\frac{\log(dn/\gamma)}{d}}$. For $j \in [d]$, compute
                \vspace{-10pt}
		\[\bar{Y}(j) = \textbf{WinsorizedMean1D}\prn*{\crl{Y_i(j)}_{ i
				\in[n]}, \eps', \tau', \sqrt{d}B}.
		\] 
		\vspace{-15pt}
		\RETURN $\bar{X} = U^{-1} \bar{Y}$.
	\end{algorithmic}
	\label{alg:winsorized_highd}
      \end{algorithm}
      
\begin{restatable}{theorem}{restateWinsorizedhighd}\label{thm:winsorized_highd}
  Let
  $\alg(X^n) = \textbf{WinsorizedMeanHighD} (X^n, \eps, \delta, \tau, B,
  \gamma)$ be the output of Algorithm~\ref{alg:winsorized_highd}.  $\A(X^n)$ is
  $(\eps, \delta)$-DP. Furthermore, if $X^n$ is $(\tau, \gamma)$-concentrated in
  $\ell_2$-norm, there exists an estimator $\alg'(X^n)$ such that
  $\alg(X^n) \sim_{\beta} \alg'(X^n)$ and
  \begin{equation}
    \E\brk*{\alg'(X^n) | X^n} = \frac1n \sum_{i = 1}^n X_i \mbox{~~and~~}
    \Var\prn*{\alg'(X^n) | X^n} \le c_0\frac{d \tau^2 \log(dn/\alpha)%
    \log(1/\delta)}{n^2 \eps^2},
  \end{equation}
  \vspace{-8pt}
  where $c_0 = 102,400$ and
  $\beta = \min\crl*{1, 2\gamma + \frac{d^2 B \sqrt{\log(dn/\gamma)}}{\tau}
  \exp\prn*{-\tfrac{n\eps}{24\sqrt{d \log(1/\delta)}}}}$.
\end{restatable}

We present the proof of Theorem~\ref{thm:winsorized_highd} in
Appendix~\ref{proof:winsorized_highd}. We are able to transfer both
Theorem~\ref{thm:winsorized} and Theorem~\ref{thm:winsorized_highd} into
finite-sample estimation error bounds for various types of concentrated
distributions and obtain near optimal guarantees (see Appendix~\ref{proof:subg}
for an example in mean estimation of sub-Gaussian distributions). The next
corollary characterizes the risk of mean estimation for distributions supported
on an $\ell_2$-bounded domain with user-level DP guarantees (see
Appendix~\ref{proof:bounded} for the proof).
\begin{corollary}\label{coro:bounded}
  Assume~\ref{ass:homogeneous} holds with $P_0$ supported on
  $\mathbb{B}_2^d(0, B)$ with mean $\mu$. Given $\cS = (S_1, S_2, ..., S_n)$, $|S_u| = m$, consisting
  of $m$ i.i.d.  samples from $P_u$. There exists an $(\eps, \delta)$-user-level
  DP algorithm $\alg(\cS)$ such that, if
  $n \ge (c_1 \sqrt{d\log(1/\delta)}/\eps) \log( m(dn + n^2\eps^2))$ for a
  numerical constant $c_1$, we have\footnote{For precise log factors, see
    Appendix~\ref{proof:bounded}.}
	\[
		\expectation{\norm{\alg(\cS) - \mu}_2^2} = 
		\frac{\Var(P_0)}{mn} + \tilde{O}\Paren{ 
		\frac{dB^2}{mn^2\eps^2}}.
	\]
	Note that $\Var(P_0) \le B^2$ for any $P_0$ supported on
	$\mathbb{B}_2^d(0, B)$. Replacing $\Var(P_0)$ by $B^2$, the bound is 
	minimax optimal up to logarithmic 
	factors. When
        only~\ref{ass:heterogeneous} holds with
        $\hetero \le \poly(d, \tfrac{1}{n}, \tfrac{1}{m}, \tfrac{1}{\eps})$, the
        same error bounds holds (up to constant) for estimating
        $\E_{Z \sim P_u}\brk{Z}$ for any $u\in\brk{n}$.
\end{corollary}
Note that algorithms in \cite{kamath2019privately, 
		kamath2020private}, which focus on estimating the mean of 
		$d$-dimensional subGaussian distributions, can also be used to 
		estimate the mean of $\ell_2$-bounded distributions since bounded 
		random variables are also subGaussian. However, applying these 
		algorithms directly will incur a superfluous $d$ factor in the mean 
		square error. We void this using the random rotation trick in 
		Algorithm~\ref{alg:winsorized_highd}.

\arxiv{%
\subsection{Uniform concentration: answering many queries privately} 
\label{sec:uc}

The statistical query framework subsumes many learning algorithms. For example,
we easily express stochastic gradient methods for solving ERM in the
language of SQ algorithms (see beginning of
Section~\ref{sec:erm}). In the next theorem, we show that with a uniform
concentration assumption we can answer a sequence of adaptively chosen
queries with variance---or, equivalently, privacy cost---proportional to the
concentration radius of the queries instead of the full range.

\begin{theorem} \label{thm:sq_highd} If $(Z^n, \cQ_B^d)$ is
	$(\tau, \gamma)$-uniformly concentrated, then for any sequence of
	(possibly adaptively chosen) queries
	$\phi_1, \phi_2, ..., \phi_K \in \cQ_B^d$, there exists an
	$(\eps, \delta)$-DP algorithm $\alg$, such that 
	$\alg$ outputs $v_1, v_2, ..., v_K$ satisfying
	$(v_1, v_2, ..., v_K) \sim_{\beta} (v'_1, v'_2, ..., v'_K)$, where
	\[
	\expectation{v'_k|Z^n} = \frac{1}{n}\sum_{i = 1}^n \phi_k(Z_i)
	\mbox{~~and~~} \Var\prn*{v'_k|Z^n} \le \frac{8 c_0 d K\tau^2
		\log(Kdn/\gamma)
		\log^2(2K/\delta)}{n^2\eps^2} = 
		\tilde{O}\prn*{\frac{dK\tau^2}{n^2\eps^2}},
	\]
	where $c_0 = 102400$ and
        $\beta = \min \crl*{1, 2 \gamma + \frac{d^2 K B
            \sqrt{\log(dKn/\gamma)}}{\tau} \exp 
            \prn*{-\frac{n\eps}{48\sqrt{2dK\log(2/\delta)
                \log(2K/\delta)}}}}$.
\end{theorem}

The algorithm for Theorem~\ref{thm:sq_highd} is simply applying
Algorithm~\ref{alg:winsorized_highd} to $\crl{\phi_k(Z_i)}_{i\in[n]}$ with
$\eps_0 = \frac{\eps}{2\sqrt{2K\log(2/\delta)}}$ and
$\delta_0 = \frac{\delta}{2K}$ for each query. Algorithm~\ref{alg:wgd} is an
illustration of an application of this result.\arxiv{ The proof is given in
Appendix~\ref{proof:sq_highd}.}

}

\notarxiv{\paragraph{Answering multiple queries.} We end this section by noting
  that, when a family of queries $\mc{Q}$ is uniformly concentrated (as made
  precise in Definition~\ref{def:uniform_concentration}), we answer sequences of
  $K$ $d$-dimensional, adaptively chosen queries with error scaling as
  $\tilde{O}(\sqrt{dK}\tau / (n\eps))$ by applying
  Algorithm~\ref{alg:winsorized_highd} to $\crl{\phi_k(Z_i)}_{i\in[n]}$ with the
  right $(\eps_0, \delta_0)$. We make this formal in Theorem~\ref{thm:sq_highd}
  in Appendix~\ref{sec:uc}.}
\section{Empirical Risk Minimization with User-Level Differential Privacy}
\label{sec:erm}

In this section, we present an algorithm to solve the ERM objective
of~\eqref{eq:erm-user} under user-level DP constraints. We apply the results of
Section~\ref{sec:sq} by noting that the SQ framework encompasses stochastic
gradient methods. Informally, one can sequentially choose queries
$\phi_k(z) = \nabla \f(\theta_k;z)$ and, for a stepsize $\eta$, update
$\theta_{k+1} = \Pi_\Theta(\theta_k - \eta v_k)$, where $v_k$ is the answer to
the $k$-th query. For the results to hold, we require a uniform concentration
result over the appropriate class of queries.

\paragraph{Uniform concentration of stochastic gradients} The class of queries
for stochastic gradient methods is
$\cQ_\msf{erm} \defeq \crl{ \nabla \f(\theta;\cdot) : \theta \in \Theta }$.  We
prove that when assumptions~\ref{ass:smoothness} and \ref{ass:subG} hold,
$(\{\nabla \f(\cdot; S_u)\}_{u \in [n]}, \cQ_\msf{erm})$ is
$(\tilde{O}(\sigma\sqrt{d/m}), \alpha)$-uniformly concentrated. The next
proposition is a simplification of the result of~\cite{mei2018landscape} under
the (stronger) assumption~\ref{ass:smoothness} that $\f$ is uniformly
$\smooth$-smooth. The proof, which we defer to Appendix~\ref{app:uc}, hinges on
a covering number argument.

\begin{restatable}[Concentration of random gradients]{proposition}
  {restateUCGrad}\label{prop:uc-grad} Let $S_u \simiid P_u$, $|S_u| = m$ 
  for $u\in\brk{n}$ and
  $\alpha \ge 0$. Under Assumptions~\ref{ass:smoothness} and \ref{ass:subG},
  with probability greater than $1-\alpha$ it holds that
  \begin{equation} 
    \max_{u\in\brk{n}}\sup_{\theta \in \Theta} \norm{\nabla \ff (\theta; S_u)
      - \nabla \ff(\theta;P_u)}_2 =
    O\Paren{\sigma\sqrt{\frac{d \log\prn*{\frac{RHm}{d\sigma}}}{m}
        + \frac{\log\prn*{\frac{n}{\alpha}}}{m}}}. \nonumber
  \end{equation}
\end{restatable}

\paragraph{Stochastic gradient methods} We state classical
convergence results for stochastic gradient methods for both convex and
non-convex losses under smoothness. For a function $F:\Theta\to\R$, we assume
access to a first-order stochastic oracle $\msf{O}_{F, \nu^2}$, i.e., a random
mapping such that for all $\theta \in \Theta$,
\begin{equation} \nonumber 
  \msf{O}_{F, \nu^2}(\theta) = \nabla\what{F}(\theta) \mbox{~~with~~}
  \E\brk*{\nabla\what{F}(\theta)} = \nabla F(\theta) \mbox{~~and~~}
  \Var\prn*{\nabla\what{F}(\theta)} \le \nu^2.
\end{equation}

We abstract optimization algorithms in the following way: an algorithm consists
of an output set $\mc{O}$, a sub-routine
$\msf{Query}:\mc{O}\to\Theta$ that takes the last output and indicates the next
point to query and a sub-routine
$\msf{Update}: \mc{O}\times\R^d\to\mc{O}$ that takes the previous output and a
stochastic gradient and returns the next output. After $T$ steps, we call
$\msf{Aggregate}: \mc{O}^* \to \Theta$, which takes all the previous outputs and
returns the final point. (See Algorithm~\ref{alg:gen-opt} in
Appendix~\ref{app:conv-sgd} for how to instantiate generic first-order
optimization in this framework.) \arxiv{For example, the classical projected
SGD algorithm with fixed stepsize $\eta > 0$ corresponds to the following
(with $o_0 = \crl{\theta_0}$):
\begin{equation*}
  \msf{Query}(\crl{\theta_s}_{s\le t}) = \theta_t,
  \, \msf{Update}(o_t, g_t) = o_t \cup \crl{\Pi_\Theta(\theta_t - \eta g_t)} \mbox{~and~}
  \msf{Aggregate}(\crl{o_t}_t) = \tfrac{1}{T}\sum_{\theta_t \in o_T}\theta_t.
\end{equation*}}
We detail in Proposition~\ref{prop:conv-sgd} in Appendix~\ref{app:conv-sgd}
standard convergence results for variations \arxiv{\footnote{For convex
    functions, the algorithm is fixed-stepsize, averaged, projected SGD. For
    strongly-convex functions, the algorithm consists of projected SGD with a
    fixed stepsize and non-uniform averaging followed by a single restart with
    decreasing stepsize. Finally, in the non-convex case, the \textsf{Query} and
    \textsf{Update} sub-routine are also projected SGD with fixed stepsize while
    the \textsf{Aggregate} selects one of the past iterates uniformly at
    random.}}  of (projected) stochastic gradient descent (SGD). We introduce
this abstraction to forego the details of each specific algorithm and instead
focus on the privacy and utility guarantees. \arxiv{We also note the progress
  in recent years to improve convergence rates in the settings we consider (see,
  e.g.,~\citep{Lan12, GhadimiLa12, GhadimiLa13, KulunchakovMa19, Allen17,
    Allen18}). While our privacy guarantees hold for these algorithms---and thus
  practitioners can benefit from the improvements---we do not focus on them as
  projected SGD already allows us to attain the minimax optimal rate for SCO
  (see Section~\ref{sec:sco}).}
\paragraph{Algorithm} We recall the ERM setting with user-level DP. We observe
$\mc{S} = \prn{S_1, \ldots, S_n}$ with $S_u \in \ZZ^m$ for $u\in\brk{n}$
and wish to solve the constrained optimization problem with objective
in~\eqref{eq:erm-user}. 
\arxiv{\begin{equation}\label{eq:min-erm-user} 
  \minimize_{\theta\in\Theta} \, \ff(\theta; \cS) := \frac{1}{mn}
  \sum_{u\in\brk{n}, j\in\brk{m}}
  \f\prn{\theta;z_j^{(u)}}.
\end{equation}}
We present our method in Algorithm~\ref{alg:wgd} and provide utility and privacy
guarantees in Theorem~\ref{thm:erm}.

\begin{algorithm}\label{alg:sgd-erm}
  \caption{Winsorized First-Order Optimization}\label{alg:wgd}
  \begin{algorithmic}[1]
    \STATE \textbf{Input:} Number of iterations $T$, optimization algorithm
    $\crl{\mc{O}, \msf{Query}, \msf{Update}, \msf{Aggregate}}$, privacy
    parameters $(\eps, \delta)$, data $\mc{S} = \prn{S_1, \ldots, S_n}$, initial
    output $o_0$,
    parameter set $\Theta$, concentration radius 
    $\tau$, probability $\gamma$.
    \STATE Set $\eps' = \frac{\eps}{2\sqrt{2T\log(2/\delta)}}$
    and $\delta' = \frac{\delta}{2T}$
    \FOR {$t=0,\ldots,T-1$}
    \STATE
    $\theta_t \gets \msf{Query}(o_t)$.
    \STATE For each user $u \in [n]$,
    compute      \vspace{-8pt}
    \begin{equation*}
      g_t^{(u)} = \nabla \ff(\theta_t;S_u) = \frac1m \sum_{j\in \brk{m}} \nabla \f(\theta_t;z_j^{(u)}).
      \vspace{-15pt}
    \end{equation*}
          \label{step:grad_computation}
    \STATE Compute
    $\bar{g}_t =
    \textbf{WinsorizedMeanHighD}\prn{\crl{g_{t}^{(u)}}_{u\in[n]},\eps',
      \delta', \tau, 
      \normgrad, \gamma}$. \label{step:mean_estimation}
   \STATE $o_{t+1} \gets \msf{Update}(o_t, 
      \bar{g}_t)$. 
    \ENDFOR \RETURN $\bar{\theta} \gets \msf{Aggregate}(o_0, \ldots, o_T)$.
  \end{algorithmic}
\end{algorithm}

\begin{restatable}[Privacy and utility guarantees for ERM]{theorem}
  {restateERM}\label{thm:erm} Assume~\ref{ass:homogeneous} holds and recall that
  $\Gsig = \sigma\sqrt{d}$, assume\footnote{For precise log factors, see
    Appendix~\ref{app:proof-thm-erm}.} $n = \tilde{\Omega}(\sqrt{d T}/\eps)$
  and let $\what{\theta}$ be the output of Algorithm~\ref{alg:wgd}. There exists
  variants of projected SGD (e.g. the ones we present in
  Proposition~\ref{prop:conv-sgd}) such that,
  with probability greater than $1-\gamma$:
  \begin{enumerate}[label=(\roman*)]
  \item If for all $\z \in \cZ, \f(\cdot;\z)$ is convex, then
    \begin{equation}
      \E\brk*{\ff(\what{\theta};\mc{S}) - \inf_{\theta' \in \Theta}\ff(\theta';\mc{S}) ~\middle|~ \mc{S}
      }
      = \Olog\Paren{\frac{R^2H}{T} + 
      R\Gsig\frac{\sqrt{d}}{n\sqrt{m}\eps}}.\nonumber 
    \end{equation}
  \item If for all $\z \in \cZ, \f(\cdot;\z)$ is $\mu$-strongly-convex, then
    \begin{equation}
      \E\brk*{\ff(\what{\theta};\mc{S}) - \inf_{\theta' \in \Theta}\ff(\theta';\mc{S}) ~\middle|~ \mc{S}}
      = \Olog\Paren{GR\exp\prn*{-\tfrac{\mu}{H}T}
        + \Gsig^2\frac{d}{\mu n^2m\eps^2}}.\nonumber 
    \end{equation}
  \item Otherwise, defining the \emph{gradient mapping}\footnote{In the
      unconstrained case---$\Theta=\R^d$---this corresponds to an
      $\epsilon$-stationary point as $\msf{G}_{F, \gamma}(x) = \nabla F(x)$.}
    $\msf{G}_{F, \gamma}(\theta) \defeq \frac{1}{\gamma}\brk*{\theta -
      \Pi_{\Theta}\prn* {\theta - \gamma\nabla F(\theta)}}$, we have
    \begin{equation}
      \E\brk*{\norm{\msf{G}_{\ff(\cdot; \mc{S}), 1/H}(\what{\theta})}_2^2 | \cS} = 
      \Olog\Paren{
        \frac{H^2R}{T} + HR\Gsig\frac{\sqrt{d}}{n\sqrt{m}\eps}}.\nonumber 
    \end{equation}
  \end{enumerate}

  For $\eps \le 1, \delta > 0$, Algorithm~\ref{alg:wgd} instantiated with any
  first-order gradient algorithm is $(\eps, \delta)$-user-level DP. In the case
  that only~\ref{ass:heterogeneous} holds, the same guarantees hold whenever
  $\hetero \le \poly(d, \tfrac{1}{n}, \tfrac{1}{m}, \tfrac{1}{\eps})$.
  
\end{restatable}

We present the proof in Appendix~\ref{app:proof-thm-erm}. For the utility
guarantees, the crux of the proof resides in Theorem~\ref{thm:sq_highd}: as well
as ensuring small excess loss in expectation, the SQ algorithm produces with
high probability a sample from the stochastic gradient oracle
$\msf{O}_{\ff(\cdot;\cS), \nu^2}$ where
$\nu^2 = \Olog\prn{T\Gsig^2\tfrac{d}{n^2m\eps^2}}$. When this happens for all
$T$ steps, the analysis of stochastic gradient methods provide the
desired regret. The privacy guarantees follow from the strong composition
theorem of~\cite{dwork2010boosting}.

Importantly, when the function exhibits (some) strong-convexity (which will be
the case for any regularized objective), we are able to \emph{localize} the
optimal parameter---up to the privacy cost---in $\Olog(H/\mu)$ steps.  This will
be particularly important in Section~\ref{sec:sco}.

\begin{corollary}[Localization]\label{coro:erm_sc}
  Let $\what{\theta}$ be the output of Algorithm~\ref{alg:wgd} on the ERM
  problem of~\eqref{eq:erm-user}. Assume that $\f(\cdot;\z)$ is
  $\mu$-strongly-convex for all $z\in\cZ$, that $n = \tilde{\Omega} (\sqrt{dH/\mu})$ %
  and set
  $T = \tfrac{H}{\mu}\log\prn*{n^2 m (\Gbar/\Gsig^2) \tfrac{\mu
      R\eps^2}{d}}$ and $\gamma = \frac{\sigma^2d^2}{\mu^2n^2m\eps^2 R^2}$. For
  $\theta_{\mc{S}}^\ast \in \argmin_{\theta'\in\Theta}\ff(\theta';\cS)$, it
  holds\footnote{A logarithmic dependence on $T$ is hiding in the result. Since
    $T = \Olog(H/\mu)$, we implicitly assume $H/\mu$ is polynomial in the stated
    parameters, which is satisfied when we later apply these results to
    regularized objectives.}
  \begin{equation} \nonumber
    \E\brk{\norm{\what{\theta} - \theta_{\mc{S}}^\ast}_2^2} = \Olog\Paren
    {\frac{\sigma^2d^2}{\mu^2 n^2m\eps^2}}.
  \end{equation}
\end{corollary}

\section{Stochastic Convex Optimization with User-level 
	Privacy}
\label{sec:sco}
In this section we address the SCO task of~\eqref{eq:pop} under user-level DP
constraints. Our approach (which we show in Algorithm~\ref{alg:phased_erm})
solves a sequence of carefully regularized ERM problems, drawing on the
guarantees of the previous section.  Recall that $\Gsig = \sigma\sqrt{d}$ and
$\Gbar = \min\crl{G, \Gsig}$, and that $\f$ is $H$-smooth under
assumption~\ref{ass:smoothness}. In this section, we assume that $\f$ is
convex. We first present our results and state an upper and lower bound for SCO
with user-level privacy constraints.

\begin{restatable}[Phased ERM for
  SCO]{theorem}{restateThmSCOUpperBound} \label{thm:dp_sco_upper}
  Algorithm~\ref{alg:phased_erm} is user-level $(\eps,
  \delta)$-DP. When~\ref{ass:homogeneous} holds and
  $n = \tilde{\Omega}({\min\{ \sqrt[3]{d^2m\smooth^2 \normpara^2/(\normgrad
      \Gbar \eps^4)}, \smooth \normpara \sqrt{m} /(\sigma \eps) \}})$, or,
  equivalently,
  $\smooth = \tilde{O}(\sqrt{\frac{n^2\eps^2\sigma^2}{R^2m} + 
  \frac{\normgrad\Gbar n^3 \eps^4}{d^2R^2m}})$
  for all $P$ and $\f$ satisfying Assumptions~\ref{ass:smoothness} and~\ref{ass:subG}, we
  have %
      \vspace{-0.2cm}
  \[
    \expectation{\ff(\alg_{\sf PhasedERM}(\cS);P_0)} - \min_{\theta' \in
      \Theta}\ff(\theta'; P_0) = \tilde{O} \Paren{ \frac{R\sqrt{\normgrad\Gbar 
      }}{\sqrt{mn}} + R\Gsig\frac{\sqrt{d}}{n\sqrt{m} \eps} }.
\]
Furthermore, our results still hold in the heterogeneous setting
(Assumption~\ref{ass:heterogeneous}) whenever
$\hetero \le \poly(d, \tfrac{1}{n}, \tfrac{1}{m},
\tfrac{1}{\eps})$; the risk guarantee being with respect to any user
distribution $P_u$.
\end{restatable}
\begin{restatable}[Lower bound for
  SCO]{theorem}{restateThmSCOLowerBound} \label{thm:dp_sco_lower} There exists a
  distribution $P$ and a loss $\f$ satisfying Assumptions~\ref{ass:smoothness}
  and~\ref{ass:subG} such that for any algorithm $\alg$ satisfying
  $(\eps, \delta)$-DP at user-level, we have
  \[
    \expectation{\ff(\alg(\cS);P) }- \min_{\theta' \in \Theta}\ff(\theta'; P) =
    \Omega \Paren{ \frac{R \Gbar }{\sqrt{mn}} + R 
    \Gbar \frac{\sqrt{d}}{n\sqrt{m} \eps} }.
  \]
\end{restatable}

When $\normgrad = \Theta(\sigma \sqrt{d})$, the upper bound matches the lower
bound up to logarithmic factors. We present the algorithm and proof for
Theorem~\ref{thm:dp_sco_upper} in
Section~\ref{sec:sco_upper}. Theorem~\ref{thm:dp_sco_lower} is proved in
Section~\ref{sec:sco_lower}. %
\subsection{Upper bound: minimizing a sequence of regularized ERM
  problems} \label{sec:sco_upper}

We now present Algorithm~\ref{alg:phased_erm}, which achieves the upper bound of
Theorem~\ref{thm:dp_sco_upper}. It is similar in spirit to Phased
ERM~\citep{feldman2020private} and EpochGD~\citep{HazanKa11}, in that at each
round we minimize a regularized ERM problem with fresh samples and increased
regularization, initializing each round from the final iterate of the previous
round. This allows us to localize the optimum with exponentially increasing
accuracy without blowing up our privacy budget. We solve each round using
Algorithm~\ref{alg:wgd} to guarantee privacy and obtain an 
\emph{approximate} minimizer.
We show the guarantee in Corollary~\ref{coro:erm_sc} is enough to
achieve optimal rates. We provide the proof of 
Theorem~\ref{thm:dp_sco_upper}
in Appendix~\ref{app:sco} and present a sketch here.

\begin{algorithm}
  \caption{$\alg_{\sf PhasedERM}$: Phased ERM}
	\begin{algorithmic}[1]
          \REQUIRE Private dataset:
          $\cS =(S_1, \ldots, S_n) \in (\cZ^m)^n: n \times m$ i.i.d samples from
          $P$, %
          $\smooth$-smooth, convex loss function $\ell$, convex set
          $\Theta \subset \RR^d$,
          privacy parameters $\eps \leq 1, \delta \leq 1/n^2$, sub-Gaussian parameter $\sigma$.\\
          \STATE Set
          $T = \ceil{\log_2\prn{\tfrac{Gn\sqrt{m}\eps}{\sigma d}}}$,
          $\lambda = \sqrt{\tfrac{G\Gbar}{nm} + \tfrac{\sigma^2d^2}{n^2m\eps^2}
          } / R$ \FOR{$t=1$ to $T$\,} \STATE Set
          $n_t = \frac{n}{2^t}, \lambda_t = 4^t \lambda$.  \STATE Sample
          $\cS_t$, $n_t$ users that have not participated in previous rounds.
          Using Algorithm~\ref{alg:wgd}, compute an approximate minimizer
          $\what{\theta}_t$, to the accuracy of Corollary~\ref{coro:erm_sc}, for
          the objective\vspace{-8pt}
          \begin{equation}\label{eqn:erm_regularized}
            \vspace{-8pt}
            \ff_{\lambda_t, \what{\theta}_{t-1}}(\theta; \cS_t) = \frac{1}{m n_t} \sum_{u \in \cS_t}
            \sum_{j = 1}^m \ell(\theta, z_j^{(u)}) + \frac{\lambda_t}{2}
            \norm{\theta - \what{\theta}_{t-1}}_2^2.
            \vspace{-5pt}
          \end{equation}
          \ENDFOR

          \RETURN
          $\what{\theta}_{T}$.
	\end{algorithmic}
	\label{alg:phased_erm}
\end{algorithm}
\begin{proof}[Proof sketch of Theorem~\ref{thm:dp_sco_upper}]
  The privacy guarantee comes directly from the privacy guarantee of
  Algorithm~\ref{alg:wgd} and the fact that $\cS_t$ are non-overlapping\arxiv{
    (since $\sum_{t = 1}^T n_t = n \sum_{t = 1}^T \frac{1}{2^t} < n$)}.
  The proof for utility is similar to the proof of Theorem~4.8
  in~\cite{feldman2020private}. In round $t$ of Algorithm~\ref{alg:phased_erm},
  we consider the true minimizer $\theta_t^\ast$ and the approximate minimizer
  $\what{\theta}_t$. By stability~\citep{bousquet2002stability}, we can bound
  the generalization error of $\theta_t^\ast$ (see
  Proposition~\ref{prop:stability} in Appendix~\ref{app:sco}) and, by
  Corollary~\ref{coro:erm_sc}, we can bound
  $\E\norm{\what{\theta}_t - \theta^\ast_t}_2^2$. We finally choose
  $\crl{(\lambda_t, n_t)}_{t\le T}$ such that the assumptions of
  Corollary~\ref{coro:erm_sc} hold and to minimize the final error.
\end{proof}

\notarxiv{\vspace{-0.25cm}}
\subsection{Lower bound: SCO is harder than Gaussian mean 
estimation}\label{sec:sco_lower}

First of all, note that it suffices to prove the lower bounds in the homogeneous
setting as any level of heterogeneity only makes the problem
harder. Theorem~\ref{thm:dp_sco_lower} holds for $(\eps, \delta)$-user-level
DP---importantly, this is a setting for which lower bounds are generally more
challenging (we provide a related lower bound for $\eps$-user-level DP in
\arxiv{Section}\notarxiv{Appendix}~\ref{sec:finite-lb}). We present the proof in
Appendix~\ref{app:sco-lb} and a sketch here.

\begin{proof}[Proof sketch of Theorem~\ref{thm:dp_sco_lower}]
  \arxiv{As is often the case in privacy, }\notarxiv{T}\arxiv{t}he (constrained)
  minimax lower bound decomposes into a statistical rate and a privacy rate. The
  statistical rate is optimal (see, e.g.,~\cite{LevyDu19,AgarwalBaRaWa12}), thus
  we focus on the privacy rate. We consider linear losses \arxiv{\footnote{We truncate
    the data distribution appropriately so the losses remain individually
    Lipschitz.}} of the form $\f(\theta;z) = -\tri{\theta, z}$. We show that
  optimizing $\ff(\theta;P) = \E_P\brk{\f(\theta;Z)}$ over $\theta\in\Theta$ is
  harder than the mean estimation task for $P$. Intuitively,
  $\ff(\theta;P) = -\tri{\theta, \E Z}$ attains its minimum at
  $\theta^\ast = R\E\brk{Z} / \norm{\E\brk{Z}}_2$ and finding $\theta^\ast$
  provides a good estimate of (the direction of) $\E\brk{Z}$. We make this
  formal in Proposition~\ref{prop:sco-gaussian-mean}. Next, for Gaussian mean
  estimation, we reduce, in Proposition~\ref{prop:item-to-user}, user-level DP
  to item-level DP with lower variance by having each user contribute their
  sample average (which is a sufficient statistic). We conclude with the results
  of~\cite{kamath2019privately} (see Proposition~\ref{prop:dp-gauss-lb}) by
  proving in Corollary~\ref{cor:lb-direction} that estimating the direction of
  the mean with item-level privacy is hard.
\end{proof}

\arxiv{%
\section{Function Classes with Bounded Metric Entropy under Pure  
DP}\label{sec:warm-up}

We consider the general task of learning hypothesis class with finite metric
entropy (i.e., such that there exists a finite $\Delta$-cover under a certain
norm) and bounded loss under \emph{pure} user-level DP constraints.

For this setting, we present
Algorithm~\ref{alg:private-selection}, which we complement with an 
information-theoretic lower
bound.
As in the previous sections, we consider a sample set
$\mc{S} = \prn{S_1, \ldots, S_n}$, with
$S_u = \crl{z^{(u)}_j}_{j\in [m]}\subset\ZZ$.  We begin by considering 
the case
of a finite parameter space: for $K \in \N, K < +\infty$, we have
\begin{equation}
  \Theta = \crl*{\theta^{(1)}, \ldots, \theta^{(K)}}.
\end{equation}
For $0 \le B <\infty$, we denote
$\mc{F}_B \defeq \crl{\f\colon \Theta\times\ZZ \to \R: \norms{\ell}_\infty \le
  B}$ the set of $B$-bounded functions and $\Auserp$ the set $\eps$-user-level
DP estimators from $\cZ^n$ to $\Theta$, the goal of this section is to elicit
the constrained minimax rate~\citep{Yu97, BarberDu14, acharya2020differentially}
\begin{equation*}
  \mathfrak{M}^{\msf{user}}_{m, n}(\Theta, \mc{F}_B, \eps) \defeq
  \sup_{\ZZ, \mc{P} \subset \mc{P}(\ZZ)} \inf_{\A \in \Auserp}
  \sup_{\f \in \mc{F}_B, P\in\mc{P}}
  \E_{\cS \simiid \prn{P^m}^n}\brk*{\ff\prn*{\A(\cS);P} - \inf_{\theta \in 
  \Theta}\ff(\theta;P)}.
\end{equation*}
We start with providing the estimator, which
combines the private mean estimator of Section~\ref{sec:sq} with the
private selection techniques of~\cite{liu2019private}. Given a collection
of $\eps$-DP mechanisms, the latter provides an $\eps$-DP way to find
an (approximate) minimum by sampling from each mechanism at random
\emph{with the same data} and returning the maximum of the values
observed. In our setup, each mechanism $\sf{A}_k$ will be a private
release of $\ff(\theta^{(k)}; \mc{S})$.

\subsection{Combining mean estimation and private 
selection}\label{sec:finite-ub}

Our first step is to show that the conditions of Section~\ref{sec:sq}
are met, that is, the data are concentrated with high probability.

\begin{lemma}\label{lem:concentration-finite}
  Let $\cS = (S_1, \ldots, S_n) \simiid (P^m)^n$ and $\alpha \in \prn{0, 1}$. With probability
  greater than $1-\alpha$, it holds that
  \begin{equation}\label{eq:concentration-finite}
    \max_{k \in K} \max_{u \in \brk{n}} \abs*{\ff(\theta^{(k)}; S_u)
      - \ff(\theta^{(k)};P)} \le \frac{B}{2}\sqrt{\frac{\log(\abs{\Theta} \cdot n) + \log(2/\alpha)}{m}}.
  \end{equation}
  In other words, $(\cS, \cQ_\Theta)$ is
  $(B/(2\sqrt{m})\sqrt{\log(2Kn/\alpha)}, \alpha)$ uniformly concentrated where
  $\cQ_\Theta = \crl{\ff(\theta; \cdot): \theta\in\Theta}$.
\end{lemma}

\begin{proof}
  The proof is straightforward: for a fixed $\theta^{(k)} \in \Theta$ and
  $u \in \brk{n}$, the random variable $\ff(\theta^{(k)};S_u)$ is
  $\tfrac{B^2}{4m}$-sub-Gaussian around its mean $\ff(\theta^{(k)};P)$. A union
  bound over the samples and parameters concludes the proof.
\end{proof}
Conditioned on that event, the data are well concentrated and the results of
Theorem~\ref{thm:winsorized} apply. We now describe the algorithm and then go on
to prove privacy and utility guarantees. We call it ``idealized'' because it is 
not
computationally efficient. Roughly, the running time scales as
$\abs{\Theta} / \alpha$ to obtain good accuracy with probability greater than
$1-\alpha$. In certain problems, $\abs{\Theta}$ can be
exponential in the dimension (e.g., the Lipschitz stochastic optimization 
problem considered in Remark~\ref{rm:l_infity}),
which makes it computationally intractable.

\begin{algorithm}
  \caption{Idealized estimator for learning with bounded losses}
  \begin{algorithmic}[1]\label{alg:private-selection}
    \STATE \textbf{Input:} Privacy parameter $\eps$, probability of stopping
    $\gamma \in (0, 1]$, concentration parameter $\tau > 0$, finite parameter set $\Theta$,
    dataset $\cS = \crl{S_1, \ldots, S_n}$ \STATE Denote
  \begin{equation*}
    \msf{A}_k(S) \defeq \textbf{WinsorizedMean1D}\prn*{\crl{\ff(\theta^{(k)};S_u)}_{u\in\brk{n}}, \eps / 3, \tau}
  \end{equation*}
  
  \STATE Initialize $\mc{T} = \emptyset$.  \FOR{$t = 0, \ldots, \infty$} \STATE
  Sample $J_t \sim \msf{Uniform}(\crl*{1, \ldots, \abs{\Theta}})$.  \STATE
  Sample $V_t \sim \msf{A}_{J_t}(\mc{S})$.  \STATE Update
  $\mc{T} \to \mc{T} \cup \crl{(J_t, V_t)}$.  \STATE Sample
  $w_t \sim \msf{Bernoulli}(\gamma)$, if $w_t = 1$, break; \ENDFOR \STATE
  $t^\ast \to \argmin_{t}V_t$.  \RETURN $(J_{t^\ast}, V_{t^\ast})$.
\end{algorithmic}
\end{algorithm}

We state the privacy and utility of our algorithm. The result follows
from the utility guarantees of the mean estimator 
(Algorithm~\ref{alg:winsorized}) and the guarantees of private selection in
\cite{liu2019private}. 
\begin{restatable}{theorem}{restateFiniteHyp}\label{thm:finite-hyp}
  Let $\alpha \in (0, 1]$ and let us consider
  Algorithm~\ref{alg:private-selection} with $q = 1/K = 1/\abs{\Theta}$ and
  $\tau = \tfrac{B}{2}\sqrt{(\log(Kn) + \log(10/\alpha) / m}$. Assuming that
  $n \ge \frac{8}{\eps}\log\prn*{\tfrac{25\log(5/\alpha)}{\alpha^2} \cdot
  	\tfrac{KB}{\tau}}$,
  the following holds:
  \begin{enumerate}[label=(\roman*)]
  \item\label{item:privacy-priv-selection} The mechanism of
    Algorithm~\ref{alg:private-selection} is $\eps$-user-level DP.
  \item\label{item:utility-priv-selection} Let $J_{t^*}$ be the output of
    Algorithm~\ref{alg:private-selection}, with probability greater than
    $1-\alpha$ it achieves the following utility
            \begin{equation}
      \ff(\theta^{(J_{t^\ast})}; \mc{S}) - 
      \inf_{\theta'\in\Theta}\ff(\theta';\mc{S})
      \le 8\frac{B}{n\sqrt{m}\eps}\log\prn*{25 K \cdot 
      \frac{\log(5/\alpha)}{\alpha^2}}
      \sqrt{\log(Kn) + \log(10/\alpha)}.
      \end{equation}
    \end{enumerate}
  \end{restatable}

  \arxiv{
  \begin{proof}[Proof sketch]
    The privacy is immediate since we select the mechanisms uniformly at
    random. We choose the parameter $\gamma$ in
    Algorithm~\ref{alg:private-selection} such that with high probability the
    algorithm queries the \emph{best} parameter $\theta^{(k^\ast)}$. Thus, if
    the algorithm returns any other parameter $\theta^{(J_{t^\ast})}$, it must
    be that the added noise on $V_{J_{t^\ast}}$ and $V_{k^\ast}$ compensate for
    their difference in utility. Since the noise is (with high-probability)
    i.i.d. Laplace noise, we bound the size of the noise for the length of the
    game which gives the final result.
  \end{proof}
}
\notarxiv{

\begin{proof} We first state the privacy guarantee followed by the utility
    guarantee.

    \paragraph{Proof of~\ref{item:privacy-priv-selection}} Since each
    $\msf{A}_k$ is $\eps/3$-user-level DP, Theorem~3.2 in~\cite{liu2019private}
    guarantees that the output of Algorithm~\ref{alg:private-selection} is
    $\eps$-user-level DP.

    \paragraph{Proof of~\ref{item:utility-priv-selection}} The proof is adapted
    from Theorem~5.2 in~\cite{liu2019private}. First of all, with probability
    greater than $1-\alpha_1$, as we prove in
    Lemma~\ref{lem:concentration-finite}, the data are uniformly concentrated for
    all $\theta^{(k)}$, meaning
      \begin{equation*}\label{eq:concentration-finite}
        \max_{k \in K} \max_{u \in \brk{n}} \abs*{\ff(\theta^{(k)}; S_u)
          - \ff(\theta^{(k)};P)} \le
        \crl*{\frac{B}{2}\sqrt{\frac{\log(\abs{\Theta} \cdot n) + 
        \log(2/\alpha_1)}{m}} \eqdef \tau}.
  \end{equation*}
  We condition on this event (Event 1) for the rest of the proof. Let
  $\alpha_1 \in (0, 1]$ and $\gamma \in (0, 1]$. Let $\tstop$ denotes the 
  time that the algortihm exists the loop, which is number of queries the 
  algorithm makes. 
  
  Let us denote $k^\ast$, the best hypothesis in $\Theta$ i.e.
  \begin{equation*}
  	k^\ast = \argmin_{k \le K} \ff(\theta^{(k)}; \mc{S}).
  \end{equation*}
  We choose $\gamma$ such that $k^\ast$ is queried with probability 
  greater than
  $1-\alpha_1$, i.e., if $E_{\neg k^*}$ is the event (denote $\neg E_{\neg 
  k^*}$ as Event 2) that the 
  algorithm 
  finishes
  without querying $k^*$, we choose $\gamma$ such that
  $\P(E_{\neg k^*}) \le \alpha_1$. More precisely,
  \begin{align*}
	\P(E_{\neg k^*})
	& = \sum_{l = 1}^\infty
	\P(E_{\neg k^*}|\tstop = l)
	\P(\tstop = l)  \\
	& = \sum_{l=1}^\infty \prn*{1-\frac1K}^l \cdot 
	\prn*{1-\gamma}^{l-1}\cdot \gamma \\
	&= 
	\prn*{1-\frac1K}\gamma\sum_{l=0}^\infty \brk*{\prn*{1-\frac1K} 
		\prn*{1-\gamma}}^l \\
	& = \frac{\prn*{1-\frac1K}\gamma}{1 - 
		\prn*{1-\frac1K} 
		\prn*{1-\gamma}}.
\end{align*}

  Choosing $\gamma = \alpha_1 / K$ guarantees that
  $\P(E_{\neg k^*}) \le \alpha_1$.   
  Let $L \defeq \tfrac{\log(1/\alpha_1)}{\gamma} = 
  \log(1/\alpha_1)\tfrac{K}{\alpha_1}$, we 
  have
 
   \begin{align*}
 	\P(\tstop > L)  = 
 	\P(\omega_1 = \ldots = \omega_L = 0)  = (1-\gamma)^L  \le \exp(-L 
 	\gamma) = \alpha_1.
 \end{align*}

  Hence with probability at least $1 - \alpha_1$, the algorithm ends in less 
  than
  $L$ throws (Event 3). Conditioned on this event, by 
  Theorem~\ref{thm:winsorized} 
  and union bound, with 
  probability greater than
  $1 - L \cdot \tfrac{B}{\tau}\exp(-n\eps / 8)$, the output of $\msf{A}_{J_t}$
  for all $t \le \tstop$ is
  \begin{equation*}
    \msf{A}_{J_t}(S) = \ff(\theta^{(J_t)}; \cS) + \msf{Lap}\prn*{\frac{8\tau}{n\eps}}
    = \frac{1}{m\cdot n}\sum_{j \in \brk{m}, u \in \brk{n}}\f\prn*{\theta^{(J_t)};z^{(u)}_j}
    + \msf{Lap}\prn*{\frac{8\tau}{n\eps}},
  \end{equation*}
which we denote as Event 4.
  For a Laplace distribution, computing the tail gives that
  $\P(\abs{\msf{Lap}(\lambda)} \ge u) \le \exp(-u / \lambda)$ and with a union
  bound and change of variables it holds that if $Y_1, Y_2, \ldots, Y_L 
  \simiid \msf{Lap}(\frac{8\tau}{n\eps})$, then
  with probability greater than $1-\alpha_1$
  \begin{equation*}
    \max_{i=1,\ldots L} \abs{Y_i} \le 
    \frac{8\tau}{n\eps}\log\prn*{\frac{L}{\alpha_1}}.
  \end{equation*}
  In other words, except with probability $\alpha_1$, the noise is bounded 
  by
  $\tfrac{8\tau}{n\eps}\log(L / \alpha_1)$ (Event 5). Conditioned on all 
  these events, the
  parameter $\theta^{(J_{t^\ast})}$ that the algorithm outputs is 
  sub-optimal by
  at most $\tfrac{16\tau}{n\eps}\log(L / \alpha_1)$ as in the worst-case the 
  noise
  is $+\tfrac{8\tau}{n\eps}\log(L / \alpha_1)$ for $J_{t^*}$ and
  $-\tfrac{8\tau}{n\eps}\log(L / \alpha_1)$ for $k^\ast$.
    Setting %
    $\alpha_1 = \alpha/5$ and as we assume that
  $n \ge \frac{8}{\eps}\log\prn*{\tfrac{25\log(5/\alpha)}{\alpha^2} \cdot
  	\tfrac{KB}{\tau}}$, we conclude the proof by taking a union bound over all 
  	5 events.
  
\end{proof}

 }

\begin{corollary}
  Assume
  $n \ge \Omloglog(1) \, \tfrac{1}{\eps}\max\crl*{\tfrac{1}{Km},
    \log\prn{Km}}$. It holds that
  \begin{equation}
    \minimaxpurebdd =  \Olog\prn*{B\crl*{\sqrt{\frac{\log K}{m\cdot n}}
      + \frac{\log^{3/2}\prn*{Knm\eps}}{n\sqrt{m}\eps}}},
\end{equation}
where $\Olog, \Omloglog$ ignores only numerical constants and log-log 
factors in 
this case.
\end{corollary}

\begin{proof} We get the result directly from
  Theorem~\ref{thm:finite-hyp}, by setting
  $\alpha = \log K/(n\sqrt{m}\eps)$, applying standard uniform
  convergence results for bounded losses with finite parameter set
  (Hoeffding bound) and ignoring log-log factors.
\end{proof}

\begin{corollary}[Parameter sets with finite metric entropy] Let us
  further assume that our loss functions are $G$-Lipschitz with
  respect to some norm $\norm{\cdot}$ with (finite) covering number
  $\msf{N}_{\norm{\cdot}}(\Theta, \Delta)$---i.e. there exists a set
  $\Gamma_{\norm{\cdot},\Delta} \subset \Theta$ such that
  $\abs{\Gamma_{\norm{\cdot}, \Delta}} =
  \msf{N}_{\norm{\cdot}}(\Theta, \Delta)$ and for all
  $\theta \in \Theta$, there exists
  $\tau \in \Gamma_{\norm{\cdot}, \Delta}$ such that
  $\norm{\theta - \tau} \le \Delta$. In this case, for any
  $\Delta > 0$ and applying Algorithm~\ref{alg:private-selection} with
  parameter set $\Gamma$ guarantees that
  \begin{multline*}
    \mathfrak{M}^{\msf{user}}_{m, n}(\Theta, \mc{F}_{B, (G,
      \norm{\cdot})}, \eps) 
    = \tilde{O}(1) \, \inf_{\Delta > 0}\crl*{B\brk*{\sqrt{\frac{\log
            \msf{N}_{\norm{\cdot}}(\Theta, \Delta)}{m\cdot n}} +
        \frac{\log^{3/2}\prn*{\msf{N}_{\norm{\cdot}}(\Theta,
            \Delta)nm\eps}}{n\sqrt{m}\eps}} + G\Delta}.
\end{multline*}
\end{corollary}

\begin{remark} \label{rm:l_infity}For
  $\norm{\cdot} = \ell_2, \Theta=\mathbb{B}^d_\infty(0, 1)$ and
  setting
  $\Delta = \tfrac{B}{G}\crl*{\sqrt{d/(mn)} +
    d^{3/2}/(n\eps\sqrt{m})}$, we directly get
  \begin{equation*}
    \mathfrak{M}^{\msf{user}}_{m, n}(\mathbb{B}^d_\infty(0, 1), \mc{F}_{B, (G, \ell_2)}, \eps)
    =  \Olog\crl*{B\sqrt{\frac{d}{m\cdot n}} + B\frac{d^{3/2}}{n\sqrt{m}\eps}}.
  \end{equation*}
  The first term, which corresponds to the statistical rate, is optimal (see 
  e.g. Proposition~2
  in~\cite{LevyDu19}). Whether the privacy rate is optimal remains open.
\end{remark}

\subsection{Information-theoretic lower bound}\label{sec:finite-lb}

We now prove a lower bound on
$\mathfrak{M}^{\msf{user}}_{m, n}(\Theta, \mc{F}_B, \eps)$ when
$\abs{\Theta} = K < \infty$. We follow the standard machinery of reducing
estimation to testing \citep{Yu97, Wainwright19} but under privacy
constraints \citep{BarberDu14, acharya2020differentially}.
\begin{restatable}[Lower bound for finite-hypothesis class]
  {theorem}{restateFiniteHypLb}\label{thm:finite-hyp-lb} Let
  $K, m, n \in \N, K < \infty, \eps \in \R_+,$ and $0\le B < \infty$. Assume
  $\log_2K \ge 32\log 2$ and
  $n\ge \log_2K \max\crl{\tfrac{1}{192\sqrt{m}\eps}, \tfrac{1}{96m}}$, there
  exists a sample space $\cZ$ and parameter set $\Theta$ with $\abs{\Theta} = K$
  and $\abs{\cZ} = \ceil{\log_2 K}$ such that the following holds
  \begin{equation}
    \mathfrak{M}^{\msf{user}}_{m, n}(\Theta, \mc{F}_B, \eps) = \Omega \prn*{ 
    B\sqrt{\frac{\log \abs{\Theta}}{m\cdot n}}
    +B\frac{\log \abs{\Theta}}{n \sqrt{m} \epsilon}}.
\end{equation}
\end{restatable}

We detail the proof of the theorem \notarxiv{below}\arxiv{in
  Appendix~\ref{app:proof-finite-lb}}. The proof relies on a (standard)
generalization of Fano's method, whcih reduces optimization to multiple
hypothesis tests. We refer to the results of~\cite{acharya2020differentially} to
obtain the lower bounds in the case of a constrained---in this case,
$\epsilon$-DP---estimators. For the user-level case, we simply consider that
samples from an $m$-fold product of measures---the separation does not change but
the KL-divergence increase by at most a $m$ factor and TV-distance increase by
at most a $\sqrt{m}$ factor thus yielding the final answer.

\notarxiv{
\begin{proposition}[{\citet[][Corollary~4]{acharya2020differentially}}]\label{prop:private-fano}
  Let $\mc{P}$ be a collection of distributions over a common sample space $\ZZ$
  and a loss function $\f:\Theta\times\ZZ \to \R_+$. For $P, Q\in\mc{P}$, define
  \begin{equation*}
    \mathsf{sep}_{\ff}(P, Q; \Theta) := \sup\left\lbrace \delta \ge 0
      \;\middle|\; \mbox{for all~}
      \theta\in\Theta, \begin{array}{c} \ff(\theta, P) \leq \delta ~\mbox{implies}~
                         \ff(\theta, Q) \geq \delta \\
                         \ff(\theta, Q) \leq \delta ~\mbox{implies}~
                         \ff(\theta, P) \geq \delta
                       \end{array}\right\rbrace.
                   \end{equation*}

                   Let $\mc{V}$ be a finite index set and
                   $\mc{P}_{\mc{V}} \defeq \crl*{P_v}_{v\in\mc{V}}$ be a
                   collection of distributions contained in $\mc{P}$ such that
                   for $\Delta \ge 0$,
                   $\min_{v\neq v'}\msf{sep}(P_v, P_{v'}, \Theta) \le \Delta$.
                   Then
                   \begin{equation*}
                     \minimaxitem \ge
                     \frac{\Delta}{4}\max\crl*{1 - \frac{I(X_1^n;V) + \log 2}{\log \abs{\mc{V}}},
                       \min\crl*{1, \frac{\abs{\mc{V}}}{\exp(c_0 n\eps \mathrm{d}_\msf{TV}(\mc{P}_\mc{V}))}}},
                   \end{equation*}
                   where
                   $V \sim \msf{Uniform}(\mc{V}), c_0 = 10,
                   \mrm{d}_\msf{TV}(\mc{P}_\mc{V}) \defeq \max_{v\neq
                     v'}\norm{P_v - P_{v'}}_\msf{TV}$ and $I(X; Y)$ is the
                   (Shannon) mutual information.
                 \end{proposition}

\begin{proof}[Proof of Theorem~\ref{thm:finite-hyp-lb}] We follow the 
standard
    steps: we first compute the separation, we bound the testing error for any
    (constrained) estimator in the item-level DP case (with
    Proposition~\ref{prop:private-fano}) and finally, we show how to adapt the
    proof to obtain the user-level DP lower bound.
    \paragraph{Separation} For simplicity, assume $K = 2^d$, if not, the problem
    is harder than for $\underline{K} = 2^{\floor{\log_2K}} \le K$ which is of
    the same order. Let us define the sample space $\ZZ$, the parameter set
    $\Theta$ and the loss function $\f$ we consider.

  We define
  \begin{equation*}
    \ZZ = \Theta \defeq \crl{-1, +1}^d\mbox{~~and~~}
    \f(\theta; z) \defeq B\sum_{j \le d}\mathbf{1}_{\theta_j = z_j}.
  \end{equation*}

  We consider $\mc{V}$ an $d/2$-$\ell_1$ packing of $\crl{\pm 1}^d$ of size at
  least $\exp(d/8)$---which the Gilbert-Varshimov bound (see e.g.,
  {\cite[][Ex. 4.2.16]{Vershynin19}}) guarantees the existence of---and consider
  the following family of distribution $\mc{P} = \crl{P_v: v\in\mc{V}}$ such
  that if $X \sim P_v$ then
  \begin{equation}\label{eqn:bcube_prob}
    X = \begin{cases}
      v_j e_j & \mbox{~~with probability~~} \frac{1+\Delta}{2d}\\
      -v_j e_j & \mbox{~~with probability~~} \frac{1-\Delta}{2d}.
    \end{cases}
  \end{equation}
  For $\theta\in\Theta$, we have that
  \begin{equation*}
    \ff(\theta;P_v) = \E_{P_v}\brk*{B\sum_{j\le d} \mathbf{1}_{\theta_j = Z_j}} 
    =
    B\sum_{j \le d}\frac{1+\theta_j v_j \Delta}{2d}.
  \end{equation*}
  Naturally, $\ff(\theta;P_v)$ achieves its minimum at $\theta_v^* = -v$ such
  that $\inf_{\theta'\in\theta}\ff(\theta;P_v) = B\tfrac{1-\Delta}{2}$. We now
  compute the separation by noting that
  \begin{equation}\label{eqn:bcube_loss}
    \mathsf{sep}_{\ff}(P_v, P_{v'}, \Theta) \ge \frac{1}{2} \min_{\theta'\in\Theta}
    \crl*{\ff(\theta'; P_v) + \ff(\theta';P_{v'}) - \ff(\theta_v^*;P_v) - \ff(\theta_{v'}^*; P_{v'})}.
  \end{equation}
  A quick computation shows that $\msf{sep}_{\ff}(P_v, P_{v'}, \Theta) \ge 
  \tfrac{B\Delta}{8}$ by
  noting that $\mrm{d}_{\msf{Ham}}(v, v') \ge d/4$.

  \paragraph{Obtaining the item-level lower bound} We can now use the results of
  Proposition~\ref{prop:private-fano}. We have that
  $\min_{v\neq v'}\msf{sep}_{\ff}(P_v, P_{v'}, \Theta) \ge
  \tfrac{B\Delta}{8}$. The identity
  $\mathrm{D}_{\mathrm{KL}}(P_v, P_{v'}) = 
  \Delta\log\tfrac{1+\Delta}{1-\Delta}
  \le 3\Delta^2$ implies that $I(Z^n;V) \le 3n\Delta^2$. Similarly, 
  Pinsker's inequality guarantees that
  \begin{equation*}
    \mrm{d}_\msf{TV} \le \sqrt{\frac{1}{2}\max_{v\neq v'}
      \mrm{D}_\mrm{KL}(P_v, P_{v'})} \le \sqrt{3/2}\Delta.
  \end{equation*}
  
  We put everything together and it holds that for $\Delta \in \brk{0, 1}$,
  \begin{equation}
    \minimaxitem \ge \frac{B\Delta}{32}\max\crl*{1 - \frac{3n\Delta^2 + 
    \log 2}{d/8},
      \min\crl*{1, \frac{\exp(d/8)}{\exp(30n\eps \Delta)}}}.
  \end{equation}
  Since $d \ge 32\log 2$, $\Delta = \sqrt{d/(96n)}$ guarantees that
  $1 - \frac{3n\Delta^2 + \log 2}{d/8} \ge 1/2$. On the other hand, setting
  $\Delta = \tfrac{5}{960}\tfrac{d}{n\eps}$, guarantees that
  $\min\crl*{1, \frac{\exp(d/8)}{\exp(30n\eps \Delta)}} \ge 1/2$. The 
  assumption
  on $n$ guarantees that these two values are in $\brk{0, 1}$ and thus setting
  $\Delta^* = \max\crl*{{\sqrt{d/(96n)}}, \tfrac{1}{192}\tfrac{d}{n\eps}}$ 
  which
  implies that
  \begin{equation*}
    \minimaxitem \ge \frac{B}{32}\crl*{\sqrt{\frac{d}{96n}} + \frac{1}{192}\frac{d}{n\eps}}.
  \end{equation*}

  \paragraph{Concluding for user-level DP} Let $m\in \N, m\ge 1$. For the
  user-level DP lower bound, the proof remains the same except that the
  collection $\mc{P}_{\mc{V}}$ becomes $\crl{P_v^m}_{v\in\mc{V}}$ i.e. the
  $m$-fold product distribution of $P_v$. The separation remains exactly the
  same but we now have
  \begin{equation*}
    \mrm{D}_\mrm{KL}(P_v^m, P_{v'}^m) \le 
    3m\Delta^2\mbox{~~and~~}\mrm{d}_\msf{TV}(\mc{P}_\mc{V})
    \le \sqrt{\frac{3m}{2}}\Delta.
  \end{equation*}
  Under the assumption
  $\Delta^* = \max\crl*{{\sqrt{d/(96mn)}},
    \tfrac{1}{192}\tfrac{d}{n\sqrt{m}\eps}}$ is less than $1$ and thus concludes
  the proof.
\end{proof}
 }

Note, the upper bound of Theorem~\ref{thm:finite-hyp} and the lower bound above
match only up to $\sqrt{\log K}$. Given that $K$ can be exponential in
the dimension---e.g. in the case of $\Theta$ being a cover of an
$\ell_p$ ball---the bound is only tight for ``small'' hypothesis
class. However,  it seems this
extra-factor cannot be removed using the techniques we present in this 
paper, as we need to both obtain uniform
concentration and bound the maximum of i.i.d. noise over $K$
samples---both of which are tight. We leave the problem of finding an
optimal estimator for this problem to future work.

}

\arxiv{%

\ignore{
In this work, we explore the
fundamental limits of learning under user-level privacy
constraints. Importantly, we provide practical algorithms with significantly
improved privacy cost in the regime where the number of samples per user
$m \gg 1$. However, the picture is not yet complete; we sketch out some possible
directions for future work.
First, all of our utility results hinge on the assumption that all users sample
their data from distributions $\crl{P_u}_{u\in\brk{n}}$ such that
$\max_{u, v}\dtv{P_u}{P_v}$ is polynomially small. Establishing utility
guarantees in the case of stronger heterogeneity remains an open question and is
relevant to a number of practical settings (e.g., users from different
geographical regions or who speak different languages). While our algorithms
remain user-level private regardless of heterogeneity, establishing risk bounds
requires further research.
Secondly, our optimization results require that the stochastic gradient be
sub-Gaussian. While the literature sometimes considers this assumption, it is
much less common than assuming that the stochastic gradients are almost surely
bounded in $\ell_2$ norm. Understanding minimax rates in this setting is an
important open question.
Finally, our work focuses on explicit information-theoretic limits but pays
less attention to the runtime of our algorithms. For example, in the case of SCO, our
algorithm runs in $\min\crl{(nm)^{3/2}, n^2m^{3/2}/ \sqrt{d}}$ time, while achieving
the optimal item-level private rate requires at most
$\min\crl{nm, (nm)^2 / d}$ time~\cite{bassily2019private}. Developing faster
algorithms in these settings is theoretically interesting but also
practically very meaningful as user-level privacy becomes widespread.}

\section*{Discussion}

In this work, we explore the fundamental limits of learning under user-level
privacy constraints. Importantly, we provide practical algorithms with
significantly improved privacy cost in the regime where the number of samples
per user $m \gg 1$. However, our work provides generalization guarantees under a
limited heterogeneity assumption. Extending our work to more heterogeneous
settings is an interesting research direction. Secondly, our work focuses on
establishing information-theoretic limits and we do not optimize the runtime of
our algorithms. For example, in the case of SCO, our algorithm runs in
$\min\crl{(nm)^{3/2}, n^2m^{3/2}/ \sqrt{d}}$ time, while achieving the optimal
item-level private rate requires at most $\min\crl{nm, (nm)^2 / d}$
time~\cite{bassily2019private}. Developing faster algorithms in these settings
is a possible future direction.

\notarxiv{\section*{Potential negative societal impact} Our work is theoretical
  in nature and we do not foresee major direct negative societal
  consequences. Because of the growing prevalence of data collection from all
  sources (mobile, browser, medical records etc.), providing meaningful
  guarantees---such as user-level DP---while preserving adequate accuracy is an
  important direction of research. Our work suffers from the same potential
  negative impact as any work in the broad differential privacy area in two
  ways:
  first, a simple way
  to guarantee privacy is to limit data collection or delete data the users
  provided in the past.
  Second, the guarantees we provide are contingent on careful choices of $\eps$
  and $\delta$ as well as rigorous and independent methodologies for evaluating
  the privacy of deployed models.  }

}

\arxiv{
  \subsection*{Acknowledgments}
  The authors would like to thank Hilal Asi and Karan Chadha for comments on an
  earlier draft as well as Yair Carmon, Peter Kairouz, Gautam Kamath, Sai
  Praneeth Karimireddy, Thomas Steinke and Sebastian Stich, for useful
  discussions and pointers to very relevant references.  }
\notarxiv{  \subsection*{Acknowledgments}
	The authors would like to thank Hilal Asi and Karan Chadha for 
	comments on an
	earlier draft as well as Yair Carmon, Peter Kairouz, Gautam Kamath, Sai
	Praneeth Karimireddy, Thomas Steinke and Sebastian Stich, for useful
	discussions and pointers to very relevant references.
}

\bibliographystyle{abbrvnat}

\newpage

\clearpage

\notarxiv{%

\ignore{
In this work, we explore the
fundamental limits of learning under user-level privacy
constraints. Importantly, we provide practical algorithms with significantly
improved privacy cost in the regime where the number of samples per user
$m \gg 1$. However, the picture is not yet complete; we sketch out some possible
directions for future work.
First, all of our utility results hinge on the assumption that all users sample
their data from distributions $\crl{P_u}_{u\in\brk{n}}$ such that
$\max_{u, v}\dtv{P_u}{P_v}$ is polynomially small. Establishing utility
guarantees in the case of stronger heterogeneity remains an open question and is
relevant to a number of practical settings (e.g., users from different
geographical regions or who speak different languages). While our algorithms
remain user-level private regardless of heterogeneity, establishing risk bounds
requires further research.
Secondly, our optimization results require that the stochastic gradient be
sub-Gaussian. While the literature sometimes considers this assumption, it is
much less common than assuming that the stochastic gradients are almost surely
bounded in $\ell_2$ norm. Understanding minimax rates in this setting is an
important open question.
Finally, our work focuses on explicit information-theoretic limits but pays
less attention to the runtime of our algorithms. For example, in the case of SCO, our
algorithm runs in $\min\crl{(nm)^{3/2}, n^2m^{3/2}/ \sqrt{d}}$ time, while achieving
the optimal item-level private rate requires at most
$\min\crl{nm, (nm)^2 / d}$ time~\cite{bassily2019private}. Developing faster
algorithms in these settings is theoretically interesting but also
practically very meaningful as user-level privacy becomes widespread.}

\section*{Discussion}

In this work, we explore the fundamental limits of learning under user-level
privacy constraints. Importantly, we provide practical algorithms with
significantly improved privacy cost in the regime where the number of samples
per user $m \gg 1$. However, our work provides generalization guarantees under a
limited heterogeneity assumption. Extending our work to more heterogeneous
settings is an interesting research direction. Secondly, our work focuses on
establishing information-theoretic limits and we do not optimize the runtime of
our algorithms. For example, in the case of SCO, our algorithm runs in
$\min\crl{(nm)^{3/2}, n^2m^{3/2}/ \sqrt{d}}$ time, while achieving the optimal
item-level private rate requires at most $\min\crl{nm, (nm)^2 / d}$
time~\cite{bassily2019private}. Developing faster algorithms in these settings
is a possible future direction.

\notarxiv{\section*{Potential negative societal impact} Our work is theoretical
  in nature and we do not foresee major direct negative societal
  consequences. Because of the growing prevalence of data collection from all
  sources (mobile, browser, medical records etc.), providing meaningful
  guarantees---such as user-level DP---while preserving adequate accuracy is an
  important direction of research. Our work suffers from the same potential
  negative impact as any work in the broad differential privacy area in two
  ways:
  first, a simple way
  to guarantee privacy is to limit data collection or delete data the users
  provided in the past.
  Second, the guarantees we provide are contingent on careful choices of $\eps$
  and $\delta$ as well as rigorous and independent methodologies for evaluating
  the privacy of deployed models.  }

}

\appendix

\notarxiv{%
\section{Function Classes with Bounded Metric Entropy under Pure  
DP}\label{sec:warm-up}

We consider the general task of learning hypothesis class with finite metric
entropy (i.e., such that there exists a finite $\Delta$-cover under a certain
norm) and bounded loss under \emph{pure} user-level DP constraints.

For this setting, we present
Algorithm~\ref{alg:private-selection}, which we complement with an 
information-theoretic lower
bound.
As in the previous sections, we consider a sample set
$\mc{S} = \prn{S_1, \ldots, S_n}$, with
$S_u = \crl{z^{(u)}_j}_{j\in [m]}\subset\ZZ$.  We begin by considering 
the case
of a finite parameter space: for $K \in \N, K < +\infty$, we have
\begin{equation}
  \Theta = \crl*{\theta^{(1)}, \ldots, \theta^{(K)}}.
\end{equation}
For $0 \le B <\infty$, we denote
$\mc{F}_B \defeq \crl{\f\colon \Theta\times\ZZ \to \R: \norms{\ell}_\infty \le
  B}$ the set of $B$-bounded functions and $\Auserp$ the set $\eps$-user-level
DP estimators from $\cZ^n$ to $\Theta$, the goal of this section is to elicit
the constrained minimax rate~\citep{Yu97, BarberDu14, acharya2020differentially}
\begin{equation*}
  \mathfrak{M}^{\msf{user}}_{m, n}(\Theta, \mc{F}_B, \eps) \defeq
  \sup_{\ZZ, \mc{P} \subset \mc{P}(\ZZ)} \inf_{\A \in \Auserp}
  \sup_{\f \in \mc{F}_B, P\in\mc{P}}
  \E_{\cS \simiid \prn{P^m}^n}\brk*{\ff\prn*{\A(\cS);P} - \inf_{\theta \in 
  \Theta}\ff(\theta;P)}.
\end{equation*}
We start with providing the estimator, which
combines the private mean estimator of Section~\ref{sec:sq} with the
private selection techniques of~\cite{liu2019private}. Given a collection
of $\eps$-DP mechanisms, the latter provides an $\eps$-DP way to find
an (approximate) minimum by sampling from each mechanism at random
\emph{with the same data} and returning the maximum of the values
observed. In our setup, each mechanism $\sf{A}_k$ will be a private
release of $\ff(\theta^{(k)}; \mc{S})$.

\subsection{Combining mean estimation and private 
selection}\label{sec:finite-ub}

Our first step is to show that the conditions of Section~\ref{sec:sq}
are met, that is, the data are concentrated with high probability.

\begin{lemma}\label{lem:concentration-finite}
  Let $\cS = (S_1, \ldots, S_n) \simiid (P^m)^n$ and $\alpha \in \prn{0, 1}$. With probability
  greater than $1-\alpha$, it holds that
  \begin{equation}\label{eq:concentration-finite}
    \max_{k \in K} \max_{u \in \brk{n}} \abs*{\ff(\theta^{(k)}; S_u)
      - \ff(\theta^{(k)};P)} \le \frac{B}{2}\sqrt{\frac{\log(\abs{\Theta} \cdot n) + \log(2/\alpha)}{m}}.
  \end{equation}
  In other words, $(\cS, \cQ_\Theta)$ is
  $(B/(2\sqrt{m})\sqrt{\log(2Kn/\alpha)}, \alpha)$ uniformly concentrated where
  $\cQ_\Theta = \crl{\ff(\theta; \cdot): \theta\in\Theta}$.
\end{lemma}

\begin{proof}
  The proof is straightforward: for a fixed $\theta^{(k)} \in \Theta$ and
  $u \in \brk{n}$, the random variable $\ff(\theta^{(k)};S_u)$ is
  $\tfrac{B^2}{4m}$-sub-Gaussian around its mean $\ff(\theta^{(k)};P)$. A union
  bound over the samples and parameters concludes the proof.
\end{proof}
Conditioned on that event, the data are well concentrated and the results of
Theorem~\ref{thm:winsorized} apply. We now describe the algorithm and then go on
to prove privacy and utility guarantees. We call it ``idealized'' because it is 
not
computationally efficient. Roughly, the running time scales as
$\abs{\Theta} / \alpha$ to obtain good accuracy with probability greater than
$1-\alpha$. In certain problems, $\abs{\Theta}$ can be
exponential in the dimension (e.g., the Lipschitz stochastic optimization 
problem considered in Remark~\ref{rm:l_infity}),
which makes it computationally intractable.

\begin{algorithm}
  \caption{Idealized estimator for learning with bounded losses}
  \begin{algorithmic}[1]\label{alg:private-selection}
    \STATE \textbf{Input:} Privacy parameter $\eps$, probability of stopping
    $\gamma \in (0, 1]$, concentration parameter $\tau > 0$, finite parameter set $\Theta$,
    dataset $\cS = \crl{S_1, \ldots, S_n}$ \STATE Denote
  \begin{equation*}
    \msf{A}_k(S) \defeq \textbf{WinsorizedMean1D}\prn*{\crl{\ff(\theta^{(k)};S_u)}_{u\in\brk{n}}, \eps / 3, \tau}
  \end{equation*}
  
  \STATE Initialize $\mc{T} = \emptyset$.  \FOR{$t = 0, \ldots, \infty$} \STATE
  Sample $J_t \sim \msf{Uniform}(\crl*{1, \ldots, \abs{\Theta}})$.  \STATE
  Sample $V_t \sim \msf{A}_{J_t}(\mc{S})$.  \STATE Update
  $\mc{T} \to \mc{T} \cup \crl{(J_t, V_t)}$.  \STATE Sample
  $w_t \sim \msf{Bernoulli}(\gamma)$, if $w_t = 1$, break; \ENDFOR \STATE
  $t^\ast \to \argmin_{t}V_t$.  \RETURN $(J_{t^\ast}, V_{t^\ast})$.
\end{algorithmic}
\end{algorithm}

We state the privacy and utility of our algorithm. The result follows
from the utility guarantees of the mean estimator 
(Algorithm~\ref{alg:winsorized}) and the guarantees of private selection in
\cite{liu2019private}. 
\begin{restatable}{theorem}{restateFiniteHyp}\label{thm:finite-hyp}
  Let $\alpha \in (0, 1]$ and let us consider
  Algorithm~\ref{alg:private-selection} with $q = 1/K = 1/\abs{\Theta}$ and
  $\tau = \tfrac{B}{2}\sqrt{(\log(Kn) + \log(10/\alpha) / m}$. Assuming that
  $n \ge \frac{8}{\eps}\log\prn*{\tfrac{25\log(5/\alpha)}{\alpha^2} \cdot
  	\tfrac{KB}{\tau}}$,
  the following holds:
  \begin{enumerate}[label=(\roman*)]
  \item\label{item:privacy-priv-selection} The mechanism of
    Algorithm~\ref{alg:private-selection} is $\eps$-user-level DP.
  \item\label{item:utility-priv-selection} Let $J_{t^*}$ be the output of
    Algorithm~\ref{alg:private-selection}, with probability greater than
    $1-\alpha$ it achieves the following utility
            \begin{equation}
      \ff(\theta^{(J_{t^\ast})}; \mc{S}) - 
      \inf_{\theta'\in\Theta}\ff(\theta';\mc{S})
      \le 8\frac{B}{n\sqrt{m}\eps}\log\prn*{25 K \cdot 
      \frac{\log(5/\alpha)}{\alpha^2}}
      \sqrt{\log(Kn) + \log(10/\alpha)}.
      \end{equation}
    \end{enumerate}
  \end{restatable}

  \arxiv{
  \begin{proof}[Proof sketch]
    The privacy is immediate since we select the mechanisms uniformly at
    random. We choose the parameter $\gamma$ in
    Algorithm~\ref{alg:private-selection} such that with high probability the
    algorithm queries the \emph{best} parameter $\theta^{(k^\ast)}$. Thus, if
    the algorithm returns any other parameter $\theta^{(J_{t^\ast})}$, it must
    be that the added noise on $V_{J_{t^\ast}}$ and $V_{k^\ast}$ compensate for
    their difference in utility. Since the noise is (with high-probability)
    i.i.d. Laplace noise, we bound the size of the noise for the length of the
    game which gives the final result.
  \end{proof}
}
\notarxiv{

\begin{proof} We first state the privacy guarantee followed by the utility
    guarantee.

    \paragraph{Proof of~\ref{item:privacy-priv-selection}} Since each
    $\msf{A}_k$ is $\eps/3$-user-level DP, Theorem~3.2 in~\cite{liu2019private}
    guarantees that the output of Algorithm~\ref{alg:private-selection} is
    $\eps$-user-level DP.

    \paragraph{Proof of~\ref{item:utility-priv-selection}} The proof is adapted
    from Theorem~5.2 in~\cite{liu2019private}. First of all, with probability
    greater than $1-\alpha_1$, as we prove in
    Lemma~\ref{lem:concentration-finite}, the data are uniformly concentrated for
    all $\theta^{(k)}$, meaning
      \begin{equation*}\label{eq:concentration-finite}
        \max_{k \in K} \max_{u \in \brk{n}} \abs*{\ff(\theta^{(k)}; S_u)
          - \ff(\theta^{(k)};P)} \le
        \crl*{\frac{B}{2}\sqrt{\frac{\log(\abs{\Theta} \cdot n) + 
        \log(2/\alpha_1)}{m}} \eqdef \tau}.
  \end{equation*}
  We condition on this event (Event 1) for the rest of the proof. Let
  $\alpha_1 \in (0, 1]$ and $\gamma \in (0, 1]$. Let $\tstop$ denotes the 
  time that the algortihm exists the loop, which is number of queries the 
  algorithm makes. 
  
  Let us denote $k^\ast$, the best hypothesis in $\Theta$ i.e.
  \begin{equation*}
  	k^\ast = \argmin_{k \le K} \ff(\theta^{(k)}; \mc{S}).
  \end{equation*}
  We choose $\gamma$ such that $k^\ast$ is queried with probability 
  greater than
  $1-\alpha_1$, i.e., if $E_{\neg k^*}$ is the event (denote $\neg E_{\neg 
  k^*}$ as Event 2) that the 
  algorithm 
  finishes
  without querying $k^*$, we choose $\gamma$ such that
  $\P(E_{\neg k^*}) \le \alpha_1$. More precisely,
  \begin{align*}
	\P(E_{\neg k^*})
	& = \sum_{l = 1}^\infty
	\P(E_{\neg k^*}|\tstop = l)
	\P(\tstop = l)  \\
	& = \sum_{l=1}^\infty \prn*{1-\frac1K}^l \cdot 
	\prn*{1-\gamma}^{l-1}\cdot \gamma \\
	&= 
	\prn*{1-\frac1K}\gamma\sum_{l=0}^\infty \brk*{\prn*{1-\frac1K} 
		\prn*{1-\gamma}}^l \\
	& = \frac{\prn*{1-\frac1K}\gamma}{1 - 
		\prn*{1-\frac1K} 
		\prn*{1-\gamma}}.
\end{align*}

  Choosing $\gamma = \alpha_1 / K$ guarantees that
  $\P(E_{\neg k^*}) \le \alpha_1$.   
  Let $L \defeq \tfrac{\log(1/\alpha_1)}{\gamma} = 
  \log(1/\alpha_1)\tfrac{K}{\alpha_1}$, we 
  have
 
   \begin{align*}
 	\P(\tstop > L)  = 
 	\P(\omega_1 = \ldots = \omega_L = 0)  = (1-\gamma)^L  \le \exp(-L 
 	\gamma) = \alpha_1.
 \end{align*}

  Hence with probability at least $1 - \alpha_1$, the algorithm ends in less 
  than
  $L$ throws (Event 3). Conditioned on this event, by 
  Theorem~\ref{thm:winsorized} 
  and union bound, with 
  probability greater than
  $1 - L \cdot \tfrac{B}{\tau}\exp(-n\eps / 8)$, the output of $\msf{A}_{J_t}$
  for all $t \le \tstop$ is
  \begin{equation*}
    \msf{A}_{J_t}(S) = \ff(\theta^{(J_t)}; \cS) + \msf{Lap}\prn*{\frac{8\tau}{n\eps}}
    = \frac{1}{m\cdot n}\sum_{j \in \brk{m}, u \in \brk{n}}\f\prn*{\theta^{(J_t)};z^{(u)}_j}
    + \msf{Lap}\prn*{\frac{8\tau}{n\eps}},
  \end{equation*}
which we denote as Event 4.
  For a Laplace distribution, computing the tail gives that
  $\P(\abs{\msf{Lap}(\lambda)} \ge u) \le \exp(-u / \lambda)$ and with a union
  bound and change of variables it holds that if $Y_1, Y_2, \ldots, Y_L 
  \simiid \msf{Lap}(\frac{8\tau}{n\eps})$, then
  with probability greater than $1-\alpha_1$
  \begin{equation*}
    \max_{i=1,\ldots L} \abs{Y_i} \le 
    \frac{8\tau}{n\eps}\log\prn*{\frac{L}{\alpha_1}}.
  \end{equation*}
  In other words, except with probability $\alpha_1$, the noise is bounded 
  by
  $\tfrac{8\tau}{n\eps}\log(L / \alpha_1)$ (Event 5). Conditioned on all 
  these events, the
  parameter $\theta^{(J_{t^\ast})}$ that the algorithm outputs is 
  sub-optimal by
  at most $\tfrac{16\tau}{n\eps}\log(L / \alpha_1)$ as in the worst-case the 
  noise
  is $+\tfrac{8\tau}{n\eps}\log(L / \alpha_1)$ for $J_{t^*}$ and
  $-\tfrac{8\tau}{n\eps}\log(L / \alpha_1)$ for $k^\ast$.
    Setting %
    $\alpha_1 = \alpha/5$ and as we assume that
  $n \ge \frac{8}{\eps}\log\prn*{\tfrac{25\log(5/\alpha)}{\alpha^2} \cdot
  	\tfrac{KB}{\tau}}$, we conclude the proof by taking a union bound over all 
  	5 events.
  
\end{proof}

 }

\begin{corollary}
  Assume
  $n \ge \Omloglog(1) \, \tfrac{1}{\eps}\max\crl*{\tfrac{1}{Km},
    \log\prn{Km}}$. It holds that
  \begin{equation}
    \minimaxpurebdd =  \Olog\prn*{B\crl*{\sqrt{\frac{\log K}{m\cdot n}}
      + \frac{\log^{3/2}\prn*{Knm\eps}}{n\sqrt{m}\eps}}},
\end{equation}
where $\Olog, \Omloglog$ ignores only numerical constants and log-log 
factors in 
this case.
\end{corollary}

\begin{proof} We get the result directly from
  Theorem~\ref{thm:finite-hyp}, by setting
  $\alpha = \log K/(n\sqrt{m}\eps)$, applying standard uniform
  convergence results for bounded losses with finite parameter set
  (Hoeffding bound) and ignoring log-log factors.
\end{proof}

\begin{corollary}[Parameter sets with finite metric entropy] Let us
  further assume that our loss functions are $G$-Lipschitz with
  respect to some norm $\norm{\cdot}$ with (finite) covering number
  $\msf{N}_{\norm{\cdot}}(\Theta, \Delta)$---i.e. there exists a set
  $\Gamma_{\norm{\cdot},\Delta} \subset \Theta$ such that
  $\abs{\Gamma_{\norm{\cdot}, \Delta}} =
  \msf{N}_{\norm{\cdot}}(\Theta, \Delta)$ and for all
  $\theta \in \Theta$, there exists
  $\tau \in \Gamma_{\norm{\cdot}, \Delta}$ such that
  $\norm{\theta - \tau} \le \Delta$. In this case, for any
  $\Delta > 0$ and applying Algorithm~\ref{alg:private-selection} with
  parameter set $\Gamma$ guarantees that
  \begin{multline*}
    \mathfrak{M}^{\msf{user}}_{m, n}(\Theta, \mc{F}_{B, (G,
      \norm{\cdot})}, \eps) 
    = \tilde{O}(1) \, \inf_{\Delta > 0}\crl*{B\brk*{\sqrt{\frac{\log
            \msf{N}_{\norm{\cdot}}(\Theta, \Delta)}{m\cdot n}} +
        \frac{\log^{3/2}\prn*{\msf{N}_{\norm{\cdot}}(\Theta,
            \Delta)nm\eps}}{n\sqrt{m}\eps}} + G\Delta}.
\end{multline*}
\end{corollary}

\begin{remark} \label{rm:l_infity}For
  $\norm{\cdot} = \ell_2, \Theta=\mathbb{B}^d_\infty(0, 1)$ and
  setting
  $\Delta = \tfrac{B}{G}\crl*{\sqrt{d/(mn)} +
    d^{3/2}/(n\eps\sqrt{m})}$, we directly get
  \begin{equation*}
    \mathfrak{M}^{\msf{user}}_{m, n}(\mathbb{B}^d_\infty(0, 1), \mc{F}_{B, (G, \ell_2)}, \eps)
    =  \Olog\crl*{B\sqrt{\frac{d}{m\cdot n}} + B\frac{d^{3/2}}{n\sqrt{m}\eps}}.
  \end{equation*}
  The first term, which corresponds to the statistical rate, is optimal (see 
  e.g. Proposition~2
  in~\cite{LevyDu19}). Whether the privacy rate is optimal remains open.
\end{remark}

\subsection{Information-theoretic lower bound}\label{sec:finite-lb}

We now prove a lower bound on
$\mathfrak{M}^{\msf{user}}_{m, n}(\Theta, \mc{F}_B, \eps)$ when
$\abs{\Theta} = K < \infty$. We follow the standard machinery of reducing
estimation to testing \citep{Yu97, Wainwright19} but under privacy
constraints \citep{BarberDu14, acharya2020differentially}.
\begin{restatable}[Lower bound for finite-hypothesis class]
  {theorem}{restateFiniteHypLb}\label{thm:finite-hyp-lb} Let
  $K, m, n \in \N, K < \infty, \eps \in \R_+,$ and $0\le B < \infty$. Assume
  $\log_2K \ge 32\log 2$ and
  $n\ge \log_2K \max\crl{\tfrac{1}{192\sqrt{m}\eps}, \tfrac{1}{96m}}$, there
  exists a sample space $\cZ$ and parameter set $\Theta$ with $\abs{\Theta} = K$
  and $\abs{\cZ} = \ceil{\log_2 K}$ such that the following holds
  \begin{equation}
    \mathfrak{M}^{\msf{user}}_{m, n}(\Theta, \mc{F}_B, \eps) = \Omega \prn*{ 
    B\sqrt{\frac{\log \abs{\Theta}}{m\cdot n}}
    +B\frac{\log \abs{\Theta}}{n \sqrt{m} \epsilon}}.
\end{equation}
\end{restatable}

We detail the proof of the theorem \notarxiv{below}\arxiv{in
  Appendix~\ref{app:proof-finite-lb}}. The proof relies on a (standard)
generalization of Fano's method, whcih reduces optimization to multiple
hypothesis tests. We refer to the results of~\cite{acharya2020differentially} to
obtain the lower bounds in the case of a constrained---in this case,
$\epsilon$-DP---estimators. For the user-level case, we simply consider that
samples from an $m$-fold product of measures---the separation does not change but
the KL-divergence increase by at most a $m$ factor and TV-distance increase by
at most a $\sqrt{m}$ factor thus yielding the final answer.

\notarxiv{
\begin{proposition}[{\citet[][Corollary~4]{acharya2020differentially}}]\label{prop:private-fano}
  Let $\mc{P}$ be a collection of distributions over a common sample space $\ZZ$
  and a loss function $\f:\Theta\times\ZZ \to \R_+$. For $P, Q\in\mc{P}$, define
  \begin{equation*}
    \mathsf{sep}_{\ff}(P, Q; \Theta) := \sup\left\lbrace \delta \ge 0
      \;\middle|\; \mbox{for all~}
      \theta\in\Theta, \begin{array}{c} \ff(\theta, P) \leq \delta ~\mbox{implies}~
                         \ff(\theta, Q) \geq \delta \\
                         \ff(\theta, Q) \leq \delta ~\mbox{implies}~
                         \ff(\theta, P) \geq \delta
                       \end{array}\right\rbrace.
                   \end{equation*}

                   Let $\mc{V}$ be a finite index set and
                   $\mc{P}_{\mc{V}} \defeq \crl*{P_v}_{v\in\mc{V}}$ be a
                   collection of distributions contained in $\mc{P}$ such that
                   for $\Delta \ge 0$,
                   $\min_{v\neq v'}\msf{sep}(P_v, P_{v'}, \Theta) \le \Delta$.
                   Then
                   \begin{equation*}
                     \minimaxitem \ge
                     \frac{\Delta}{4}\max\crl*{1 - \frac{I(X_1^n;V) + \log 2}{\log \abs{\mc{V}}},
                       \min\crl*{1, \frac{\abs{\mc{V}}}{\exp(c_0 n\eps \mathrm{d}_\msf{TV}(\mc{P}_\mc{V}))}}},
                   \end{equation*}
                   where
                   $V \sim \msf{Uniform}(\mc{V}), c_0 = 10,
                   \mrm{d}_\msf{TV}(\mc{P}_\mc{V}) \defeq \max_{v\neq
                     v'}\norm{P_v - P_{v'}}_\msf{TV}$ and $I(X; Y)$ is the
                   (Shannon) mutual information.
                 \end{proposition}

\begin{proof}[Proof of Theorem~\ref{thm:finite-hyp-lb}] We follow the 
standard
    steps: we first compute the separation, we bound the testing error for any
    (constrained) estimator in the item-level DP case (with
    Proposition~\ref{prop:private-fano}) and finally, we show how to adapt the
    proof to obtain the user-level DP lower bound.
    \paragraph{Separation} For simplicity, assume $K = 2^d$, if not, the problem
    is harder than for $\underline{K} = 2^{\floor{\log_2K}} \le K$ which is of
    the same order. Let us define the sample space $\ZZ$, the parameter set
    $\Theta$ and the loss function $\f$ we consider.

  We define
  \begin{equation*}
    \ZZ = \Theta \defeq \crl{-1, +1}^d\mbox{~~and~~}
    \f(\theta; z) \defeq B\sum_{j \le d}\mathbf{1}_{\theta_j = z_j}.
  \end{equation*}

  We consider $\mc{V}$ an $d/2$-$\ell_1$ packing of $\crl{\pm 1}^d$ of size at
  least $\exp(d/8)$---which the Gilbert-Varshimov bound (see e.g.,
  {\cite[][Ex. 4.2.16]{Vershynin19}}) guarantees the existence of---and consider
  the following family of distribution $\mc{P} = \crl{P_v: v\in\mc{V}}$ such
  that if $X \sim P_v$ then
  \begin{equation}\label{eqn:bcube_prob}
    X = \begin{cases}
      v_j e_j & \mbox{~~with probability~~} \frac{1+\Delta}{2d}\\
      -v_j e_j & \mbox{~~with probability~~} \frac{1-\Delta}{2d}.
    \end{cases}
  \end{equation}
  For $\theta\in\Theta$, we have that
  \begin{equation*}
    \ff(\theta;P_v) = \E_{P_v}\brk*{B\sum_{j\le d} \mathbf{1}_{\theta_j = Z_j}} 
    =
    B\sum_{j \le d}\frac{1+\theta_j v_j \Delta}{2d}.
  \end{equation*}
  Naturally, $\ff(\theta;P_v)$ achieves its minimum at $\theta_v^* = -v$ such
  that $\inf_{\theta'\in\theta}\ff(\theta;P_v) = B\tfrac{1-\Delta}{2}$. We now
  compute the separation by noting that
  \begin{equation}\label{eqn:bcube_loss}
    \mathsf{sep}_{\ff}(P_v, P_{v'}, \Theta) \ge \frac{1}{2} \min_{\theta'\in\Theta}
    \crl*{\ff(\theta'; P_v) + \ff(\theta';P_{v'}) - \ff(\theta_v^*;P_v) - \ff(\theta_{v'}^*; P_{v'})}.
  \end{equation}
  A quick computation shows that $\msf{sep}_{\ff}(P_v, P_{v'}, \Theta) \ge 
  \tfrac{B\Delta}{8}$ by
  noting that $\mrm{d}_{\msf{Ham}}(v, v') \ge d/4$.

  \paragraph{Obtaining the item-level lower bound} We can now use the results of
  Proposition~\ref{prop:private-fano}. We have that
  $\min_{v\neq v'}\msf{sep}_{\ff}(P_v, P_{v'}, \Theta) \ge
  \tfrac{B\Delta}{8}$. The identity
  $\mathrm{D}_{\mathrm{KL}}(P_v, P_{v'}) = 
  \Delta\log\tfrac{1+\Delta}{1-\Delta}
  \le 3\Delta^2$ implies that $I(Z^n;V) \le 3n\Delta^2$. Similarly, 
  Pinsker's inequality guarantees that
  \begin{equation*}
    \mrm{d}_\msf{TV} \le \sqrt{\frac{1}{2}\max_{v\neq v'}
      \mrm{D}_\mrm{KL}(P_v, P_{v'})} \le \sqrt{3/2}\Delta.
  \end{equation*}
  
  We put everything together and it holds that for $\Delta \in \brk{0, 1}$,
  \begin{equation}
    \minimaxitem \ge \frac{B\Delta}{32}\max\crl*{1 - \frac{3n\Delta^2 + 
    \log 2}{d/8},
      \min\crl*{1, \frac{\exp(d/8)}{\exp(30n\eps \Delta)}}}.
  \end{equation}
  Since $d \ge 32\log 2$, $\Delta = \sqrt{d/(96n)}$ guarantees that
  $1 - \frac{3n\Delta^2 + \log 2}{d/8} \ge 1/2$. On the other hand, setting
  $\Delta = \tfrac{5}{960}\tfrac{d}{n\eps}$, guarantees that
  $\min\crl*{1, \frac{\exp(d/8)}{\exp(30n\eps \Delta)}} \ge 1/2$. The 
  assumption
  on $n$ guarantees that these two values are in $\brk{0, 1}$ and thus setting
  $\Delta^* = \max\crl*{{\sqrt{d/(96n)}}, \tfrac{1}{192}\tfrac{d}{n\eps}}$ 
  which
  implies that
  \begin{equation*}
    \minimaxitem \ge \frac{B}{32}\crl*{\sqrt{\frac{d}{96n}} + \frac{1}{192}\frac{d}{n\eps}}.
  \end{equation*}

  \paragraph{Concluding for user-level DP} Let $m\in \N, m\ge 1$. For the
  user-level DP lower bound, the proof remains the same except that the
  collection $\mc{P}_{\mc{V}}$ becomes $\crl{P_v^m}_{v\in\mc{V}}$ i.e. the
  $m$-fold product distribution of $P_v$. The separation remains exactly the
  same but we now have
  \begin{equation*}
    \mrm{D}_\mrm{KL}(P_v^m, P_{v'}^m) \le 
    3m\Delta^2\mbox{~~and~~}\mrm{d}_\msf{TV}(\mc{P}_\mc{V})
    \le \sqrt{\frac{3m}{2}}\Delta.
  \end{equation*}
  Under the assumption
  $\Delta^* = \max\crl*{{\sqrt{d/(96mn)}},
    \tfrac{1}{192}\tfrac{d}{n\sqrt{m}\eps}}$ is less than $1$ and thus concludes
  the proof.
\end{proof}
 }

Note, the upper bound of Theorem~\ref{thm:finite-hyp} and the lower bound above
match only up to $\sqrt{\log K}$. Given that $K$ can be exponential in
the dimension---e.g. in the case of $\Theta$ being a cover of an
$\ell_p$ ball---the bound is only tight for ``small'' hypothesis
class. However,  it seems this
extra-factor cannot be removed using the techniques we present in this 
paper, as we need to both obtain uniform
concentration and bound the maximum of i.i.d. noise over $K$
samples---both of which are tight. We leave the problem of finding an
optimal estimator for this problem to future work.

}
\section{Limit of Learning with a Fixed Number of Users} \label{sec:limit}

In this section, we consider the following binary testing 
problem between $P_1$ and $P_2$ supported on $\{+B, -B\}$ where
\begin{align*}
	P_0(+B) = 1, & \;\;\;\;\; P_0(-B) = 0,\\
	P_1(+B) = 0, & \;\;\;\;\; P_1(-B) = 1.\\
\end{align*}
We prove the following result.
\begin{theorem} \label{thm:limit}
	For all user-level $(\eps, \delta)$-DP algorithm $\alg: \{+B, 
	-B\}^{m\times n} \rightarrow [0, 1]$, let $\cS$ be $n 
	\times m$ i.i.d samples from $P_{\vartheta}, \vartheta \in \{0, 1\}$, we 
	have when $\delta < 1/2ne^{n\eps}$,
	\[
		\max_{\theta \in \{0, 1\}} \expectation{ \Paren{\alg(\cS) - 
		\vartheta}^2} = 
		\Omega(e^{-n\eps}).
	\]
\end{theorem}

Before proving the theorem, we describe the implications of the theorem to 
applications considered in this work. Let $\Auser$ denote the set of all 
user-level $(\eps, \delta)$-DP algorithms.

\paragraph{Reduction from mean estimation} $P_0$ and $P_1$ are 
both bounded distributions. Moreover, we have $\mu_\vartheta = 
B(2\vartheta - 
1)$. For any user-level $(\eps, \delta)$-DP mean estimator 
$\what{\mu}:  \{+B, 
-B\}^{m\times n}  \rightarrow [-B, +B]$, set $\alg_{\what{\mu}}(\cS) = 
(\what{\mu} + B
)/(2B) \in [0, 1]$, we have $\forall \vartheta \in \{0, 1\}$,
\[
	\expectation{ \Paren{\what{\mu}(\cS) - 
	\mu_\vartheta}^2} = 
	4B^2 \expectation{ \Paren{\alg_{\what{\mu}}(\cS) - 
			\vartheta}^2} .
\]
We have
\begin{align*}
	\inf_{\what{\mu} \in \Auser} \max_{\vartheta \in \{0, 1\}} \expectation{ 
	\Paren{\what{\mu}(\cS) - 
			\mu_\vartheta}^2} & = 4B^2\inf_{\what{\mu} \in \Auser} 
			\max_{\vartheta 
			\in 
			\{0, 1\}} \expectation{ 
		\Paren{\alg_{\what{\mu}}(\cS) - 
		\vartheta}^2}\\
	&  \ge 4B^2\inf_{\alg\in \Auser} \max_{\vartheta 
	\in 
\{0, 1\}} \expectation{ 
\Paren{\alg(\cS) - 
\vartheta}^2} = \Omega(B^2e^{-n\eps}).
\end{align*}

\paragraph{Reduction from SCO} Let $\Theta = [-1, 1]$ and
$\ell(\theta, Z) = \theta \cdot Z$. Setting $B = \normgrad$. The loss is linear
(and thus convex), $\normgrad$-Lipschitz and satisfies
Assumptions~\ref{ass:smoothness} and~\ref{ass:subG}. For $P_\vartheta$,
\[
	\ploss(\theta, P_\vartheta) = \theta \normgrad(2\vartheta - 
	1).
\]
Hence the minimizer is $\theta^*_\vartheta = 1 - 2\vartheta$ and 
\[
	\ploss(\theta, P_\vartheta)  - \ploss(\theta^*_\vartheta, P_\vartheta)  = 
	(2\vartheta-1)\normgrad(\theta - 1 + 2\vartheta) = \normgrad(1 - 
	\theta(2\vartheta-1) ) 
	\ge \frac{\normgrad}{2}(\theta - 2\vartheta + 1)^2 = 
	\frac{\normgrad}{2}(\theta - 
	\mu_\vartheta)^2.
\]
With similar arguments as in the mean estimation reduction, we get
\[
	\inf_{\alg\in \Auser} \max_{\vartheta \in \{0, 1\}} \expectation{ 
		\ploss(\alg(\cS); P_\vartheta) - \min_{\theta \in [-1,1]} \ploss(\theta; 
		P_\vartheta)}  = \Omega(\normgrad e^{-n\eps}).
\]
\paragraph{Reduction from Bounded Losses} In the reduction from SCO, the loss is
uniformly bounded and thus this is a sub-problem of the boundeed loss class and
the same bound holds.

Finally, let us prove the theorem. 
\begin{proof}[Proof of Theorem~\ref{thm:limit}]
	Note that there is only two possible sets that each user can observe. 
	Let $S_+$ be the multiset consisting of $m$ copies of $+B$ and Let 
	$S_-$ be the multiset consisting of $m$ copies of $-B$. Let $\beta_1 = 
	\probof{\alg((S_+)^n) < 1/2}$ and $\beta_0 = 
	\probof{\alg((S_-)^n) \ge 1/2}$. We first show that these two 
	probabilities 
	cannot be simultaneously small.
	
	Since $(S_+)^n$ can be changed into $(S_-)^n$ by changing $n$ users' 
	samples, by group property of differential 
	privacy,
	\[
		1 - \beta_1 = \probof{\alg((S_+)^n) \ge 1/2} \le e^{n\eps } 
		\probof{\alg((S_-)^n) \ge 1/2} + ne^{n\eps }\delta =  e^{n\eps } 
		\beta_0 + ne^{n\eps }\delta.
	\]
	Similarly, we get
	\[
		1 - \beta_0 \le e^{n\eps }  
		\beta_1 + ne^{n\eps }\delta.
	\]
	Combining the two, we get:
	\[
		\beta_0 + \beta_1 \ge \frac{2(1 - n\delta e^{n \eps})}{1 + e^{n \eps}} 
		\ge \frac{1}{1 + e^{n \eps}}.
	\]
	Note that when $\vartheta = 1$, we have $\probof{\cS = (S_+)^n} = 1$. 
	Hence
	\[
		\expectsub{P_1}{ \Paren{\alg(\cS) - 
				1}^2} \ge \frac{1}{4} \probof{\alg((S_+)^n) < 1/2}.
	\]
	Similarly,
	\[
	\expectsub{P_0}{ \Paren{\alg(\cS) - 
			0}^2} \ge \frac{1}{4} \probof{\alg((S_-)^n) \ge 1/2}.
	\]
	We conclude the proof by noting that
	\[
		\max_{\vartheta \in \{0, 1\}} \expectation{ \Paren{\alg(\cS) - 
				\vartheta}^2} \ge \frac12\Paren{\expectsub{P_0}{ \Paren{\alg(\cS) 
				- 
					0}^2}  + \expectsub{P_1}{ \Paren{\alg(\cS) - 
					1}^2}}.
	\]
\end{proof} %
\section{Extension to Limited Heterogeneity Setting} \label{sec:extension}
In this section, we show that our results and techniques developed under 
the homogeneous setting (Assumption~\ref{ass:homogeneous}) can be 
extended to 
the setting with limited heterogeneity 
(Assumption~\ref{ass:heterogeneous}). 

In particular, we show that applying the algorithms under the i.i.d setting in 
a 
black-box fashion will work with an additional bounded error under the 
limited 
heterogeneity setting, stated in the 
theorem below.

\begin{theorem}\label{thm:reduction}
	Let $\alg: \cZ^{m \times n} \rightarrow \Theta$ be a learning algorithm 
	and $\ell: \cZ \times \Theta \rightarrow \RR_+$ be a loss function with 
	$\max_{z \in \cZ} \max_{\theta \in \Theta}\ff(\theta; z) \le B$.
	 Given samples
	$\cS = (S_1, \ldots, S_n) \sim \otimes_{u \in \brk{n}}(P_u)^m$, if under 
	Assumption~\ref{ass:homogeneous},  we have
	\[
		\expectation{\ff(\alg(\cS);P_0)} - \min_{\theta' \in
			\Theta}\ff(\theta'; P_0)  \le L(m, n),
	\]
	then under Assumption~\ref{ass:heterogeneous}, we have
	\[
		\max_{u} \left\{\expectation{\ff(\alg(\cS);P_u)} - \min_{\theta' \in
			\Theta}\ff(\theta'; P_u) \right\} \le L(m, n) + B (mn + 2)\hetero.
	\]
\end{theorem}

Before proving the theorem, we can see that for any learning 
task, when $\hetero < L(m,n)/(B(mn+2))$, we can get the same performance 
as in 
the homogeneous case up to constant factors. This is only inverse 
polynomial in the problem parameters for all considered tasks.
\begin{proof}
	We first show that $\cS$ have a similar distribution under 
	Assumption~\ref{ass:homogeneous} and~\ref{ass:heterogeneous} when 
	$\hetero$ is small.
	By sub-additivity of total variantion distance. Under 
	Assumption~\ref{ass:heterogeneous}, we have
	\begin{equation} \label{eqn:bound_homo}
	\dtv{\otimes_{u 
			\in \brk{n}}(P_u)^m}{(P_0)^{n \times m}} \le mn\hetero.
	\end{equation}
	By definition of TV distance, there exists a coupling $(\cS, \cS')$ where 
	$\cS \sim \otimes_{u 
		\in \brk{n}}(P_u)^m$, $\cS' \sim (P_0)^{n \times m}$ and 
	\[
		\probof{\cS \neq \cS'} \le mn\hetero.
	\]
	Since $\max_{z \in \cZ} \max_{\theta \in \Theta}\ff(\theta; z) \le B$, we 
	have
	\begin{equation}\label{eqn:diff_coupling}
		\expectation{\ff(\alg(\cS);P_0)} - \expectation{\ff(\alg(\cS');P_0)} \le  
		B 
		\times \probof{\cS \neq \cS'} \le Bmn\hetero.
	\end{equation}
	Under Assumption~\ref{ass:heterogeneous}, for all $u \in [n]$, 
	$\dtv{P_u}{P_0} \le \hetero$. For all $\theta \in \Theta$,
	\[
		\ff(\theta;P_0) - \ff(\theta;P_u)   \le B \hetero,
	\]
	Hence we have
	\begin{equation}\label{eqn:diff_min}
		\min_{\theta' \in
			\Theta}\ff(\theta';P_0) - \min_{\theta' \in
			\Theta}\ff(\theta';P_u)  \le \max_{\theta \in \Theta} 	
			|\ff(\theta;P_0) - \ff(\theta;P_u)    \le B \hetero,
	\end{equation}
	and 
	\begin{equation}\label{eqn:diff_expectation}
		\expectation{\ff(\alg(\cS');P_u)} - 
		\expectation{\ff(\alg(\cS');P_0)}  \le B \hetero.
		\end{equation}
	Therefore, for all  $u \in [n]$,
	\begin{align*}
		&\expectation{\ff(\alg(\cS);P_u)} - \min_{\theta' \in
			\Theta}\ff(\theta'; P_u) \\
		= & \;\;\Paren{ \expectation{\ff(\alg(\cS);P_u)} - 
			\expectation{\ff(\alg(\cS');P_u)} }+ \Paren{ 
			\expectation{\ff(\alg(\cS');P_u)} - 
			\expectation{\ff(\alg(\cS');P_0)}  }\\
			&+\Paren{  \expectation{\ff(\alg(\cS');P_0)}  - 	\min_{\theta' \in
				\Theta}\ff(\theta';P_0) } + \Paren{ \min_{\theta' \in
				\Theta}\ff(\theta';P_0) - \min_{\theta' \in
				\Theta}\ff(\theta';P_u) }  \\
			\le & \;\; L(m, n) + B (mn + 2)\hetero,
	\end{align*}
	where we bound each term using~\eqref{eqn:bound_homo}, 
	\eqref{eqn:diff_coupling}, \eqref{eqn:diff_min} and 
	\eqref{eqn:diff_expectation} respectively.
\end{proof} %
\section{Proofs for Section~\ref{sec:sq}}

\subsection{Private range estimation}\label{sec:private-range}
\begin{algorithm} [h]
	\caption{\textbf{PrivateRange($X^n, \eps, \tau, B$)}: Private Range
		Estimation~\cite{feldman2017median}}
	\begin{algorithmic}[1]\label{alg:priv_range}
		\REQUIRE $X^n := (X_1, X_2, ..., X_n) \in [-B, B]^n$, $\tau:$ 
		concentration
		radius, privacy parameter $\eps > 0$.
		\STATE Divide the interval $[-B, B]$ into $l = B/\tau$ disjoint bins\footnotemark{}, 
		each
		with width $2\tau$. Let $T$ be the set of middle 
		points of
		intervals.  \STATE $\forall i \in [n]$, let
		$X_i' = \min_{x \in T} |X_i - x|$ be the point in $T$ closest to $X_i$.
		\STATE $\forall x \in T$, define cost function
		\[
		c(x) = \max \{ |\{i \in [n] \mid X'_i < x\}|, |\{i \in [n] \mid X'_i >
		x\}|\}.
		\]
		\STATE Sample $x \in T$ based on the following distribution:
		\[
		\probof{\hat{\mu} = x} = \frac{e^{-\eps c(x)/2}}{\sum_{x' \in 
		T}e^{-\eps
				c(x')/2}}.
		\]
		\STATE Return $R = [\hat{\mu} - 2 \tau, \hat{\mu} + 2 \tau]$.
	\end{algorithmic}
\end{algorithm}
\footnotetext{The last interval is of length $2B - (t-1)\tau$ if $\tau$ 
	doesn't
	divide $B$.}

\subsection{Proof of Theorem~\ref{thm:winsorized}}\label{proof:thm-winsorized}

\restateWinsorized*

\begin{proof}
  The privacy guarantee of the algorithm follows from the composition theorem of
  DP and the privacy guarantees of the exponential and Laplace
  mechanisms. For utility, it is enough to show that with probability at
  least
  $1 - (\gamma + \frac{B}{\tau} \exp \Paren{-\frac{n\eps}{8}}), \forall i \in
  [n], X_i$ is not truncated, i.e. $X_i \in [\hat{\mu} - 2\tau, \hat{\mu} + 
  2\tau]$.
	
  Recall that $X'_i$ is the middle of the interval in which $X_i$ falls. By the
  definition of $(\tau, \gamma)$-concentration, with probability at least
  $1 - \gamma, \forall i \in [n]$,
  \[
    |X_i - x_0| \le \tau.
  \]
  This implies that $\forall i \in [n]$,
  \[
    |X'_i - x_0| \le 2\tau,
  \]
  hence so is the $\Paren{\frac14, \frac34}$-quantile of $\{X'_i\}_{i =
    1}^n$. According to~\cite{feldman2017median} (Theorem 3.1),
  Algorithm~\ref{alg:priv_range} outputs $\Paren{\frac14, \frac34}$-quantile of
  $\{X'_i\}_{i = 1}^n$ with probability at least
  $1 - \frac{B}{\tau}e^{-\frac{n\eps}{8}}$. The proof follows by a union bound
  of both events.
\end{proof}

\subsection{Proof of Theorem~\ref{thm:winsorized_highd}} 
\label{proof:winsorized_highd}
\restateWinsorizedhighd*

We start by proving the following Lemma, which states that if the data is 
concentrated in $\ell_2$-norm with radius $\tau$, then after a random 
rotation, the points are concentrated in $\ell_\infty$-norm with radius 
$\tau/\sqrt{d}$ up to logarithmic factors.

\begin{lemma}\label{lem:random-rotation}
	Let $U = \frac{1}{\sqrt{d}} \whmat D$, where $\whmat$ is the  Walsh Hadamard 
	matrix and $D$ is a diagonal matrix with i.i.d. uniformly random $\{ +1, 
	-1 
	\}$ entries. Let $x_1, 
	x_2, \ldots, x_n$ and $x_0$ be vectors in $\R^d$. With probability at 
	least 
	$1-\alpha$, then the following holds.
	\[
	\max_i \| Ux_i - Ux_0\|_\infty \leq  \frac{10\max_i \|x_i - x_0\|_2 \sqrt{\log 
	\frac{nd}{\alpha}}}{\sqrt{d}}.
	\]
\end{lemma}
\begin{proof}
	Let $z_i = x_i - x_0$. It suffices to show that
	\[
	\max_i \| U z_i\|_\infty \leq  \frac{10\max_i \|z_i\|_2 \sqrt{\log 
	\frac{nd}{\alpha}}}{\sqrt{d}}.
	\]
	holds with probability at least $1-\alpha$.
	Let $y_i = U z_i$ and let $y_{i,j}$ denote the $j^\text{th}$ coordinate of 
	$y_j$. Let $D_j$ denote that $j^{th}$ diagonal of $D$. Then
	\[
	y_{i,j} =  \frac{1}{\sqrt{d}}  \sum_k \whmat{}_{j, k} D_k z_{i,k}
	\]
	Hence,
	\[
	\E[y_{i,j}] = \frac{1}{\sqrt{d}} \sum_k \whmat{}_{j, k} \E[D_k] z_{i,k} = 0.
	\]
	However, observe that changing one coordinate of $D$, say $D_{k}$ 
	changes the value of $y_{i,j}$ by at most
	\[
	y_{i,j} - y'_{i,j} \leq \frac{2}{\sqrt{d}} z_{i,k}
	\leq  \frac{2\|z_i\|_2 }{\sqrt{d}}.
	\]
	Hence, by the McDiarmid's inequality with probability at least $1- \alpha'$
	\[
	|y_{i,j}| \leq \frac{10\|z_i\|_2 \sqrt{\log \frac{1}{\alpha'}}}{\sqrt{d}}.
	\]
	Choosing $\alpha' = \alpha / n d$ and applying union bound over all 
	coordinates of all vectors yields the desired bound.
\end{proof}

Thus, after applying the random rotation, we have with probability $1 - 2\gamma$
that for all $j\in\brk{d}$, $\{Y_i(j)\}_{u \in [n]}$ is
$(\tau', 0)$-concentrated with $\tau' = 10\tau \sqrt{\log(nd/\alpha)/d}$. Hence
conditioned on this event, by Theorem~\ref{thm:winsorized} and a union bound over
$d$ coordinates, after applying $\textbf{WinsorizedMean1D}$ to each dimension,
we have that for all $j \in [d], \bar{Y}(j) \sim_{\beta} \bar{Y}'(j) $ where
$\beta = 1 - \frac{\sqrt{d}B}{\tau'} \exp \Paren{-\frac{n\eps'}{8}} $ and
	\[
		\bar{Y}'(j)  = \frac{1}{n} \sum_{i= 1}^n Y_i(j) + 
		\lap{\frac{8\tau'}{n \eps'}},
	\]
	Plugging in values of $\tau'$ and $\eps'$, it can be seen that $\bar{Y}'$ 
	satisfies the conditions in the theorem. By subadditivity of TV distance, 
	we have
	\[
		\bar{Y} \sim_{d\beta} \bar{Y}'.
	\]
	The theorem follows by noting the random rotation is an orthogonal transform 
	and preserves variance.

\subsection{Proof of Corollary~\ref{coro:bounded}} \label{proof:bounded} 

For all $i \in [n],$ let $X_i = \frac{1}{m} \sum_{j = 1}^m Z_j^{(i)}$, i.e., the
average of user $i$'s samples. Since $\norm{Z_j^{(i)}} \le B$, we know that
$X^n$ is $(B\sqrt{\log(2 n/\gamma)/(2m)}, \gamma)$-concentrated (e.g.,
see~\cite{JinNeGeKaJo19}). Hence by Theorem~\ref{thm:winsorized_highd}, if we
apply Algorithm~\ref{alg:winsorized_highd} to $X^n$, we have
$\alg(X^n) \sim_{\beta} \alg'(X^n)$ with
$\beta = \min \{1, \gamma + \alpha + \frac{d^2 B \sqrt{\log(dn/\alpha)}}{\tau}
\exp (-\frac{n\eps}{24\sqrt{d \log(1/\delta)}})\}$ with
$\tau = B\sqrt{\log(2 n/\gamma)/(2m)}$ and

\begin{equation} 
\E\brk*{\alg'(X^n) | X^n} = \frac1n \sum_{i = 1}^n X_i \mbox{~~and~~}
\Var\prn*{\alg'(X^n) | X^n} \le c_0\frac{d \tau^2 \log(dn/\alpha)
	\log(1/\delta)}{n^2 \eps^2}. \nonumber
\end{equation}
Hence 
\[
\E\brk*{\alg'(X^n) } = \E \brk{\E\brk*{\alg'(X^n) | X^n} } = \E\brk*{\frac1n 
	\sum_{i = 1}^n X_i } = \mu.
\]
\begin{align}
\Var\prn*{\alg'(X^n)} & = \E\brk*{ \Var\prn*{\alg'(X^n) | X^n} } + 
\Var\prn*{\E\brk*{\alg'(X^n) | X^n} } \nonumber \\
& \le \Var\prn*{\frac1n 
	\sum_{i = 1}^n X_i} + c_0 \frac{d \tau^2 \log(dn/\alpha)
	\log(1/\delta)}{n^2 \eps^2} \nonumber \\
& = \frac{\var(P_0)}{m n} + c_0\frac{d B^2 \log(2n/\gamma) 
	\log(dn/\alpha)
	\log(1/\delta)}{mn^2 \eps^2} \nonumber.
\end{align}
Combining the two, we have
\[
\expectation{\norm{\alg'(X^n) - \mu}_2^2} \le \frac{\var(P_0)}{m n} + 
c_0\frac{d B^2\log(2n/\gamma) 
	\log(dn/\alpha)
	\log(1/\delta)}{mn^2 \eps^2}.
\]
Since $\alg(X^n) \sim_{\beta} \alg'(X^n)$, we have
\[
\expectation{\norm{\alg(X^n) - \mu}_2^2} \le  \frac{\var(P_0)}{m 
n} + 
c_0\frac{d B^2\log(2n/\gamma) 
	\log(dn/\alpha)
	\log(1/\delta)}{mn^2 \eps^2} + \beta B^2.
\]
Taking $\alpha =\gamma = 
\frac{c_0d}{3mn^2\eps^2}$, we have when $n \ge c_1 
\frac{\sqrt{d\log(1/\delta)}}{\eps} \log( d m^{3/2}\eps^2)$ for a 
constant $c_1$, we have
\[
\expectation{\norm{\alg(X^n) - \mu}_2^2}  \le \frac{\var(P_0)}{m n} 
+ 
c_0\frac{ 2dB^2 \log(mn^2\eps^2/d) 
	\log(mn^3\eps^2)) 
	\log(1/\delta)}{mn^2 \eps^2}.
\]

\paragraph{Tightness of Corollary~\ref{coro:bounded}.} The first term is the
classic statistical rate even with unconstrained access to the samples. We prove
the tightness of the second term using the following family of truncated
Gaussian distributions. The proof follows a similar line of argument of the
proof for Theorem~\ref{thm:dp_sco_lower} in Section~\ref{app:sco-lb}. For a mean
$\mu\in\R^d$, a covariance $\Sigma \in \R^{d\times d}$ and $B > 0$, we consider
the family of $\ell_\infty$-truncated Gaussians, meaning
\begin{equation} \label{eqn:truncated_gaussian}
Z \sim \Gtr(\mu, \Sigma, B) \mbox{~~if~~} Z_0 \sim \msf{N}(\mu, \Sigma)
\mbox{~~and set for all $j \in \brk{d}$~~} Z(j) = \frac{Z_0(j)}{\max\crl{1, \abs{Z_0(j)} / 
		B}}.
\end{equation}
In other words, the standard high-dimensional Gaussian distribution where the
mass outside of $\mathbb{B}_{\infty}^d(0, B)$ has been projected back onto the
hyperrectangle coordinate-wise.

In this proof, we will take $\Sigma = \sigma^2 I_d$. We first state the
following Lemma, proved in Section~\ref{app:sco-lb}, which shows that when $B$
is large enough compared to $\norm{\mu}_2$ and $\sigma$, then the expectation of
$\Gtr(\mu, \sigma^2 I_d, B/\sqrt{d})$ and $\mu$ are exponentially close in
$\ell_2$-norm.

\begin{restatable}{lemma}{restateTruncatedMean}\label{lem:mean_diff}
	Suppose $\norm{\mu}_2 + 10\sqrt{d}\sigma < \normgrad$,
	\[
	\norm{ \expectsub{Z \sim \Gtr(\mu, \sigma^2 I_d, 
			\normgrad/\sqrt{d})}{Z} - \mu}_2 =
	\O\Paren{\sigma e^{-10d}}.
	\]
\end{restatable}

\paragraph{Reducing to standard Gaussian mean estimation} We will take
$\sigma = B/20\sqrt{d}$ and $\norm{\mu}_2 \le B/2$, Hence assuming $m, n$ is
polynomial in $d$, $\O\Paren{\sigma e^{-10d}}$ is small compared to the bound in
Corollary~\ref{coro:bounded}. Note that we can always simulate a sample from
$\Gtr(\mu, \sigma^2 I_d, B/\sqrt{d})$ using a sample from
$\msf{N}(\mu, \sigma^2I_d)$ by performing truncation. Taking
$\sigma = B/20\sqrt{d}$, it would be enough to prove the following:
\[
	\inf_{\what{\mu}\in\mc{A}^{\msf{item}}_{\eps, \delta}}
	\sup_{\mu: \norm{\mu}_2 \le B/2}
	\E_{\cS \simiid  \msf{N}(\mu, \sigma^2I_d)}\brk*{\norm*{\what{\mu}(\cS) - \mu}_2^2} = \tilde{\Omega}\Paren{
		\frac{d^2\sigma^2}{mn^2\eps^2}},
\]
where $\mc{A}_{\eps, \delta}^{\msf{user}}$ denotes set of all user-level
$(\eps,\delta)$-DP algorithms. The next proposition, based on the fact that
sample mean is a sufficient statistic for i.i.d Gaussian samples, shows that we
can reduce the problem to Gaussian mean estimation under item-level DP, with a
smaller variance. The proposition is proved in Section~\ref{app:sco-lb}.

\begin{restatable}[From multiple samples to one good sample]{proposition}
  {restateItemtoUser}\label{prop:item-to-user}
	Suppose each user $u\in\brk{n}$ observe
	$(Z_1^{(u)}, \ldots, Z_m^{(u)}) \simiid \msf{N}(\mu, \sigma^2I_d)$. For any
	$(\eps, \delta)$ user-level DP algorithm $\alg^\msf{user}$, there exists an
	$(\eps, \delta)$-item-level DP algorithm $\alg^{\msf{item}}$ that takes as
	input $(\bar{Z}^{(1)}, \ldots, \bar{Z}^{(n)})$ with
	$\bar{Z}^{(u)} \defeq \tfrac{1}{m}\sum_{j\le m}Z^{(u)}_j$ and has the same
	performance as $\alg^{\msf{user}}$.
\end{restatable}

Since $\bar{Z}^{(u)}$ is a sample from $\msf{N}(\mu, \frac{\sigma^2}{m}I_d)$, it remains to prove
\[
\inf_{\what{\mu}\in\mc{A}^{\msf{item}}_{\eps, \delta}}
\sup_{\mu: \norm{\mu}_2 \le B/2}
\E_{Z^n \simiid  \msf{N}(\mu, \frac{\sigma^2}{m}I_d)}\brk*{\norm*{\what{\mu}(Z^n) - \mu}_2^2} = \tilde{\Omega}\Paren{
	\frac{d^2\sigma^2}{mn^2\eps^2}},
\]
where $\mc{A}^{\msf{item}}$ denotes set of all item-level $(\eps,\delta)$-DP
algorithms. This directly follows from~\citet[][Lemma~6.7]{kamath2019privately},
concluding the proof.

\subsection{Mean Estimation of Sub-Gaussian Distribution} 
\label{proof:subg}
In this section, we prove error guarantees for mean estimation of
sub-Gaussian distributions. We note that known results in mean estimation of
Gaussian distributions and moment bounded
distributions~\cite{kamath2019privately, kamath2020private} imply this bound. We
include it here for the sake of completeness to demonstrate the strength of our
techniques.
\begin{corollary} \label{coro:subg} Suppose $P$ is a $\sigma$-sub-Gaussian
  distribution supported on $[-B, B]^d$ with mean $\mu$. Assume
  $n \ge (c_1 \sqrt{d\log(1/\delta)}/\eps) \log( B(dn + n^2\eps^2)/\sigma)$ for
  a numerical constant $c_1 < \infty$, if $X^n \simiid P$, the output
  $\alg(X^n)$ of Algorithm~\ref{alg:winsorized_highd} statisfies\footnote{For
    precise log factors, see Appendix~\ref{proof:subg}.}
	\[
	\expectation{\norm{\alg(X^n) - \mu}_2^2} = 
	\tilde{O}\Paren{\frac{d\sigma^2}{n} + 
		\frac{d^2\sigma^2}{n^2\eps^2}}.
	\]
	Furthermore, the bound is tight up to logarithmic factors.
\end{corollary}

The proof is almost parallel to the proof of Corollary~\ref{coro:bounded} by 
noting that $X^n$ is $(\sigma \sqrt{d \log(2 
	n/\gamma)}, \gamma)$-concentrated and 
\[
	\Var\prn*{\frac1n 
			\sum_{i = 1}^n X_i} = \frac{d\sigma^2}{n}.
\]
The tightness of the result follows from Theorem~3.1 and Lemma~3.1 in
\citep{cai2019cost}, which proves lower bounds for mean estimation
of $k$-dimensional random variables supported on $[-\sigma, 
\sigma]^k$
under $(\eps, \delta)$-DP constraints.

\arxiv{\subsection{Proof of Theorem~\ref{thm:sq_highd}} \label{proof:sq_highd}}

\notarxiv{%
\subsection{Uniform concentration: answering many queries privately} 
\label{sec:uc}

The statistical query framework subsumes many learning algorithms. For example,
we easily express stochastic gradient methods for solving ERM in the
language of SQ algorithms (see beginning of
Section~\ref{sec:erm}). In the next theorem, we show that with a uniform
concentration assumption we can answer a sequence of adaptively chosen
queries with variance---or, equivalently, privacy cost---proportional to the
concentration radius of the queries instead of the full range.

\begin{theorem} \label{thm:sq_highd} If $(Z^n, \cQ_B^d)$ is
	$(\tau, \gamma)$-uniformly concentrated, then for any sequence of
	(possibly adaptively chosen) queries
	$\phi_1, \phi_2, ..., \phi_K \in \cQ_B^d$, there exists an
	$(\eps, \delta)$-DP algorithm $\alg$, such that 
	$\alg$ outputs $v_1, v_2, ..., v_K$ satisfying
	$(v_1, v_2, ..., v_K) \sim_{\beta} (v'_1, v'_2, ..., v'_K)$, where
	\[
	\expectation{v'_k|Z^n} = \frac{1}{n}\sum_{i = 1}^n \phi_k(Z_i)
	\mbox{~~and~~} \Var\prn*{v'_k|Z^n} \le \frac{8 c_0 d K\tau^2
		\log(Kdn/\gamma)
		\log^2(2K/\delta)}{n^2\eps^2} = 
		\tilde{O}\prn*{\frac{dK\tau^2}{n^2\eps^2}},
	\]
	where $c_0 = 102400$ and
        $\beta = \min \crl*{1, 2 \gamma + \frac{d^2 K B
            \sqrt{\log(dKn/\gamma)}}{\tau} \exp 
            \prn*{-\frac{n\eps}{48\sqrt{2dK\log(2/\delta)
                \log(2K/\delta)}}}}$.
\end{theorem}

The algorithm for Theorem~\ref{thm:sq_highd} is simply applying
Algorithm~\ref{alg:winsorized_highd} to $\crl{\phi_k(Z_i)}_{i\in[n]}$ with
$\eps_0 = \frac{\eps}{2\sqrt{2K\log(2/\delta)}}$ and
$\delta_0 = \frac{\delta}{2K}$ for each query. Algorithm~\ref{alg:wgd} is an
illustration of an application of this result.\arxiv{ The proof is given in
Appendix~\ref{proof:sq_highd}.}

}

\begin{proof}
For each query $\phi_k, k \in [K]$, the algorithm computes 
$\phi_k(Z_i), i \in [n]$ and returns
\[
	v_k = 
	\textbf{WinsorizedMeanHighD}\prn*{\crl{\phi_k(Z_i)}_{i\in[n]},\eps_0,
		\delta_0, \tau, 
		B, \gamma/K}
\]
where
\[
	\eps_0 = \frac{\eps}{2\sqrt{2K\log(2/\delta)}}, \;\;\; \delta_0 =
	\frac{\delta}{2K}.
\]

\paragraph{Privacy guarantee.} The proof is immediate and hinges on the
strong-composition theorem. Under the standard strong
composition results of~\cite[Theorem~III.3]{dwork2010boosting}, for
any $\delta' \in (0, 1]$, the output of Algorithm~\ref{alg:wgd} is
$(\bar{\eps}, \bar{\delta})$-user-level DP with
\begin{equation*}
\bar{\eps} = K\eps_0(\exp(\eps_0) - 1) + \sqrt{2K\ln(1/\delta')}\eps_0, 
\qquad
\delta = K\delta_0 + \delta'.
\end{equation*}
Plugging in values of $\eps_0, \delta_0$
concludes the proof.

\paragraph{Utility guarantee.} The proof follows is very similar to the proof of
Theorem~\ref{thm:winsorized_highd} with $\alpha = \gamma/K$. We conclude by
using the subadditivity of the TV distances (or equivalently, a union bound)
over all $K$ queries.\end{proof}

\section{Proofs from Section~\ref{sec:erm}}

\subsection{Uniform Concentration}\label{app:uc}

\restateUCGrad*

\begin{proof}
  The proof relies on a standard covering number argument. We know that
  $\sup_{\theta_1, \theta_2 \in \Theta}\norm{\theta_1 - \theta_2} \le R$. This
  implies that $\Theta \subset \mathbb{B}^d_2(\theta_0, R)$, where
  $\mathbb{B}^d_2(v, r)$ is the $d$-dimensional $\ell_2$-ball centered at
  $v\in\R^d$ of radius $r$. Without loss of generality, we assume
  $\theta_0 = 0$, i.e. the constraint set $\Theta$ is centered at $0$.

  Let us consider $\Gamma_{\norm{\cdot}_2}(\Theta, \Delta) \eqdef \Gamma$, a
  $\Delta$-net of $\Theta$ for the $\ell_2$ norm, i.e. such that
  $\abs{\Gamma} < \infty$ and that for all $\theta, \vartheta\in\Theta$,
  $\norm{\theta-\vartheta}_2 \le \Delta$. Standard results
  (e.g.~{\citet[][Corollary 4.2.13]{Vershynin19}}) guarantee that there exists
  such a set and that its cardinality is smaller than $(1+2R/\Delta)^d$.

  Since $\ell$ is uniformly $H$-smooth, for any sample $S$ we immediately have that
  \begin{equation*}
    \sup_{\theta\in\Theta}\norm{\nabla \ff(\theta;S) - \nabla\ff(\theta;P)}_2
    \le \max_{\vartheta \in \Gamma} \norm{\nabla \ff(\vartheta;S) - \nabla\ff(\vartheta;P)}_2
    + 2H\Delta.
  \end{equation*}
  Consequently, letting $t > 0$, we have that
  \begin{equation*}
    \P\prn*{\sup_{\theta\in\Theta}\norm{\nabla \ff(\theta;S) - \nabla\ff(\theta;P)}_2 \ge t}
    \le \P\prn*{\max_{\vartheta \in \Gamma}
      \norm{\nabla \ff(\vartheta;S) - \nabla\ff(\vartheta;P)} \ge t/2}
    + \P(H\Delta \ge t/4).
  \end{equation*}
  For the second term, we simply need to ensures that when choosing $t$ and
  $\Delta$, it holds that $H\Delta < t/4$. Let us now bound the first term. Once
  again, let us consider $\Xi$ a $1/2$-net of $\mathbb{B}_2^d(0, 1)$. For any
  $v\in\R^d$, it holds that
  \begin{equation*}
    \norm{v}_2 = \sup_{\norm{u}_2 \le 1} \tri{u, v} \le \max_{\tilde{u} \in \Xi} \tri{\tilde{u}, v}
    + \sup_{w\in\mathbb{B}_2^d(0, 1/2)}\tri{w, v} = \max_{\tilde{u} \in \Xi} 
    \tri{\tilde{u}, v}
    + \frac{1}{2}\norm{v}_2,
  \end{equation*}
  which implies that $\norm{v}_2 \le 2 \max_{\tilde{u} \in \Xi}\tri{\tilde{u}, 
  v}$. Thus,
  \begin{align*}
    \P\prn*{\max_{\vartheta \in \Gamma} \norm{\nabla \ff(\vartheta;S) - \nabla\ff(\vartheta;P)}_2 \ge t/2}
    & \le \P\prn*{\max_{\vartheta\in \Gamma, v \in \Xi} \tri{v, \nabla
    \ff(\vartheta;S) - \nabla\ff(\vartheta;P)} \ge t/4}\\
    & \le \abs{\Gamma} \cdot \abs{\Xi} e^{-\tfrac{mt^2}{2\sigma^2}} \\
    & = 5^d \prn*{1+\tfrac{2R}{\Delta}}^d e^{-\tfrac{mt^2}{2\sigma^2}},
  \end{align*}
  where the penultimate line follows from a union bound and
  Assumption~\ref{ass:subG} which guarantees that $\nabla 
  \ff(\vartheta;S)$ is a
  $\sigma^2/m$-sub-Gaussian vector. We set
  $t = \sigma\sqrt{\tfrac{2}{m}\prn{d\log(5 + 10R / \Delta) +
    \log(n/\alpha)}}$. Picking
  $\Delta = \min\crl{1, \tfrac{\sqrt{2}\sigma}{4H}\sqrt{\tfrac{d}{m}}}$ 
  and applying a union bound over $n$ points conclude 
  the proof.
  
\end{proof}

\subsection{Stochastic gradient algorithms}\label{app:conv-sgd}

\begin{algorithm}\label{alg:sgd-erm}
  \caption{Generic optimization algorithm}\label{alg:gen-opt}
  \begin{algorithmic}[1]
    \STATE \textbf{Input:} Number of steps $T$, stochastic first-order oracle
    $\msf{O}_{F, \nu^2}$, optimization algorithm with
    $\crl*{\mc{O}, \msf{Query}, \msf{Update}, \msf{Aggregate}}$, initial output
    $o_0$.  \FOR{$t=0, \ldots, T-1$} \STATE $\theta_t \gets \msf{Query}(o_t)$.
    \STATE $g_t \gets \msf{O}_{F, \nu^2}(\theta_t)$.  \STATE
    $o_{t+1} \gets \msf{Update}(o_t, g_t)$.  \ENDFOR \RETURN
    $\what{\theta}_T \gets \msf{Aggregate}(o_0, \ldots, o_T)$.
  \end{algorithmic}
\end{algorithm}

\begin{proposition}[Convergence of stochastic gradient
  methods]\label{prop:conv-sgd}

  Let $F:\Theta\to\R$ be an $\smooth$-smooth function. Assume that we 
  have access to a
  stochastic first-order gradient oracle with variance bounded by $\nu^2$, 
  denoted by
  $\msf{O}_{F, \nu^2}$. In each of the following cases, let $T$ be the desired
  number of calls to $\msf{O}_{F, \nu^2}$, there exist an optimization
  algorithm---defined by \textsf{Update}, \textsf{Query} and \textsf{Aggregate}
  and used as in Algorithm~\ref{alg:gen-opt}---with output
  $\what{\theta}_T \in \Theta$ such that the following convergence guarantees
  hold.
  \begin{enumerate}[label=(\roman*)]
  \item\label{item:conv} \cite[][Theorem~6.3]{bubeck2014convex} Assume $F$ is convex, then it holds that
    \begin{equation}
      \E\brk{F(\what{\theta}_T) - \inf_{\theta' \in \Theta}F(\theta')}
      \le O\Paren{\frac{\smooth R^2}{T} + \frac{\nu R}{\sqrt{T}}}.
    \end{equation}
  \item\label{item:sconv} \cite[][Corollary~32]{KulunchakovMa19} Assume that $F$ is
    $\mu$-strongly-convex, and that we have access to $\theta_0\in\Theta$ such that
    $F(\theta_0) - \inf_{\theta'\in\Theta}F(\theta') \le \Delta_0$, then it holds
    that
    \begin{equation}
      \E\brk{F(\what{\theta}_T) - \inf_{\theta' \in \Theta}F(\theta')}
      \le \O\Paren{\Delta_0 \exp\prn*{-\frac{\mu}{\smooth}T} + 
      \frac{\nu^2}{\mu 
      T}}.
    \end{equation}
  \item\label{item:nonconv} \cite[][Corollary~3.6]{DavisDr19} Let us define the
    \emph{gradient mapping} $\msf{G}_{F, \gamma}$
    \begin{equation*}
      \msf{G}_{F, \gamma}(\theta) \defeq \frac{1}{\gamma}\brk*{\theta - \Pi_{\Theta}\prn*
        {\theta - \gamma\nabla F(\theta)}}.
    \end{equation*}
    Assume that we have access to $\theta_0$ such that
    $\norm{\msf{G}_{F, 1/\smooth}(\theta_0)}_2 -\inf_{\theta'}\norm{\msf{G}_{F, 
    1/\smooth}(\theta')}_2 \le \Delta_1$, it
    holds that
    \begin{equation}
      \E\norm{\msf{G}_{F, 1/\smooth}(\what{\theta}_T)}_2^2 \le 
      \O\Paren{\frac{\smooth\Delta}{T} + 
      \nu\sqrt{\frac{\smooth\Delta_1}{T}}}.
    \end{equation}
  \end{enumerate}
\end{proposition}

\notarxiv{
\begin{remark}
  For convex functions, the algorithm is fixed-stepsize, averaged, projected
  SGD. For strongly-convex functions, the algorithm consists of projected SGD
  with a fixed stepsize and non-uniform averaging followed by a single restart
  with decreasing stepsize. Finally, in the non-convex case, the \textsf{Query}
  and \textsf{Update} sub-routine are also projected SGD with fixed stepsize
  while the \textsf{Aggregate} selects one of the past iterates uniformly at
  random.

\end{remark}}

\subsection{Proof of Theorem~\ref{thm:erm}}\label{app:proof-thm-erm}

\restateERM*

\begin{proof} 
	First note that the gradient estimation steps 
	(Step~\ref{step:grad_computation} and~\ref{step:mean_estimation}) in 
	Algorithm~\ref{alg:wgd} can 
	be viewed as answering $T$ adaptively chosen queries. 
	
	\paragraph{Privacy guarantees.} The privacy guarantee follows directly 
	from Theorem~\ref{thm:sq_highd}.
	
	\paragraph{Utility guarantees.}
	By 
	Proposition~\ref{prop:uc-grad}, we have the gradients are $(\tau, 
	\gamma/3)$-concentrated with $\tau = \sigma\sqrt{d 
	\log\prn*{\frac{RHm}{d\sigma}}/m
		+ \log\prn*{\frac{3n}{\gamma}}/m}$.
	Hence, Theorem~\ref{thm:sq_highd} guarantees that 
	\[
		(\bar{g}_0, \ldots, 
		\bar{g}_{T-1}) \sim_{\beta} (\bar{g}'_0, \ldots, 
		\bar{g}'_{T-1}),
	\]
	where $\beta = \min \crl*{1, \frac{2\gamma}{3} + \frac{d^2 T B
			\sqrt{\log(3dTn/\gamma)}}{\tau} \exp 
			\prn*{-\frac{n\eps}{48\sqrt{2dT\log(2/\delta)
					\log(2T/\delta)}}}}$ and $\forall i \in [T]$, $\bar{g}'_0$ is from 
					$\msf{O}_{\ff(\cdot; \cS), \nu^2}(\theta_t)$ with
					\[
						\nu^2 \le \frac{8 c_0 d T\tau^2
							\log(3Tdn/\gamma)
							\log^2(2T/\delta)}{n^2\eps^2} \le \frac{8 c_0 d^2 
							T\sigma^2
							\log(3Tdn/\gamma)
							\log^2(2T/\delta)\log(3\normpara \smooth 
							mn/d\sigma\gamma)}{n^2\eps^2}.
					\]
	Moreover, when $n \ge \tilde{\Omega}(1) 
	\sqrt{dT\log(2/\delta)\log(2T/\delta)\log(dmTB/\sigma\gamma)}/\eps$, 
	where 
	$\tilde{\Omega}(1)$ hides log-log factors, we have $\beta < \gamma$. 
	  \paragraph{Convergence rates} Finally, depending on the assumptions 
	  on the
	function $\ff(\cdot;\mc{S})$, we use the various results of
	Proposition~\ref{prop:conv-sgd} for the value of $\nu$ above. To make 
	the
	results simpler we note that for \ref{item:sconv} of
	Proposition~\ref{prop:conv-sgd}, we upper bound $\Delta_0$ by
	$G\normpara$ and for \ref{item:nonconv}, we upper 
	bound
	$\Delta_1$ by $HR^2$. This concludes the proof.

\end{proof}

\section{Proofs for Section~\ref{sec:sco}}\label{app:sco}

\subsection{Proofs for Theorem~\ref{thm:dp_sco_upper}}\label{app:sco-ub}

We begin with a result that guarantees that the (regularized) empirical risk
minimizer has good generalization properties. It relies on a combination of
convex analysis and stability arguments. This proof exists in the literature
(see, e.g.~\cite{ShalevShwartz2009StochasticCO}), we add it here for
completeness and with some small variation: (1) that the optimization is
constrained (2) that Assumption~\ref{ass:subG} might improve stability when
$\sigma\sqrt{d} \le G$.

\begin{proposition}[Generalization properties of regularized
  ERM]\label{prop:stability}
  Let $(Z_1, \ldots, Z_N) \simiid P$. Let $\ell:\Theta\times\ZZ \to \R$ be
  convex, $G$-Lipschitz with respect to the $\norm{\cdot}_2$ and such that
  Assumption~\ref{ass:subG} holds. Let us denote
  $\underline{G} = \min\crl{G, \sigma\sqrt{d}}$. Let
  \begin{equation*}
    \theta^\ast_{S, \lambda, \vartheta} \defeq \argmin_{\theta\in\Theta}
    \crl*{\ff(\theta; S) + \frac{\lambda}{2}\norm{\theta - \vartheta}_2^2}.
  \end{equation*}
  The following holds
  \begin{equation}
    \E\brk*{\ff(\theta^\ast_{S, \lambda, \vartheta}; P)} - \ff(\theta; P)
    \le \frac{\lambda}{2}\E\brk*{\norm{\theta - \vartheta}_2^2} +
    \Olog(1)\, \frac{G\underline{G}}{N\lambda}, \mbox{~~for all~~}
    \theta \in \Theta.
  \end{equation}
\end{proposition}

\begin{proof}
  We first show the stability of the minimizer of the regularized empirical
  risk. Let us consider $S_0 = \crl{Z_1, \ldots, Z_N}$ and $S_1 = \crl{Z'_1, \ldots, Z'_N}$
  where $Z_j = Z'_j$ for all $j \neq i$ in $\brk{N}$. We first show that
  \begin{equation*}
    \norm*{\theta^\ast_{S, \lambda, \vartheta} - \theta^\ast_{S', \lambda, \vartheta}}_2
    \le \Olog(1)\, \frac{\underline{G}}{N\lambda}.
  \end{equation*}
  For conciseness, we denote
  $\ff_b(\theta) \defeq \ff(\theta;S_b) +
  \tfrac{\lambda}{2}\norm{\theta-\vartheta}_2^2$ and
  $\theta^\ast_{S_b, \lambda, \vartheta} = \theta_b$ for $b\in\crl{0, 1}$. Since
  $\ff_0$ is $\lambda$-strongly-convex, its
  gradients are co-coercive, meaning
  \begin{equation*}
    \frac{\lambda}{2}\norm{\theta_0 - \theta_1}_2^2
    \le \tri*{\nabla \ff_0
      (\theta_0) - \nabla \ff_0
      (\theta_1),
      \theta_0 - \theta_1}.
  \end{equation*}

  First, let us note that
  $\nabla\ff_0(\theta_1) = \nabla \ff_1(\theta_1) + \tfrac{1}{N}(\nabla
  \f(\theta_1;Z_i) - \nabla\f(\theta_1;Z'_i))$. In other words,
  \begin{equation*}
    \frac{\lambda}{2}\norm{\theta_0 - \theta_1}_2^2
    \le \tri*{\nabla \ff_0(\theta_0), \theta_0 - \theta_1}
    + \tri*{\nabla \ff_1(\theta_1), \theta_1 - \theta_0}
    + \frac{1}{N}\tri*{\nabla \f(\theta_1;Z_i) - \nabla \f(\theta_1;Z'_i), \theta_1 - \theta_0}.
  \end{equation*}
  Since $\theta_b$ is the minimizer of $\ff_b(\cdot)$ constrained in $\Theta$
  for $b\in\crl{0, 1}$, by first-order optimiality condition, it holds that
  \begin{equation*}
    \tri{\nabla \ff_b(\theta_b), \theta_b - \theta_{1-b}} \le 0.
  \end{equation*}
  Consequently,
  \begin{equation*}
    \frac{\lambda}{2}\norm{\theta_0 - \theta_1}_2^2 \le
    \frac{1}{N}\tri*{\nabla \f(\theta_1;Z_i) - \nabla \f(\theta_1;Z'_i), \theta_1 - \theta_0}
    \le \frac{1}{N}\norm{\nabla \f(\theta_1;Z_i) - \nabla \f(\theta_1;Z'_i)}_2
    \norm{\theta_1 - \theta_0}_2.
  \end{equation*}
  Since $\f(\cdot;z)$ is $G$-Lipschitz for all $z\in\ZZ$, we have that
  $\norm{\nabla \f(\theta_1;Z_i) - \nabla \f(\theta_1;Z'_i)}_2 \le 2G$. However,
  with the addition of Assumption~\ref{ass:subG}, Proposition~\ref{prop:uc-grad}
  (applied with $m=1$) guarantees that with probability greater than $1-\alpha$,
  \begin{equation*}
    \sup_{\theta \in \Theta} \norm{\nabla \ff (\theta; Z_i) - \nabla \ff(\theta;P)}
    \le \Olog(1)\, \sigma\sqrt{d},
  \end{equation*}
  where we note that the dependence is only \emph{logarithmic} in $\alpha$. This
  immediately yields that with probability greater than $1-\alpha$,
  \begin{equation*}
    \frac{\lambda}{2}\norm{\theta_0 - \theta_1}_2
    \le \Olog(1)\, \frac{\underline{G}}{\lambda N}.
  \end{equation*}

  Finally, this implies that
  \begin{equation*}
    \mbox{for all~} z\in\ZZ,
    \E\brk*{\abs*{\f(\theta_0;z) - \f(\theta_1;z)}} \le G\E\brk*{\norm{\theta_0 - \theta_1}_2}
    \le \Olog(1) \, \frac{G\underline{G}}{\lambda N},
  \end{equation*}
  by $G$-Lipschitzness of $\f$ and setting
  $\alpha = \frac{\underline{G}}{\lambda N R}$, or in the language of stability
  (see e.g.~\cite{bousquet2002stability}),
  $S \to \theta^\ast_{S, \lambda, \vartheta}$ is
  $\tfrac{G\underline{G}}{\lambda N}$-uniformly-stable. Standard stability
  arguments let us conclude the proof.
\end{proof}

We now state and prove Theorem~\ref{thm:dp_sco_upper}.

\restateThmSCOUpperBound*

\begin{proof}
  The proof hinges on repeatedly using of Corollary~\ref{coro:erm_sc} and
  Proposition~\ref{prop:stability} after decomposing the excess risk. Recall
  that $\what{\theta}_t$ is the output of round $t$ i.e.
  \begin{equation*}
    \what{\theta}_t \approx \argmin_{\theta\in\Theta}\ff(\theta;S_t)
    + \frac{\lambda_t}{2}\norm{\theta-\what{\theta}_{t-1}}_2^2.
  \end{equation*}
  We denote by $\theta_t^\ast$ the \emph{true} minimizer at round $t$ i.e.
  \begin{equation*}
    \theta_t^\ast \defeq \argmin_{\theta\in\Theta}\ff(\theta;S_t)
    + \frac{\lambda_t}{2}\norm{\theta-\what{\theta}_{t-1}}_2^2.
  \end{equation*}
  Let us denote
  $\theta^\ast = \argmin_{\theta\in\Theta}\ff(\theta;P)$, we decompose the
  regret in the following way
  \begin{align*}
    \E\brk{\ff(\what{\theta}_T;P) - \ff(\theta^\ast;P)}
    & =
    \underbrace{\E\brk*{\ff(\what{\theta}_T;P) - \ff(\theta_T^\ast;P)}}_{\eqdef \Delta_0}
  + \underbrace{\sum_{t=2}^T \E\brk*{\ff(\theta^\ast_t;P) - \ff(\theta_{t-1}^\ast;P)}}_{\eqdef \Delta_1} \\
& + \underbrace{\E\brk*{\ff(\theta_1^\ast;P) - \ff(\theta^\ast;P)}}_{\eqdef \Delta_2}.
\end{align*}

By Proposition~\ref{prop:stability} and because $\Theta$ is bounded by $R$, we
directly have that
\begin{equation*}
  \Delta_2 \le \frac{\lambda_1 R^2}{2} + \Olog\prn*{\frac{G\Gbar}{\lambda_1 n_1 m}}.
\end{equation*}

Turning to $\Delta_1$, for every $t \in \crl{2, \ldots, T}$, again by Proposition~\ref{prop:stability}, it holds that
\begin{align*}
  \E\brk*{\ff(\theta^\ast_t;P) - \ff(\theta_{t-1}^\ast;P)}
  & \le \frac{\lambda_t}{2}\E\brk*{\norm{\theta^\ast_{t-1} - \what{\theta}_{t-1}}_2^2}
  + \Olog\prn*{\frac{G\Gbar}{\lambda_t n_t m}}. \\
  & \le \Olog\prn*{\frac{\lambda_t}{2}\frac{\sigma^2 d^2}{\lambda^2_{t-1} n_{t-1}^2 m \eps^2}
    + \frac{G\Gbar}{\lambda_t n_t m}} \\
  & \le \Olog\prn*{\frac{\sigma^2 d^2}{\lambda_{t-1} n_{t-1}^2 m \eps^2}
    + \frac{G\Gbar}{\lambda_t n_t m}},
\end{align*}
where the second inequality is an application of
Corollary~\ref{coro:erm_sc}\footnote{The condition on $n$ for the corollary
  holds when the condition on $n$ is satisfied in the Theorem statement.} and
the third is because $\lambda_{t-1} = \lambda_t / 4$. Noting that
$\lambda_{t-1} n^2_{t-1} = 2^{t-1}\lambda n$, we have
\begin{equation*}
  \Delta_1 \le \Olog\prn*{(T-1)\frac{\sigma^2 d^2}{\lambda n^2 m \eps^2}
    + \frac{G\Gbar}{\lambda n m}\sum_{t=2}^T 2^{-t}} =
  \Olog\prn*{\frac{\sigma^2 d^2}{\lambda n^2 m \eps^2} + 
  \frac{G\Gbar}{\lambda n m}},
\end{equation*}
where we use that $T$ is logarithmic. Finally, using
Corollary~\ref{coro:erm_sc}, and that $\ff(\cdot;P)$ is $G$-Lipschitz, we have that
\begin{equation*}
  \Delta_0 \le \E\brk*{G\norm{\theta_T^\ast - \what{\theta}_T}_2}
  \le G\sqrt{ \E\brk*{\norm{\theta_T^\ast - \what{\theta}_T}_2^2}}
  = \Olog\prn*{\frac{G\sigma d}{2^T\lambda n \sqrt{m} \eps}}.
\end{equation*}

Combining the upper bounds, we have
\begin{equation*}
  \E\brk*{\ff(\what{\theta}_T;P) - \ff(\theta^\ast;P)}
  = \Olog\prn*{
    \frac{G\sigma d}{2^T\lambda n \sqrt{m} \eps}
    + \frac{\sigma^2 d^2}{\lambda n^2 m \eps^2} + \frac{G\Gbar}{\lambda n m}
    + \frac{\lambda_1 R^2}{2} + \frac{G\Gbar}{\lambda_1 n_1 m}},
\end{equation*}
and setting $T = \ceil{\log_2\prn{\tfrac{Gn\sqrt{m}\eps}{\sigma d}}}$ and
$\lambda = \sqrt{\tfrac{\sigma^2d^2}{n^2m\eps^2} +
  \tfrac{G\Gbar}{nm}} / R$ yields the final result.

\end{proof}

\subsection{Proofs of Theorem~\ref{thm:dp_sco_lower}}\label{app:sco-lb}

\restateThmSCOLowerBound*

The first term is a lower bound for SCO without any
constraints~\citep{LevyDu19,AgarwalBaRaWa12}. We only prove the second term
here. Note that without loss of generality, we can assume
$G \ge 20\sigma\sqrt{d}$ and prove a lower bound of
$\Omega(R\Gsig \sqrt{d}/n\sqrt{m} \eps )$. Else, we set $\sigma' = G/(20\sqrt{d})$ and
embed the original problem into a lower-dimensional (thus easier) problem where
the gradients are $\sigma'^2$sub-Gaussian. In the rest of the section, we
consider $\Theta = \mathbb{B}_2^d(0, R)$ for $R > 0$. As we explained in
Section~\ref{sec:sco_lower}, we consider the following loss\footnote{The
  negative sign is here for convenience; a positive sign would entail reducing
  it to finding the \emph{negative} normalized mean.}
\begin{equation*}
  \f(\theta;z) \defeq -\tri{\theta, z}.
\end{equation*}
Finally, we define (a collection) of data distributions. For a mean
$\mu\in\R^d$, a covariance $\Sigma \in \R^{d\times d}$ and $B > 0$, we consider
the family of $\ell_\infty$-truncated Gaussians. Recall the definition in~\eqref{eqn:truncated_gaussian},
\begin{equation*}
  Z \sim \Gtr(\mu, \Sigma, B) \mbox{~~if~~} Z_0 \sim \msf{N}(\mu, \Sigma)
  \mbox{~~and set for all $j \in \brk{d}$~~} Z(j) = \frac{Z_0(i)}{\max\crl{1, \abs{Z_0(j)} / 
  B}}.
\end{equation*}
In other words, the standard high-dimension Gaussian distribution where the mass
outside of $\mathbb{B}_{\infty}^d(0, B)$ has been radially projected back on 
the
sphere on each dimension.

Consequently, considering the data distribution $P = \Gtr(\mu, 
\sigma^2I_d, G/\sqrt{d})$,
$\f$ is almost surely $G$-Lipschitz. Additionally, both
assumptions~\ref{ass:smoothness} and \ref{ass:subG} hold.

We now formally state the reduction from SCO to Gaussian mean-estimation. The
main difficulty is that the mean of $\Gtr(\mu, \sigma^2I_d, G/\sqrt{d})$ and
$\msf{N}(\mu, \sigma^2I_d)$ do not coincide. However, we show that 
when $G$ is
sufficiently large compared to $\norm{\mu}_2$---which implies that we rarely
clip---then the reduction holds.

\begin{proposition}[Reduction from SCO to Gaussian mean estimation with item-level DP constraints]
  \label{prop:sco-gaussian-mean}
  Let $B > 0, \sigma > 0, G > 0$ such that $B + 10\sigma\sqrt{d} < G$, we
  consider the following collections of distributions
  \begin{equation*}
    \mc{P}_{\sigma, B} \defeq \crl*{\msf{N}(\mu, \sigma^2 I_d) : \norm{\mu}_2 \in [B/2, B]}
    \mbox{~~and~~}
    \mc{P}^\mrm{tr}_{\sigma, B, G/\sqrt{d}} \defeq \crl*{\Gtr(\mu, \sigma^2 
    I_d, G/\sqrt{d}) :
      \norm{\mu}_2 \in [B/2, B]}.
  \end{equation*}

  The following reduction holds
  \begin{align*}
    \inf_{\substack{\alg:\ZZ\to\Theta\\\alg\in\mc{A}^{\msf{item}}_{\eps, \delta}}}
    \sup_{P\in\mc{P}^\mrm{tr}_{\sigma, B, G/\sqrt{d}}}
      \E_{P} & \brk*{\ff(\A(Z^n); P) - \inf_{\theta'\in\Theta} \ff(\theta';P)}
    \ge \\
             & \frac{BR}{4}\inf_{\substack{\what{u}:\ZZ\to\mathbb{S}^{d-1} \\ \what{u}\in\mc{A}^{\msf{item}}_{\eps, \delta}}}
    \sup_{P \in
    \mc{P}_{\sigma, B}}
    \E_{P}\brk*{\norm*{\what{u}(Z^n) - \mu / \norm{\mu}_2}_2^2}+ \O\prn*{R\sigma e^{-10d}},
  \end{align*}
  where we recall that $\mc{A}^{\msf{item}}_{\eps, \delta}$ is the set of
  $(\eps, \delta)$-item-level DP algorithm for which the domain and co-domain
  are clear from context.
\end{proposition}

Before proving the proposition, we prove a Lemma that says, as previewed, that
when $G$ is large enough compared to $\norm{\mu}_2$ and $\sigma$, 
then the expectation of
$\Gtr(\mu, \sigma^2 I_d, G/\sqrt{d})$ and $\mu$ are exponentially close 
in $\ell_2$-norm.
\restateTruncatedMean*
\begin{proof}[Proof of Lemma~\ref{lem:mean_diff}]
	It would be enough to show that $\forall i \in [d]$,
	\[
		|\expectsub{Z \sim \Gtr(\mu, \sigma^2 I_d, \normgrad/\sqrt{d})}{Z}(i) 
		- 
		\mu(i)| = O(\sigma e^{-10d}/\sqrt{d}).
	\]
	Let $\alpha = \frac{\mu(i) + G/\sqrt{d}}{\sigma}$, $\beta = 
	\frac{\mu(i) - G/\sqrt{d}}{\sigma}$ and $\phi(x) = 
	\frac{1}{\sqrt{2\pi}}e^{-\frac{1}{2}x^2}$ be the density function 
	of $N(0,1)$. We have
	\[
		\expectsub{Z \sim \Gtr(\mu, \sigma^2 I_d, \normgrad/\sqrt{d})}{Z}(i) 
		= \mu(i) -
		\sigma \frac{\phi(\alpha) - \phi(\beta) 
		}{\int_{\alpha}^{\beta}\phi(x)dx}.
	\]
	Plugging in $\norm{\mu}_2 + 10\sqrt{d}\sigma < \normgrad$ we obtain 
	the lemma.
\end{proof}
We can now prove the proposition.
\begin{proof}
  Let $P = \Gtr(\mu, 
  \sigma^2I_d, G/\sqrt{d})$
  and denote $\mu^\mrm{tr} = \E_P\brk{Z}$ the mean of the truncated
  Gaussian. We consider $\theta_0 = - R \tfrac{\mu}{\norm{\mu}_2}$, in other
  words the minimum of $\ff(\theta;P)$, if the Gaussian was not truncated. Let $\theta\in\Theta$, we
  have that
  \begin{align*}
    \ff(\theta;P) - \ff(\theta^\ast;P)
    & \ge -\ff(\theta; P) - \ff(\theta_0;P) \\
    & = -\tri{\theta - \theta_0, \mu^\mrm{tr}} \\
    & = -\tri{\theta - \theta_0, \mu} - \tri{\theta - \theta_0, \mu^\mrm{tr} - \mu} \\
    & \ge -\tri{\theta - \theta_0, \mu} + 2 \inf_{\theta'} \tri{\theta', \mu^\mrm{tr} - \mu} \\
    & = -\tri{\theta - \theta_0, \mu} + \O(R\sigma e^{-10d}),
  \end{align*}
  where the final line uses the fact that
  $\inf_{\norm{v}_2 \le R}\tri{u, v} = - R \norm{u}_2$ and
  Lemma~\ref{lem:mean_diff}.

  Moreover, we have
  \begin{align*}
    \tri{\theta_0 - \theta, \mu}
    & = R\norm{\mu}_2 \prn*{1 - \tri*{\tfrac{\theta}{R}, \tfrac{\mu}{\norm{\mu}_2}}} \\
     & \ge \frac{R\norm{\mu}_2}{2}\prn*{\norm*{\tfrac{\theta}{R}}_2^2
      + \norm*{\tfrac{\mu}{\norm{\mu}_2}}_2^2 - 2\tri*{\tfrac{\theta}{R}, \tfrac{\mu}{\norm{\mu}_2}}} \\
    & =\frac{R\norm{\mu}_2}{2}\norm*{\tfrac{\theta}{R} - \tfrac{\mu}{\norm{\mu}_2}}_2^2,
  \end{align*}
  where we used that $\norm{\theta / R} \le 1$ and completed the square.

  We now finally prove the main statement of the proposition. The first
  observation is that, since the loss is linear, we only need to consider
  estimators $\alg:\ZZ^n \to \Theta$ such that $\norm{\alg(z^n)}_2 = R$ for all
  $z^n \in \ZZ^n$, as the minimum is always on the boundary\footnote{ To make
    this rigorous, we consider Yao's minimax principe. It holds that
    $\min_{\alg: |\alg|_2 \le \normpara} \max_{\mu} \E \brk{\bar{\ff}(\alg(Z^n);
      P)} = \max_{\cD} \min_{\alg: |\alg|_2 \le \normpara} \EE_{\mu \sim \cD}
    \E\brk{\bar{\ff}(\alg(Z^n);P) | \mu}$ where
    $\bar{\ff}(\theta;P) \defeq \ff(\theta;P) - \inf_{\theta'}\ff(\theta';P)$
    and $\cD$ is a prior over $\mu$. For a given prior $\cD$, the Bayes optimal
    classifier is the minimum of the posterior mean, which means that
    $\alg(Z^n)$ minimizes $\tri{\theta, \E[\mu|Z^n]}$ over
    $\mathbb{B}_2^d(0, R)$ and thus has norm $R$. We can thus constrain the
    class of estimators to be of norm exactly $R$ for any prior $\cD$. Another
    application of Yao's minimax principle guarantees that this is also the case
    for the original (minimax) problem.}. Consequently, we have
 \begin{align*}
   & \inf_{\alg: |\alg|_2 \le R} \sup_{P\in\mc{P}^\mrm{tr}_{\sigma, B, 
   G/\sqrt{d}}} \E 
     \brk{\ploss(\alg(Z^n); P) - \min_{\theta' \in \Theta 
     }\ploss(\theta';P)} \\
   = & \inf_{\alg: |\alg|_2 = R} \sup_{P\in\mc{P}^\mrm{tr}_{\sigma, B, 
   G/\sqrt{d}}} \E 
       \brk{\ploss(\alg(Z^n); P) - \min_{\theta' \in \Theta 
       }\ploss(\theta';P)} \\
   \ge  & \inf_{\alg: |\alg|_2 = R} \sup_{P\in\mc{P}^\mrm{tr}_{\sigma, B, 
   G/\sqrt{d}}} \E
          \frac{R\norm{\mu}_2}{2}\norm*{\tfrac{\alg(Z^n)}{R} - \tfrac{\mu}{\norm{\mu}_2}}_2^2
          + O\Paren{\normpara \rho e^{-10d}} 
           \\ 
   \ge & \inf_{\alg: |\alg|_2 = R} \sup_{P\in\mc{P}^\mrm{tr}_{\sigma, B, 
   G/\sqrt{d}}}\frac{RB}{4} 
         \E \norm*{\tfrac{\alg(Z^n)}{R} - \tfrac{\mu}{\norm{\mu}_2}}_2^2
         + O\Paren{\normpara \rho e^{-10d}} 
          \\
   = & \inf_{\what{u}: \norm{\hat{u}} = 
   1}\sup_{P\in\mc{P}^\mrm{tr}_{\sigma, B, G/\sqrt{d}}}\frac{RB}{4} 
       \E \norm*{\what{u}(Z^n) - \tfrac{\mu}{\norm{\mu}_2}}_2^2
       + O\Paren{\normpara \rho e^{-10d}} \\
   \ge & \inf_{\what{u}: \norm{\hat{u}} = 1}\sup_{P\in\mc{P}_{\sigma, B}}\frac{RB}{4} 
       \E \norm*{\what{u}(Z^n) - \tfrac{\mu}{\norm{\mu}_2}}_2^2
       + O\Paren{\normpara \rho e^{-10d}},
 \end{align*}
 where the last line uses that we can always sample from
 $\Gtr(\mu, \sigma^2I_d, G/\sqrt{d})$ using samples from $\msf{N}(\mu, 
 \sigma^2I_d)$ and
 truncating them, thus the problem over $\mc{P}^\mrm{tr}_{\sigma, B, G}$ is
 harder than over $\mc{P}_{\sigma, B}$. This concludes the proof.
\end{proof}

Because of this reduction, for the remainder of this proof we consider Gaussian
mean estimation with user-level DP constraints. Recall that in this setting, we
have $n$ users, each having $m$ i.i.d. samples from
$\msf{N}(\mu, \sigma^2 I_d)$. However, the lower bound
of~\cite{kamath2019privately} only holds for \emph{item-level} DP
constraints. In the next proposition, we show that mean estimation of
$\msf{N}(\mu, \sigma^2 I_d)$ with $n$ users and $m$ samples per user under
user-level DP constraints is equivalent to mean estimation of
$\msf{N}(\mu, \tfrac{\sigma^2}{m}I_d)$ with $n$ samples under \emph{item-level
  constraints}. In other words, any user-level DP estimator taking as input
$n\cdot m$ samples is equivalent to an item-level DP estimator taking as input
$n$ samples corresponding the each user's average.

\restateItemtoUser*

\begin{proof}
  First of all, note that for Gaussians with unknown mean but known variance,
  the sample mean is a sufficient statistic. As such, we have that for all
  $u\in\brk{n}$
\begin{equation*}
  \mbox{the distribution of~} (Z^{(u)}_1, \ldots, Z^{(u)}_m) | \bar{Z}^{(u)}
  \mbox{~does not depend on $\mu$}.
\end{equation*}
 Let us now consider an arbitrary user-level DP estimator $\msf{A}^\msf{user}$
  and show how to construct an equivalent item-level DP estimator. When provided
  with $(\bar{Z}^{(1)}, \ldots, \bar{Z}^{(n)})$, for each $j\le m$, we can
  sample
  \begin{equation}
    \tilde{S}_u = (\tilde{Z}^{(u)}_1, \ldots \tilde{Z}^{(u)}_m) \simiid (Z^{(u)}_1, \ldots, Z^{(u)}_m) | \bar{Z}^{(u)}
  \end{equation}
  and return
  $\msf{A}^{\msf{item}}((\bar{Z}^{(u)})_{u\le n}) =
  \msf{A}^{\msf{user}}((\tilde{S}_1, \ldots, \tilde{S}_n))$. Since the
  distributions are equal given $\bar{Z}^{(u)}$, in expectation the error is the
  same.
\end{proof}

This proposition allows us to reduce Gaussian mean estimation with user-level
DP, to Gaussian mean estimation with item-level DP albeit with the variance
divided by $m$. We thus conclude with (a slight modification of) the results
of~\cite{kamath2019privately}. Indeed, we differ only in that their results show
that mean estimation is hard, whereas we require that estimating the
\emph{direction of the mean} is hard.

First, let us recall the a modified version of the result 
in~\cite{kamath2019privately}\footnote{It is not guaranteed 
	in the lower bound construction of~\cite{kamath2019privately}
	that $B/2 \le \norm{\mu}_2 \le B$. In their construction, the mean is 
	taken uniformly from $[-\sqrt{2}B/\sqrt{d}, \sqrt{2}B/\sqrt{d}]^d$. 
	However, the probability that the 
	mean 
	in the lower bound construction being out of this range is exponentially 
	small in $d$. Hence the same lower bound can be obtained by 
	straightforward 
	modifications of the construction.}.

\begin{proposition}[{\citet[][Lemma~6.7]{kamath2019privately}}]\label{prop:dp-gauss-lb}
  Let $Z^n\simiid \msf{N}(\mu, \sigma^2 I_d)$ and assume\\
  $\delta \le \frac{\sqrt{d}}{48\sqrt{2}Bn\sqrt{\log(100Rn/\sqrt{d})}}$, then
  it holds that if $n < d\sigma/(512B \eps)$,
  \[
    \inf_{\hat{\mu}, \what{\mu}\in\mc{A}_{\eps, \delta}^{\msf{item}}} \sup_{\mu:
      B/2 \le \norm{\mu}_2 \le B} \E \brk*{\norm{\what{\mu}(Z^n) - 
      \mu}_2^2} \ge
    \frac{B^2}{6}.
  \]
\end{proposition}

\begin{corollary}[Estimating the direction of the mean is hard]\label{cor:lb-direction}
  Let $Z^n\simiid \msf{N}(\mu, \sigma^2 I_d)$, set
  $B = \frac{d\rho}{512n\eps}$ and assume that
  $\delta \le \frac{\sqrt{d}}{48\sqrt{2}Bn\sqrt{\log(100Rn/\sqrt{d})}}$, then
  it holds that if $n < d\sigma/(512B \eps)$,
  \[
    \inf_{\substack{\what{u}: \norm*{\what{u}}_2 = 1 \\
        \what{u}\in\mc{A}_{\eps, \delta}^{\msf{item}}}}
    \sup_{P\in\mc{P}_{\sigma, B}} \E \brk*{\norm*{\what{u}(Z^n) -
        \frac{\mu}{\norm{\mu}_2}}_2^2} \ge \frac{1}{10}.
  \]
\end{corollary}

\begin{proof}
  We prove the corollary by contradiction. Assume there exists  an $(\eps, \delta)$-DP
  estimator $\what{u}$ such that
  \[
    \sup_{P\in\mc{P}_{\sigma, B}} \E \brk*{\norm*{\what{u}(Z^n) -
        \frac{\mu}{\norm{\mu}_2}}_2^2} < \frac{1}{10}.
  \]
    
  Then let $\what{\mu} = \frac{3}{4B} \what{u}$,
  \begin{align*}
    \E [\norm*{\what{\mu}(Z^n) - \mu}_2^2]
    & = \E\brk*{\norm*{\frac{3B}{4} \what{u}(Z^n) - \mu}_2^2} \\
    & \le \E\brk*{\norm*{\frac{3B}{4} \what{u}(Z^n) - \norm{\mu}_2 \cdot \what{u}(Z^n)}}
      + \E\brk*{\norm*{\norm{\mu}_2 \cdot \what{u}(Z^n) - \mu}_2^2} \\
    & = \prn*{\frac{3B}{4} - \norm{\mu}_2}^2 + \norm{\mu}_2^2 \, \E\brk*{
      \norm*{\what{u}(Z^n) - \frac{\mu}{\norm{\mu}_2}}_2^2} \\
    & \le \frac{B^2}{16} + \frac{B^2}{10} \\
    & < \frac{B^2}{6},
  \end{align*}
  which contradicts with Proposition~\ref{prop:dp-gauss-lb}.
\end{proof}

 Applying
Corollary~\ref{cor:lb-direction} with $B = \sigma / \sqrt{m}$ concludes the 
proof
of the lower bound.

\arxiv{%
\section{Proofs from Section~\ref{sec:warm-up}}

\subsection{Proofs from 
Section~\ref{sec:finite-ub}}\label{app:proof-finite-ub}

\restateFiniteHyp*

\begin{proof} We first state the privacy guarantee followed by the utility
    guarantee.

    \paragraph{Proof of~\ref{item:privacy-priv-selection}} Since each
    $\msf{A}_k$ is $\eps/3$-user-level DP, Theorem~3.2 in~\cite{liu2019private}
    guarantees that the output of Algorithm~\ref{alg:private-selection} is
    $\eps$-user-level DP.

    \paragraph{Proof of~\ref{item:utility-priv-selection}} The proof is adapted
    from Theorem~5.2 in~\cite{liu2019private}. First of all, with probability
    greater than $1-\alpha_1$, as we prove in
    Lemma~\ref{lem:concentration-finite}, the data are uniformly concentrated for
    all $\theta^{(k)}$, meaning
      \begin{equation*}\label{eq:concentration-finite}
        \max_{k \in K} \max_{u \in \brk{n}} \abs*{\ff(\theta^{(k)}; S_u)
          - \ff(\theta^{(k)};P)} \le
        \crl*{\frac{B}{2}\sqrt{\frac{\log(\abs{\Theta} \cdot n) + 
        \log(2/\alpha_1)}{m}} \eqdef \tau}.
  \end{equation*}
  We condition on this event (Event 1) for the rest of the proof. Let
  $\alpha_1 \in (0, 1]$ and $\gamma \in (0, 1]$. Let $\tstop$ denotes the 
  time that the algortihm exists the loop, which is number of queries the 
  algorithm makes. 
  
  Let us denote $k^\ast$, the best hypothesis in $\Theta$ i.e.
  \begin{equation*}
  	k^\ast = \argmin_{k \le K} \ff(\theta^{(k)}; \mc{S}).
  \end{equation*}
  We choose $\gamma$ such that $k^\ast$ is queried with probability 
  greater than
  $1-\alpha_1$, i.e., if $E_{\neg k^*}$ is the event (denote $\neg E_{\neg 
  k^*}$ as Event 2) that the 
  algorithm 
  finishes
  without querying $k^*$, we choose $\gamma$ such that
  $\P(E_{\neg k^*}) \le \alpha_1$. More precisely,
  \begin{align*}
	\P(E_{\neg k^*})
	& = \sum_{l = 1}^\infty
	\P(E_{\neg k^*}|\tstop = l)
	\P(\tstop = l)  \\
	& = \sum_{l=1}^\infty \prn*{1-\frac1K}^l \cdot 
	\prn*{1-\gamma}^{l-1}\cdot \gamma \\
	&= 
	\prn*{1-\frac1K}\gamma\sum_{l=0}^\infty \brk*{\prn*{1-\frac1K} 
		\prn*{1-\gamma}}^l \\
	& = \frac{\prn*{1-\frac1K}\gamma}{1 - 
		\prn*{1-\frac1K} 
		\prn*{1-\gamma}}.
\end{align*}

  Choosing $\gamma = \alpha_1 / K$ guarantees that
  $\P(E_{\neg k^*}) \le \alpha_1$.   
  Let $L \defeq \tfrac{\log(1/\alpha_1)}{\gamma} = 
  \log(1/\alpha_1)\tfrac{K}{\alpha_1}$, we 
  have
 
   \begin{align*}
 	\P(\tstop > L)  = 
 	\P(\omega_1 = \ldots = \omega_L = 0)  = (1-\gamma)^L  \le \exp(-L 
 	\gamma) = \alpha_1.
 \end{align*}

  Hence with probability at least $1 - \alpha_1$, the algorithm ends in less 
  than
  $L$ throws (Event 3). Conditioned on this event, by 
  Theorem~\ref{thm:winsorized} 
  and union bound, with 
  probability greater than
  $1 - L \cdot \tfrac{B}{\tau}\exp(-n\eps / 8)$, the output of $\msf{A}_{J_t}$
  for all $t \le \tstop$ is
  \begin{equation*}
    \msf{A}_{J_t}(S) = \ff(\theta^{(J_t)}; \cS) + \msf{Lap}\prn*{\frac{8\tau}{n\eps}}
    = \frac{1}{m\cdot n}\sum_{j \in \brk{m}, u \in \brk{n}}\f\prn*{\theta^{(J_t)};z^{(u)}_j}
    + \msf{Lap}\prn*{\frac{8\tau}{n\eps}},
  \end{equation*}
which we denote as Event 4.
  For a Laplace distribution, computing the tail gives that
  $\P(\abs{\msf{Lap}(\lambda)} \ge u) \le \exp(-u / \lambda)$ and with a union
  bound and change of variables it holds that if $Y_1, Y_2, \ldots, Y_L 
  \simiid \msf{Lap}(\frac{8\tau}{n\eps})$, then
  with probability greater than $1-\alpha_1$
  \begin{equation*}
    \max_{i=1,\ldots L} \abs{Y_i} \le 
    \frac{8\tau}{n\eps}\log\prn*{\frac{L}{\alpha_1}}.
  \end{equation*}
  In other words, except with probability $\alpha_1$, the noise is bounded 
  by
  $\tfrac{8\tau}{n\eps}\log(L / \alpha_1)$ (Event 5). Conditioned on all 
  these events, the
  parameter $\theta^{(J_{t^\ast})}$ that the algorithm outputs is 
  sub-optimal by
  at most $\tfrac{16\tau}{n\eps}\log(L / \alpha_1)$ as in the worst-case the 
  noise
  is $+\tfrac{8\tau}{n\eps}\log(L / \alpha_1)$ for $J_{t^*}$ and
  $-\tfrac{8\tau}{n\eps}\log(L / \alpha_1)$ for $k^\ast$.
    Setting %
    $\alpha_1 = \alpha/5$ and as we assume that
  $n \ge \frac{8}{\eps}\log\prn*{\tfrac{25\log(5/\alpha)}{\alpha^2} \cdot
  	\tfrac{KB}{\tau}}$, we conclude the proof by taking a union bound over all 
  	5 events.
  
\end{proof}

\subsection{Proofs from Section~\ref{sec:finite-lb}}\label{app:proof-finite-lb}

\restateFiniteHypLb*

The lower bounds relies on two results: the standard reduction from (stochastic)
optimization to multiple hypothesis testing and lower bounded the testing
error. The first step is folklore and appears in a number of
works~\cite{LevyDu19, AgarwalBaRaWa12, Yu97}. For the second step, we state a
result of~\cite{acharya2020differentially}, that extends the classical Fano's
inequality for multiple hypothesis test under privacy constraints.

\begin{lemma}[{\cite[][Lemma~1]{LevyDu19}}] Let $\mc{P}$ be a collection of
    distributions over a common sample space $\ZZ$ and a loss function
    $\f:\Theta\times\ZZ \to \R_+$. For $P, Q\in\mc{P}$, define
	\begin{equation*}
	\mathsf{sep}_{\ff}(P, Q; \Theta) := \sup\left\lbrace \Delta \ge 0
	\;\middle|\; \begin{array}{c} \ff(\theta; P) - \min_{\theta'} \ff(\theta'; 
	P) \leq \Delta 
		~\mbox{implies}~
		\ff(\theta; Q)  - \min_{\theta'} \ff(\theta'; Q) \geq \Delta \\
		\ff(\theta; Q) - \min_{\theta'} \ff(\theta'; Q)  \leq \Delta 
		~\mbox{implies}~
		\ff(\theta; P) - \min_{\theta'} \ff(\theta'; P)  \geq \Delta
	\end{array}\right\rbrace.
    \end{equation*}
    Let $\mc{V}$ be a finite index set and
    $\mc{P}_{\mc{V}} \defeq \crl*{P_v}_{v\in\mc{V}}$ be a collection of
    distributions contained in $\mc{P}$ such that
    $\min_{v\neq v'}\msf{sep}(P_v, P_{v'}, \Theta) \ge \Delta$. Then for
    $V\sim \msf{Uniform}\prn{\mc{V}}$ and $Z^n|V = v \simiid P_v$, it holds that
    \begin{equation*}
      \minimaxitem \ge \Delta \inf_{\psi \in \mc{A}^{\msf{item}}_\eps}
      \mathbb{Q}(\psi(Z^n) \neq V),
    \end{equation*}
    where $\mathbb{Q}$ is the joint distribution over $V$ and $Z^n$.
  \end{lemma}

  \begin{proposition}[{\cite[][Corollary~4]{acharya2020differentially}}]
    \label{prop:private-fano}
    Let $\mc{V}$ be a finite index set and
    $\mc{P}_{\mc{V}} \defeq \crl*{P_v}_{v\in\mc{V}}$ be a collection of
    distributions contained in $\mc{P}$. Then for
    $V\sim \msf{Uniform}\prn{\mc{V}}$ and $Z^n|V = v \simiid P_v$, it holds that
    \begin{equation*}
      \inf_{\psi \in \mc{A}^{\msf{item}}_\eps}
      \mathbb{Q}(\psi(Z^n) \neq V) \ge \frac{1}{4}\max\crl*{1 - \frac{I(Z^n;V) + \log 2}{\log 
          \abs{\mc{V}}},
        \min\crl*{1, \frac{\abs{\mc{V}}}{\exp(c_0 n\eps 
            \mathrm{d}_\msf{TV}(\mc{P}_\mc{V}))}}},
    \end{equation*}
    where
    $c_0 = 10, \mrm{d}_\msf{TV}(\mc{P}_\mc{V}) \defeq \max_{v\neq v'}\norm{P_v -
      P_{v'}}_\msf{TV}$ and $I(X; Y)$ is the (Shannon) mutual information.
\end{proposition}

\begin{proof}[Proof of Theorem~\ref{thm:finite-hyp-lb}] We follow the 
standard
    steps: we first compute the separation, we bound the testing error for any
    (constrained) estimator in the item-level DP case (with
    Proposition~\ref{prop:private-fano}) and finally, we show how to adapt the
    proof to obtain the user-level DP lower bound.
    \paragraph{Separation} For simplicity, assume $K = 2^d$, if not, the problem
    is harder than for $\underline{K} = 2^{\floor{\log_2K}} \le K$ which is of
    the same order. Let us define the sample space $\ZZ$, the parameter set
    $\Theta$ and the loss function $\f$ we consider.

  We define
  \begin{equation*}
    \ZZ = \Theta \defeq \crl{-1, +1}^d\mbox{~~and~~}
    \f(\theta; z) \defeq B\sum_{j \le d}\mathbf{1}_{\theta_j = z_j}.
  \end{equation*}

  We consider $\mc{V}$ an $d/2$-$\ell_1$ packing of $\crl{\pm 1}^d$ of size at
  least $\exp(d/8)$---which the Gilbert-Varshimov bound (see e.g.,
  {\cite[][Ex. 4.2.16]{Vershynin19}}) guarantees the existence of---and consider
  the following family of distribution $\mc{P} = \crl{P_v: v\in\mc{V}}$ such
  that if $X \sim P_v$ then
  \begin{equation}\label{eqn:bcube_prob}
    X = \begin{cases}
      v_j e_j & \mbox{~~with probability~~} \frac{1+\Delta}{2d}\\
      -v_j e_j & \mbox{~~with probability~~} \frac{1-\Delta}{2d}.
    \end{cases}
  \end{equation}
  For $\theta\in\Theta$, we have that
  \begin{equation*}
    \ff(\theta;P_v) = \E_{P_v}\brk*{B\sum_{j\le d} \mathbf{1}_{\theta_j = Z_j}} 
    =
    B\sum_{j \le d}\frac{1+\theta_j v_j \Delta}{2d}.
  \end{equation*}
  Naturally, $\ff(\theta;P_v)$ achieves its minimum at $\theta_v^* = -v$ such
  that $\inf_{\theta'\in\theta}\ff(\theta;P_v) = B\tfrac{1-\Delta}{2}$. We now
  compute the separation by noting that
  \begin{equation}\label{eqn:bcube_loss}
    \mathsf{sep}_{\ff}(P_v, P_{v'}, \Theta) \ge \frac{1}{2} \min_{\theta'\in\Theta}
    \crl*{\ff(\theta'; P_v) + \ff(\theta';P_{v'}) - \ff(\theta_v^*;P_v) - \ff(\theta_{v'}^*; P_{v'})}.
  \end{equation}
  A quick computation shows that $\msf{sep}_{\ff}(P_v, P_{v'}, \Theta) \ge 
  \tfrac{B\Delta}{8}$ by
  noting that $\mrm{d}_{\msf{Ham}}(v, v') \ge d/4$.

  \paragraph{Obtaining the item-level lower bound} We can now use the results of
  Proposition~\ref{prop:private-fano}. We have that
  $\min_{v\neq v'}\msf{sep}_{\ff}(P_v, P_{v'}, \Theta) \ge
  \tfrac{B\Delta}{8}$. The identity
  $\mathrm{D}_{\mathrm{KL}}(P_v, P_{v'}) = 
  \Delta\log\tfrac{1+\Delta}{1-\Delta}
  \le 3\Delta^2$ implies that $I(Z^n;V) \le 3n\Delta^2$. Similarly, 
  Pinsker's inequality guarantees that
  \begin{equation*}
    \mrm{d}_\msf{TV} \le \sqrt{\frac{1}{2}\max_{v\neq v'}
      \mrm{D}_\mrm{KL}(P_v, P_{v'})} \le \sqrt{3/2}\Delta.
  \end{equation*}
  
  We put everything together and it holds that for $\Delta \in \brk{0, 1}$,
  \begin{equation}
    \minimaxitem \ge \frac{B\Delta}{32}\max\crl*{1 - \frac{3n\Delta^2 + 
    \log 2}{d/8},
      \min\crl*{1, \frac{\exp(d/8)}{\exp(30n\eps \Delta)}}}.
  \end{equation}
  Since $d \ge 32\log 2$, $\Delta = \sqrt{d/(96n)}$ guarantees that
  $1 - \frac{3n\Delta^2 + \log 2}{d/8} \ge 1/2$. On the other hand, setting
  $\Delta = \tfrac{5}{960}\tfrac{d}{n\eps}$, guarantees that
  $\min\crl*{1, \frac{\exp(d/8)}{\exp(30n\eps \Delta)}} \ge 1/2$. The 
  assumption
  on $n$ guarantees that these two values are in $\brk{0, 1}$ and thus setting
  $\Delta^* = \max\crl*{{\sqrt{d/(96n)}}, \tfrac{1}{192}\tfrac{d}{n\eps}}$ 
  which
  implies that
  \begin{equation*}
    \minimaxitem \ge \frac{B}{32}\crl*{\sqrt{\frac{d}{96n}} + \frac{1}{192}\frac{d}{n\eps}}.
  \end{equation*}

  \paragraph{Concluding for user-level DP} Let $m\in \N, m\ge 1$. For the
  user-level DP lower bound, the proof remains the same except that the
  collection $\mc{P}_{\mc{V}}$ becomes $\crl{P_v^m}_{v\in\mc{V}}$ i.e. the
  $m$-fold product distribution of $P_v$. The separation remains exactly the
  same but we now have
  \begin{equation*}
    \mrm{D}_\mrm{KL}(P_v^m, P_{v'}^m) \le 
    3m\Delta^2\mbox{~~and~~}\mrm{d}_\msf{TV}(\mc{P}_\mc{V})
    \le \sqrt{\frac{3m}{2}}\Delta.
  \end{equation*}
  Under the assumption
  $\Delta^* = \max\crl*{{\sqrt{d/(96mn)}},
    \tfrac{1}{192}\tfrac{d}{n\sqrt{m}\eps}}$ is less than $1$ and thus concludes
  the proof.
\end{proof}
}

\end{document}